\documentclass[twoside,11pt]{article}

\usepackage{jmlr2e}
\usepackage{jcdChange}
\usepackage{pslatex}

\newtheorem{remark}[theorem]{Remark}
\let\oldremark\remark
\renewcommand{\remark}{\oldremark\normalfont}

\jmlrheading{}{}{}{4/13; Revised 9/13}{-}{Divyanshu Vats and Robert D. Nowak}

\ShortHeadings{A Junction Tree Framework for Undirected Graphical Model Selection}{Vats and Nowak}
\firstpageno{1}

\usepackage{amssymb,pstricks}
\usepackage{graphicx}
\usepackage{cite}
\usepackage{enumerate}
\usepackage{paralist}
\usepackage{tikz}
\usetikzlibrary{shapes,decorations}
\usepgflibrary{shapes.geometric}
\usepackage{amsmath}
\usepackage{subfigure}
\usepackage{url}
\usepackage{multirow}
\usepackage{booktabs}
\usepackage{algorithmic}
\usepackage[ruled,vlined,boxed]{algorithm2e}
 \usepackage{epsfig}
 \usepackage{pst-grad} % For gradients
 \usepackage{arydshln}
\allowdisplaybreaks

\long\def\symbolfootnote[#1]#2{\begingroup%
 \def\thefootnote{\fnsymbol{footnote}}\footnote[#1]{#2}\endgroup}

\newcommand{\fPC}{\mathsf{PC}}

\newcommand\ind{\protect\mathpalette{\protect\independenT}{\perp}} \def\independenT#1#2{\mathrel{\rlap{$#1#2$}\mkern2mu{#1#2}}}

\newcommand{\JgL}{\mathsf{JgL}}
\newcommand{\gL}{\mathsf{gL}}
\newcommand{\JnL}{\mathsf{JnL}}
\newcommand{\nL}{\mathsf{nL}}

\newcommand{\JfPC}{\mathsf{JPC}}
\newcommand{\CHo}{\mathsf{CH_1}}
\newcommand{\CHt}{\mathsf{CH_2}}
\newcommand{\HBo}{\mathsf{HB_1}}
\newcommand{\HBt}{\mathsf{HB_2}}
\newcommand{\CYo}{\mathsf{CY_1}}
\newcommand{\CYt}{\mathsf{CY_2}}
\newcommand{\NBo}{\mathsf{NB_1}}
\newcommand{\NBt}{\mathsf{NB_2}}

\renewcommand\eqref[1]{(\ref{#1})}

\begin{document}

\title{ A Junction Tree Framework for \\
Undirected Graphical Model Selection}

\author{\name Divyanshu Vats \email dvats@rice.edu \\
       \addr Department of Electrical and Computer Engineering\\
       Rice University\\
       Houston, TX 77251, USA
       \AND
       \name Robert D. Nowak\email nowak@ece.wisc.edu \\
       \addr Department of Electrical and Computer Engineering \\
       University of Wisconsin - Madison\\
       Madison, WI 53706, USA}

\editor{Sebastian Nowozin}

\maketitle

\begin{abstract} %
An undirected graphical model is a joint probability distribution defined on an undirected graph $G^*$, where the vertices in the graph index a collection of random variables and the edges encode conditional independence relationships among random variables.  The undirected graphical model selection (UGMS) problem is to estimate the graph $G^*$ given observations drawn from the undirected graphical model.  This paper proposes a framework for decomposing the UGMS problem into multiple subproblems over clusters and subsets of the separators in a junction tree.  The junction tree is constructed using a graph that contains a superset of the edges in $G^*$.  We highlight three main properties of using junction trees for UGMS.  First, different regularization parameters or different UGMS algorithms can be used to learn different parts of the graph.  This is possible since the subproblems we identify can be solved independently of each other.  Second, under certain conditions, a junction tree based UGMS algorithm can produce consistent results with fewer observations than the usual requirements of existing algorithms.  Third, both our theoretical and experimental results show that the junction tree framework does a significantly better job at finding the weakest edges in a graph than existing methods.  This property is a consequence of both the first and second properties.  Finally, we note that our framework is independent of the choice of the UGMS algorithm and can be used as a wrapper around standard UGMS algorithms for more accurate graph estimation.
\end{abstract}


\section{Introduction}
An undirected graphical model is a joint probability distribution $P_X$ of a random vector $X$ defined on an undirected graph $G^*$.  The graph $G^*$ consists of a set of vertices $V = \{1,\ldots,p\}$ and a set of edges $E(G^*) \subseteq V \times V$.  The vertices index the $p$ random variables in $X$ and the edges $E(G^*)$ characterize conditional independence relationships among the random variables in $X$ \citep{Lauritzen1996}.  We study undirected graphical models (also known as Markov random fields) so that the graph $G^*$ is undirected, i.e., if an edge $(i,j) \in E(G^*)$, then $(j,i) \in E(G^*)$.  The undirected graphical model selection (UGMS) problem is to estimate $G^*$ given $n$ observations $\Xf^n = \left(X^{(1)},\ldots,X^{(n)}\right)$ drawn from $P_X$.  This problem is of interest in many areas including biological data analysis, financial analysis, and social network analysis; see \citet{KollerFriedman2009} for some more examples.

\begin{quotation}
\noindent \textbf{This paper studies the following problem:} \textit{Given the observations $\Xf^n$ drawn from $P_X$ and a graph $H$ that contains all the true edges $E(G^*)$, and possibly some extra edges, estimate the graph $G^*$.  }
\end{quotation}

A natural question to ask is how can the graph $H$ be selected in the first place? One way of doing so is to use  screening algorithms, such as in \citet{fan2008sure} or in \citet{VatsMuG2012}, to eliminate edges that are clearly non-existent in $G^*$. Another method can be to use partial prior information about $X$ to remove unnecessary edges.  For example, this could be based on (i) prior knowledge about statistical properties of genes when analyzing gene expressions, (ii) prior knowledge about companies when analyzing stock returns, or (iii) demographic information when modeling social networks.  Yet another method can be to use clever model selection algorithms that estimate more edges than desired.  Assuming an initial graph $H$ has been computed, our main contribution in this paper is to show how a junction tree representation of $H$ can be used as a wrapper around UGMS algorithms for more accurate graph estimation.

\subsection{Overview of the Junction Tree Framework}
A junction tree is a tree-structured representation of an arbitrary graph \citep{Robertson1986}.  The vertices in a junction tree are clusters of vertices from the original graph.  An edge in a junction tree connects two clusters.  Junction trees are used in many applications to reduce the computational complexity of solving graph related problems \citep{ArnborgPro1989JT}.  Figure~\ref{fig:FirstExample}(c) shows an example of a junction tree for the graph in Figure~\ref{fig:FirstExample}(b).  Notice that each edge in the junction tree is labeled by the set of vertices common to both clusters connected by the edge.  These set of vertices are referred to as a \textit{separator}.

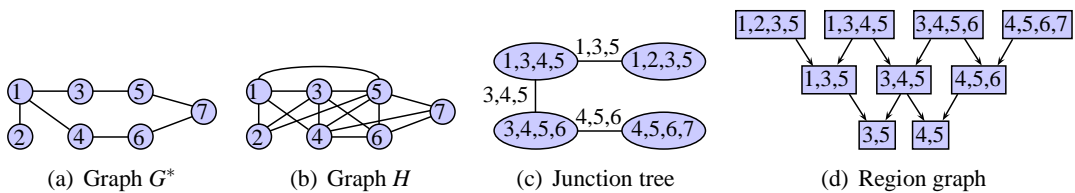
\begin{figure}
\centering
\subfigure[Graph $G^*$]{
\scalebox{0.5} % Change this value to rescale the drawing.
{
\begin{pspicture}(0,-0.95)(5.58,0.95)
\definecolor{color1971b}{rgb}{0.8,0.8,1.0}
\pscircle[linewidth=0.04,dimen=outer,fillstyle=solid,fillcolor=color1971b](0.35,0.6){0.35}
\usefont{T1}{ptm}{m}{n}
\rput(0.2603125,0.6099999){\LARGE 1}
\pscircle[linewidth=0.04,dimen=outer,fillstyle=solid,fillcolor=color1971b](0.35,-0.59999996){0.35}
\usefont{T1}{ptm}{m}{n}
\rput(0.32984376,-0.60999995){\LARGE 2}
\pscircle[linewidth=0.04,dimen=outer,fillstyle=solid,fillcolor=color1971b](1.95,-0.6){0.35}
\usefont{T1}{ptm}{m}{n}
\rput(1.9215624,-0.59){\LARGE 4}
\pscircle[linewidth=0.04,dimen=outer,fillstyle=solid,fillcolor=color1971b](1.95,0.59999996){0.35}
\usefont{T1}{ptm}{m}{n}
\rput(1.9314063,0.61){\LARGE 3}
\pscircle[linewidth=0.04,dimen=outer,fillstyle=solid,fillcolor=color1971b](3.55,-0.6){0.35}
\usefont{T1}{ptm}{m}{n}
\rput(3.5053124,-0.61){\LARGE 6}
\pscircle[linewidth=0.04,dimen=outer,fillstyle=solid,fillcolor=color1971b](3.55,0.59999996){0.35}
\usefont{T1}{ptm}{m}{n}
\rput(3.5164063,0.61){\LARGE 5}
\psline[linewidth=0.04cm](0.68,0.61)(1.62,0.61)
\psline[linewidth=0.04cm](2.28,0.61)(3.22,0.61)
\psline[linewidth=0.04cm](2.28,-0.60999995)(3.22,-0.60999995)
\pscircle[linewidth=0.04,dimen=outer,fillstyle=solid,fillcolor=color1971b](5.23,0.05999998){0.35}
\usefont{T1}{ptm}{m}{n}
\rput(5.193281,0.04999998){\LARGE 7}
\psline[linewidth=0.04cm](0.34,0.27000004)(0.34,-0.26999998)
\psline[linewidth=0.04cm](0.58,0.37000003)(1.66,-0.40999997)
\psline[linewidth=0.04cm](3.86,0.51000005)(4.94,0.21000002)
\psline[linewidth=0.04cm](3.86,-0.46999997)(4.92,-0.06999997)
\end{pspicture} 
}
}
\subfigure[Graph $H$]{
\scalebox{0.5} % Change this value to rescale the drawing.
{
\begin{pspicture}(0,-1.165)(5.58,1.185)
\definecolor{color2875b}{rgb}{0.8,0.8,1.0}
\pscircle[linewidth=0.04,dimen=outer,fillstyle=solid,fillcolor=color2875b](0.35,0.385){0.35}
\usefont{T1}{ptm}{m}{n}
\rput(0.18859375,0.3949999){\LARGE 1}
\pscircle[linewidth=0.04,dimen=outer,fillstyle=solid,fillcolor=color2875b](0.35,-0.81499994){0.35}
\usefont{T1}{ptm}{m}{n}
\rput(0.31093752,-0.8249999){\LARGE 2}
\pscircle[linewidth=0.04,dimen=outer,fillstyle=solid,fillcolor=color2875b](1.95,-0.815){0.35}
\usefont{T1}{ptm}{m}{n}
\rput(1.9064062,-0.805){\LARGE 4}
\pscircle[linewidth=0.04,dimen=outer,fillstyle=solid,fillcolor=color2875b](1.95,0.38499996){0.35}
\usefont{T1}{ptm}{m}{n}
\rput(1.8943751,0.395){\LARGE 3}
\pscircle[linewidth=0.04,dimen=outer,fillstyle=solid,fillcolor=color2875b](3.55,-0.815){0.35}
\usefont{T1}{ptm}{m}{n}
\rput(3.480781,-0.825){\LARGE 6}
\pscircle[linewidth=0.04,dimen=outer,fillstyle=solid,fillcolor=color2875b](3.55,0.38499996){0.35}
\usefont{T1}{ptm}{m}{n}
\rput(3.4821875,0.395){\LARGE 5}
\psline[linewidth=0.04cm](0.68,0.395)(1.62,0.395)
\psline[linewidth=0.04cm](2.28,0.395)(3.22,0.395)
\psline[linewidth=0.04cm](2.28,-0.8249999)(3.22,-0.8249999)
\psline[linewidth=0.04cm](1.94,0.05500004)(1.94,-0.48499998)
\psline[linewidth=0.04cm](3.54,0.05500004)(3.54,-0.48499998)
\pscircle[linewidth=0.04,dimen=outer,fillstyle=solid,fillcolor=color2875b](5.23,-0.15500002){0.35}
\usefont{T1}{ptm}{m}{n}
\rput(5.168906,-0.16500002){\LARGE 7}
\psline[linewidth=0.04cm](0.34,0.05500004)(0.34,-0.48499998)
\psline[linewidth=0.04cm](0.58,0.15500003)(1.66,-0.62499994)
\psline[linewidth=0.04cm](0.56,-0.56499994)(1.7,0.19500002)
\psline[linewidth=0.04cm](0.64,-0.66499996)(3.22,0.29500005)
\psline[linewidth=0.04cm](2.18,0.17500003)(3.28,-0.60499996)
\psline[linewidth=0.04cm](2.24,-0.66499996)(4.9,-0.12499997)
\psline[linewidth=0.04cm](2.14,-0.54499996)(3.3,0.17500003)
\psline[linewidth=0.04cm](3.86,0.29500005)(4.94,-0.00499998)
\psline[linewidth=0.04cm](3.86,-0.68499994)(4.92,-0.28499997)
\psbezier[linewidth=0.04](0.3404207,0.7040723)(0.34,1.1639334)(3.5395792,1.165)(3.5399997,0.70513886)
\end{pspicture} 
}
}
\subfigure[Junction tree]{
% Generated with LaTeXDraw 2.0.8
% Fri Aug 10 20:36:21 CDT 2012
% \usepackage[usenames,dvipsnames]{pstricks}
% \usepackage{epsfig}
% \usepackage{pst-grad} % For gradients
% \usepackage{pst-plot} % For axes
\scalebox{0.5} % Change this value to rescale the drawing.
{
\begin{pspicture}(0,-1.4628125)(5.8885937,1.5028125)
\definecolor{color1971b}{rgb}{0.8,0.8,1.0}
\psellipse[linewidth=0.04,dimen=outer,fillstyle=solid,fillcolor=color1971b](1.3585937,0.8871875)(1.13,0.51)
\usefont{T1}{ptm}{m}{n}
\rput(1.300625,0.8771875){\LARGE 1,3,4,5}
\psellipse[linewidth=0.04,dimen=outer,fillstyle=solid,fillcolor=color1971b](4.7585936,0.8871875)(1.13,0.51)
\usefont{T1}{ptm}{m}{n}
\rput(4.700625,0.9171875){\LARGE 1,2,3,5}
\psellipse[linewidth=0.04,dimen=outer,fillstyle=solid,fillcolor=color1971b](1.3585937,-0.9528125)(1.13,0.51)
\usefont{T1}{ptm}{m}{n}
\rput(1.331875,-0.9228125){\LARGE 3,4,5,6}
\psellipse[linewidth=0.04,dimen=outer,fillstyle=solid,fillcolor=color1971b](4.7585936,-0.9528125)(1.13,0.51)
\usefont{T1}{ptm}{m}{n}
\rput(4.737656,-0.9228125){\LARGE 4,5,6,7}
\psline[linewidth=0.04cm](2.4685938,0.8971875)(3.6485937,0.8971875)
\psline[linewidth=0.04cm](1.3485937,0.4171875)(1.3485937,-0.4628125)
\usefont{T1}{ptm}{m}{n}
\rput(2.990625,1.2371875){\LARGE 1,3,5}
\usefont{T1}{ptm}{m}{n}
\rput(0.55359375,-0.0228125){\LARGE 3,4,5}
\psline[linewidth=0.04cm](2.4685938,-0.9628125)(3.6485937,-0.9628125)
\usefont{T1}{ptm}{m}{n}
\rput(3.029375,-0.6028125){\LARGE 4,5,6}
\end{pspicture} 
}
}
\subfigure[Region graph]{
% Generated with LaTeXDraw 2.0.8
% Fri Aug 10 21:08:27 CDT 2012
% \usepackage[usenames,dvipsnames]{pstricks}
% \usepackage{epsfig}
% \usepackage{pst-grad} % For gradients
% \usepackage{pst-plot} % For axes
\scalebox{0.5} % Change this value to rescale the drawing.
{
\begin{pspicture}(0,-1.86)(8.98,1.86)
\definecolor{color2875b}{rgb}{0.8,0.8,1.0}
\psframe[linewidth=0.04,dimen=outer,fillstyle=solid,fillcolor=color2875b](1.88,1.86)(0.0,1.08)
\usefont{T1}{ptm}{m}{n}
\rput(0.8720313,1.46){\LARGE 1,2,3,5}
\psframe[linewidth=0.04,dimen=outer,fillstyle=solid,fillcolor=color2875b](4.2466664,1.86)(2.3666666,1.08)
\usefont{T1}{ptm}{m}{n}
\rput(3.238698,1.46){\LARGE 1,3,4,5}
\psframe[linewidth=0.04,dimen=outer,fillstyle=solid,fillcolor=color2875b](6.613333,1.86)(4.733333,1.08)
\usefont{T1}{ptm}{m}{n}
\rput(5.636615,1.46){\LARGE 3,4,5,6}
\psframe[linewidth=0.04,dimen=outer,fillstyle=solid,fillcolor=color2875b](8.98,1.86)(7.1,1.08)
\usefont{T1}{ptm}{m}{n}
\rput(8.009063,1.46){\LARGE 4,5,6,7}
\psframe[linewidth=0.04,dimen=outer,fillstyle=solid,fillcolor=color2875b](3.24,0.39)(1.76,-0.39)
\usefont{T1}{ptm}{m}{n}
\rput(2.4220312,-0.01){\LARGE 1,3,5}
\psframe[linewidth=0.04,dimen=outer,fillstyle=solid,fillcolor=color2875b](5.23,0.39)(3.75,-0.39)
\usefont{T1}{ptm}{m}{n}
\rput(4.455,-0.01){\LARGE 3,4,5}
\psframe[linewidth=0.04,dimen=outer,fillstyle=solid,fillcolor=color2875b](7.22,0.39)(5.74,-0.39)
\usefont{T1}{ptm}{m}{n}
\rput(6.460781,-0.01){\LARGE 4,5,6}
\psframe[linewidth=0.04,dimen=outer,fillstyle=solid,fillcolor=color2875b](4.28,-1.08)(3.28,-1.86)
\usefont{T1}{ptm}{m}{n}
\rput(3.755,-1.48){\LARGE 3,5}
\psframe[linewidth=0.04,dimen=outer,fillstyle=solid,fillcolor=color2875b](5.7,-1.08)(4.7,-1.86)
\usefont{T1}{ptm}{m}{n}
\rput(5.1825,-1.48){\LARGE 4,5}
\psline[linewidth=0.04cm,arrowsize=0.05291667cm 3.0,arrowlength=1.4,arrowinset=0.4]{->}(3.12,1.1)(2.72,0.4)
\psline[linewidth=0.04cm,arrowsize=0.05291667cm 3.0,arrowlength=1.4,arrowinset=0.4]{->}(3.68,1.1)(4.14,0.38)
\psline[linewidth=0.04cm,arrowsize=0.05291667cm 3.0,arrowlength=1.4,arrowinset=0.4]{->}(5.2,1.08)(4.8,0.38)
\psline[linewidth=0.04cm,arrowsize=0.05291667cm 3.0,arrowlength=1.4,arrowinset=0.4]{->}(6.1,1.1)(6.56,0.38)
\psline[linewidth=0.04cm,arrowsize=0.05291667cm 3.0,arrowlength=1.4,arrowinset=0.4]{->}(7.38,1.08)(6.98,0.38)
\psline[linewidth=0.04cm,arrowsize=0.05291667cm 3.0,arrowlength=1.4,arrowinset=0.4]{->}(1.62,1.08)(2.08,0.36)
\psline[linewidth=0.04cm,arrowsize=0.05291667cm 3.0,arrowlength=1.4,arrowinset=0.4]{->}(3.04,-0.38)(3.5,-1.1)
\psline[linewidth=0.04cm,arrowsize=0.05291667cm 3.0,arrowlength=1.4,arrowinset=0.4]{->}(4.58,-0.38)(5.04,-1.1)
\psline[linewidth=0.04cm,arrowsize=0.05291667cm 3.0,arrowlength=1.4,arrowinset=0.4]{->}(5.92,-0.38)(5.52,-1.08)
\psline[linewidth=0.04cm,arrowsize=0.05291667cm 3.0,arrowlength=1.4,arrowinset=0.4]{->}(4.4,-0.38)(4.0,-1.08)
\end{pspicture} 
}
}
\caption{ Our framework for estimating the graph in (a) using (b) computes the junction tree in (c) and uses a region graph representation in (d) of the junction tree to decompose the UGMS problem into multiple subproblems.}
\label{fig:FirstExample}
\end{figure}

Let $H$ be a graph that contains all the edges in $G^*$.  We show that the UGMS problem can be decomposed into multiple subproblems over \emph{clusters} and \emph{subsets of the separators} in a junction tree representation of $H$.  In particular, using the junction tree, we construct a region graph, which is a directed graph over clusters of vertices.  An example of a region graph for the junction tree in Figure~\ref{fig:FirstExample}(c) is shown in Figure~\ref{fig:FirstExample}(d).   The first two rows in the region graph are the clusters and separators of the junction tree, respectively.  The rest of the rows contain subsets of the separators\footnote{See Algorithm~\ref{alg:constructregion} for details on how to exactly construct the region graph.}.  The multiple subproblems we identify correspond to estimating a subset of edges over each cluster in the region graph.  For example, the subproblem over the cluster $\{1,2,3,5\}$ in Figure~\ref{fig:FirstExample}(d) estimates the edges $(2,3)$ and $(2,5)$.

We solve the subproblems over the region graph in an iterative manner. First, all subproblems in the first row of the region graph are solved in parallel.  Second, the region graph is updated taking into account the edges removed in the first step.  We keep solving subproblems over rows in the region graph and update the region graph until all the edges in the graph $H$ have been estimated.

As illustrated above, our framework depends on a junction tree representation of the graph $H$ that contains a superset of the true edges.  Given any graph, there may exist several junction tree representations.  An optimal junction tree is a junction tree representation such that the maximum size of the cluster is as small as possible.  Since we apply UGMS algorithms to the clusters of the junction tree, and the complexity of UGMS depends on the number of vertices in the graph, it is useful to apply our framework using optimal junction trees.  Unfortunately, it is computationally intractable to find optimal junction trees \citep{arnborg1987complexity}.  However, there exists several computationally efficient greedy heuristics that compute close to optimal junction trees \citep{Kjaerulff1990,berry2003minimum}.  We use such heuristics to find junction trees when implementing our algorithms in practice.

\subsection{Advantages of Using Junction Trees}
\label{sec:introsum}

We highlight three main advantages of the junction tree framework for UGMS.

\medskip

\noindent
\textbf{Choosing Regularization Parameters and UGMS Algorithms:}  UGMS algorithms typically depend on a regularization parameter that controls the number of estimated edges.  This regularization parameter is usually chosen using model selection algorithms such as cross-validation or stability selection.  Since each subproblem we identify in the region graph is solved independently, different regularization parameters can be used to learn different parts of the graph.  This has advantages when the true graph $G^*$ has different characteristics in different parts of the graph.  Further, since the subproblems are independent, different UGMS algorithms can be used to learn different parts of the graph.  Our numerical simulations clearly show the advantages of this property.

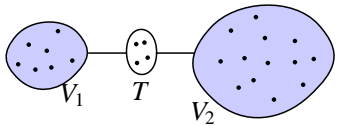
\begin{figure}
\centering
\scalebox{0.5}{
\begin{pspicture}(0,-1.7426543)(9.106651,1.7426542)
\definecolor{color30b}{rgb}{0.8,0.8,1.0}
\psbezier[linewidth=0.04,fillstyle=solid,fillcolor=color30b](0.38826162,0.6918938)(0.77652323,1.1330954)(1.6471232,1.1826543)(2.0335617,0.7397271)(2.42,0.2967998)(2.3162072,0.011489006)(1.8954945,-0.39631054)(1.4747819,-0.80411005)(0.88718396,-0.8773457)(0.45729518,-0.4800186)(0.027406394,-0.08269151)(0.0,0.25069222)(0.38826162,0.6918938)
\psbezier[linewidth=0.04,fillstyle=solid,fillcolor=color30b](5.554909,0.9018673)(6.229819,1.6397681)(7.743169,1.7226542)(8.41491,0.9818674)(9.086651,0.2410806)(8.906227,-0.23609608)(8.174909,-0.91813266)(7.443591,-1.6001692)(6.422179,-1.7226542)(5.6749096,-1.0581326)(4.92764,-0.39361107)(4.8799996,0.1639665)(5.554909,0.9018673)
\psbezier[linewidth=0.04](3.3149095,0.26186737)(3.3149095,-0.5381326)(4.1149096,-0.55813265)(4.1149096,0.24186738)(4.1149096,1.0418674)(3.3149095,1.0618674)(3.3149095,0.26186737)
\psdots[dotsize=0.12](1.4949093,0.8018674)
\psdots[dotsize=0.12](1.1949093,0.28186736)
\psdots[dotsize=0.12](0.89490926,-0.21813263)
\psdots[dotsize=0.12](0.43490934,-0.07813263)
\psdots[dotsize=0.12](0.7149093,0.44186735)
\psdots[dotsize=0.12](1.8749093,0.02186736)
\psdots[dotsize=0.12](1.3149093,-0.15813264)
\psdots[dotsize=0.12](5.8149095,0.00186737)
\psdots[dotsize=0.12](6.9949102,0.6218674)
\psdots[dotsize=0.12](7.794909,0.5418674)
\psdots[dotsize=0.12](8.2949095,0.06186737)
\psdots[dotsize=0.12](8.014909,-0.61813265)
\psdots[dotsize=0.12](7.234909,-1.0181326)
\psdots[dotsize=0.12](7.0749097,-0.03813262)
\psdots[dotsize=0.12](6.4549103,0.08186738)
\psdots[dotsize=0.12](6.2949095,-0.69813263)
\psdots[dotsize=0.12](6.694909,-0.49813262)
\psdots[dotsize=0.12](7.794909,-0.01813262)
\psdots[dotsize=0.12](6.094909,0.8218674)
\psdots[dotsize=0.12](6.754909,1.1818672)
\psdots[dotsize=0.12](3.574909,0.46186736)
\psdots[dotsize=0.12](3.874909,0.10186737)
\psdots[dotsize=0.12](3.594909,-0.03813262)
\psdots[dotsize=0.12](3.7949092,0.5018674)
\usefont{T1}{ptm}{m}{n}
\rput(3.7089717,-0.7831326){\huge $T$}
\usefont{T1}{ptm}{m}{n}
\rput(1.8961593,-0.8831326){\huge $V_1$}
\usefont{T1}{ptm}{m}{n}
\rput(5.3631907,-1.3831327){\huge $V_2$}
\psline[linewidth=0.04cm](2.2949092,0.22186737)(3.3149095,0.22186737)
\psline[linewidth=0.04cm](4.094909,0.22186737)(5.1149096,0.22186737)
\end{pspicture}
}
\caption{Structure of the graph used to analyze the junction tree framework for UGMS.}
\label{fig:graphAss}
\end{figure}

\medskip

\noindent
\textbf{Reduced Sample Complexity:}  One of the key results of our work is to show that in many cases, the junction tree framework is capable of consistently estimating a graph under weaker conditions than required by previously proposed methods.  For example, we show that if $G^*$ consists of two main components that are separated by a relatively small number of vertices (see Figure~\ref{fig:graphAss} for a general example), then, under certain conditions, the number of observations needed for consistent estimation scales like $\log(p_{\min})$, where $p_{\min}$ is the number of vertices in the smaller of the two components.   In contrast, existing methods are known to be consistent if the observations scale like $\log p$, where $p$ is the total number of vertices.  If the smaller component were, for example, exponentially smaller than the larger component, then the junction tree framework is consistent with about $\log\log p$ observations.    For generic problems, without structure that can be exploited by the junction tree framework, we recover the standard conditions for consistency.  

\medskip

\noindent
\textbf{Learning Weak Edges:}  A direct consequence of choosing different regularization parameters and the reduced sample complexity is that certain weak edges, not estimated using standard algorithms, may be estimated when using the junction tree framework.  We show this theoretically and using numerical simulations on both synthetic and real world data.

\subsection{Related Work}
Several algorithms have been proposed in the literature for learning undirected graphical models.  Some examples include References~\citet{PCAlgorithm,KalischBuhlmann2007,BanerjeeGhaoui2008,Friedman2008,NicolaiPeter2006,AnimaTanWillsky2011b,cai2011constrained} for learning Gaussian graphical models, references~\citet{liu2009nonparanormal,xue-zou-2012,liu2012high,lafferty2012sparse,HanLiuTrans2012} for learning non-Gaussian graphical models, and references~\citet{BreslerMosselSly2008,BrombergMargaritis2009,RavikumarWainwrightLafferty2010,NetrapalliSanghaviAllerton2010,AnimaTanWillsky2011a,jalali2011learning,jalali2012high,yang2012graphical} for learning discrete graphical models.  Although all of the above algorithms can be modified to take into account prior knowledge about a graph $H$ that contains all the true edges (see Appendix~\ref{app:examples} for some examples), our junction tree framework is fundamentally different than the standard modification of these algorithms.  The main difference is that the junction tree framework allows for using the \textit{global Markov property} of undirected graphical models (see Definition~\ref{def:ugm}) when learning graphs.  This allows for improved graph estimation, as illustrated by both our theoretical and numerical results.  We note that all of the above algorithms can be used in conjunction with the junction tree framework.

Junction trees have been used for performing \textit{exact} probabilistic inference in graphical models \citep{LauritzenSpiegelhalter1988}.  In particular, given a graphical model, and its junction tree representation, the computational complexity of exact inference is exponential in the size of the cluster in the junction tree with the most of number of vertices.  This has motivated a line of research for learning \textit{thin junction trees} so that the maximum size of the cluster in the estimated junction tree is small so that inference is computationally tractable \citep{ChowLiu1968,bach2001thin,karger2001learning,chechetka2007efficient,KumarBach13a}.  We also make note of algorithms for learning decomposable graphical models where the graph structure is assumed to triangulated \citep{malvestuto1991approximating,giudici1999decomposable}.  In general, the goal in the above algorithms is to learn a joint probability distribution that approximates a more complex probability distribution so that computations, such as inference, can be done in a tractable manner.  On the other hand, this paper considers the problem of learning the \textit{structure of the graph} that best represents the conditional dependencies among the random variables under consideration.

% talk about relationship to Ma's paper
There are two notable algorithms in the literature that use junction trees for learning graphical models.  The first is an algorithm presented in \citet{XieGeng2008} that uses junction trees to find the direction of edges for learning \textit{directed} graphical models.  Unfortunately, this algorithm cannot be used for UGMS.  The second is an algorithm presented in \citet{MaXieGeng2008Chain} for learning chain graphs, that are graphs with both directed and undirected edges.  The algorithm in \citet{MaXieGeng2008Chain} uses a junction tree representation to learn an undirected graph before orienting some of the edges to learn a chain graph.  Our proposed algorithm, and subsequent analysis, differs from the work in \citet{MaXieGeng2008Chain} in the following ways: 
\begin{enumerate}[(i)]
\item Our algorithm identifies an ordering on the edges, which subsequently results in a lower sample complexity and the possibility of learning weak edges in a graph.  The ordering on the edges is possible because of our novel region graph interpretation for learning graphical models.  For example, when learning the graph in Figure~\ref{fig:FirstExample}(a) using Figure~\ref{fig:FirstExample}(b), the algorithm in \citet{MaXieGeng2008Chain} learns the edge $(3,5)$ by applying a UGMS algorithm to the vertices $\{1,2,3,4,5,6\}$.  In contrast, our proposed algorithm first estimates all edges in the second layer of the region graph in Figure~\ref{fig:FirstExample}(d), re-estimates the region graph, and then only applies a UGMS algorithm to $\{3,4,5\}$ to determine if the edge $(3,4)$ belongs to the graph.  In this way, our algorithm, in general, requires applying a UGMS algorithm to a smaller number of vertices when learning edges over separators in a junction tree representation.
\item  Our algorithm for using junction trees for UGMS is independent of the choice of the UGMS algorithm, while the algorithm presented in \citet{MaXieGeng2008Chain} uses conditional independence tests for UGMS.
\item Our algorithm, as discussed in (i), has the additional advantage of learning certain weak edges that may not be estimated when using standard UGMS algorithms.  We theoretically quantify this property of our algorithm, while no such theory was presented in \citet{MaXieGeng2008Chain}.
\end{enumerate}

Recent work has shown that solutions to the graphical lasso (gLasso) \citep{Friedman2008} problem for UGMS over Gaussian graphical models can be computed, under certain conditions, by decomposing the problem over connected components of the graph computed by thresholding the empirical covariance matrix \citep{witten2011fast,Mazumder2012}.  The methods in \citet{witten2011fast,Mazumder2012} are useful for computing solutions to gLasso for particular choices of the regularization parameter and not for accurately estimating graphs.  
Thus, when using gLasso for UGMS, we can use the methods in \citet{witten2011fast,Mazumder2012} to solve gLasso when performing model selection for choosing suitable regularization parameters.  
Finally, we note that recent work in \citet{loh2012structure} uses properties of junction trees to learn discrete graphical models.  The algorithm in \citet{loh2012structure} is designed for learning discrete graphical models and our methods can be used to improve its performance.

\subsection{Paper Organization}

The rest of the paper is organized as follows:

\begin{itemize}\itemsep0pt

\item Section~\ref{sec:background} reviews graphical models and formulates the undirected graphical model selection (UGMS) problem.

\item Section~\ref{sec:regiongraphs} shows how junction trees can be represented as region graphs and outlines an algorithm for constructing a region graph from a junction tree.

\item Section~\ref{sec:applyingUGMSSubset} shows how the region graphs can be used to apply a UGMS algorithm to the clusters and separators of a junction tree.

\item Section~\ref{sec:ugmsjunctiontrees} presents our main framework for using junction trees for UGMS.  In particular, we show how the methods in Sections~\ref{sec:regiongraphs}-\ref{sec:applyingUGMSSubset} can be used iteratively to estimate a graph.

\item Section~\ref{sec:fPC} reviews the PC-Algorithm, which we use to study the theoretical properties of the junction tree framework.

\item Section~\ref{sec:theoreticalPC} presents theoretical results on the sample complexity of learning graphical models using the junction tree framework.  We also highlight advantages of using the junction tree framework as summarized in Section~\ref{sec:introsum}.

\item Section~\ref{sec:numericalSim} presents numerical simulations to highlight the advantages of using junction trees for UGMS in practice.

\item Section~\ref{sec:summary} summarizes the paper and outlines some future work.

\end{itemize}

\section{Preliminaries}
\label{sec:background}
In this section, we review some necessary background on graphs and graphical models that we use in this paper.
Section~\ref{subsec:graphtheory} reviews some graph theoretic concepts.
Section~\ref{subsec:gm} reviews undirected graphical models.  Section~\ref{subsec:ugmsps} formally defines the undirected graphical model selection (UGMS) problem.  Section~\ref{subsec:junctiontrees} reviews junction trees, which we use use as a tool for decomposing UGMS into multiple subproblems.

\subsection{Graph Theoretic Concepts}
\label{subsec:graphtheory}
A graph is a tuple $G = (V,E(G))$, where $V$ is a set of vertices and $E(G) \subseteq V \times V$ are edges connecting vertices in $V$.  For any graph $H$, we use the notation $E(H)$ to denote its edges.  We only consider undirected graphs where if $(v_1,v_2) \in E(G)$, then $(v_2,v_1) \in E(G)$ for $v_1,v_2 \in V$.  Some graph theoretic notations that we use in this paper are summarized as follows:

\begin{itemize}\itemsep1pt
\item Neighbor $ne_G(i)$: Set of nodes connected to $i$.
%\item Degree $deg_G(i)$: Cardinality of $ne_G(i)$.
\item Path $\{i,s_1,\ldots,s_d,j\}$: A sequence of nodes such that $(i,s_1),(s_d,j),(s_{k},s_{k+1}) \in E$ for $k = 1,\ldots,d-1$.  
%\item Cycle $\{i,s_1,\ldots,s_d,i\}$: A closed path. The length of a cycle is the number of edges in the cycle.
%\item Chordless cycle $i,s_1,\ldots,s_d,i$: A closed path such that there are no edges between non-adjacent nodes in the path.
\item Separator $S$: A set of nodes such that all paths from $i$ to $j$ contain at least one node in $S$. The separator $S$ is \textit{minimal} if no proper subset of $S$ separates $i$ and $j$.
\item Induced Subgraph $G[A] = (A,E(G[A]))$: A graph over the nodes $A$ such that $E(G[A])$ contains the edges only involving the nodes in $A$.
%\item Clique: A set of nodes that are all connected to each other.  A clique is \emph{maximal} if the addition of any other node to the set makes the resulting set a non-clique.
\item Complete graph $K_A$: A graph that contains all possible edges over the nodes $A$.  
\end{itemize}

\noindent
For two graphs $G_1 = (V_1,E(G_1))$ and $G_2 = (V_2,E(G_2))$, we define the following standard operations:

\begin{itemize}\itemsep1pt
\item Graph Union: $G_1 \cup G_2 = (V_1 \cup V_2, E_1 \cup E_2)$.
%\item Graph Intersection: $G_1 \cap G_2 = (V_1 \cup V_2, E_1 \cap E_2)$.
\item Graph Difference: $G_1 \backslash G_2 = (V_1 , E_1 \backslash E_2)$.
%\item $G_1 = G_2$ if and only if $E_1 = E_2$.
%\item $G_1 \subseteq G_2$ if and only if $E_1 \subseteq E_2$.
\end{itemize}

\subsection{Undirected Graphical Models}
\label{subsec:gm}
\begin{definition}[Undirected Graphical Model \citep{Lauritzen1996}]
\label{def:ugm}
An undirected graphical model is a probability distribution $P_X$ defined on a graph $G^* = (V,E(G^*))$, where $V = \{1,\ldots,p\}$ indexes the random vector $X = (X_1,\ldots,X_p)$ and the edges $E(G^*)$ encode the following Markov property: for a set of nodes $A$, $B$, and $S$, if $S$ separates $A$ and $B$, then $X_{A} \ind X_{B} | X_S$.
\end{definition}
The Markov property outlined above is referred to as the \textit{global Markov property}.  Undirected graphical models are also referred to as Markov random fields or Markov networks in the literature.  When the joint probability distribution $P_X$ is non-degenerate, i.e., $P_X > 0$, the Markov property in Definition~\ref{def:ugm} are equivalent to the pairwise and local Markov properties:
\begin{itemize}
 \item Pairwise Markov property: For all $(i,j) \notin E$, $X_i \ind X_j | X_{V \backslash \{i,j\}}$.
 \item Local Markov property: For all $i \in V$, $X_{i} \ind X_{V \backslash \{ne_G(i) \cup \{i\}\}} | X_{ne_G(i)}$.
\end{itemize}
In this paper, we always assume $P_X > 0$ and say $P_X$ is \textit{Markov} on $G$ to reflect the Markov properties.  Examples of conditional independence relations conveyed by a probability distribution defined on the graph in Figure~\ref{fig:jtexample}(d) are $X_1 \ind X_6 | \{X_2,X_4\}$ and $X_4 \ind X_6 | \{X_2,X_5,X_8\}$.

\subsection{Undirected Graphical Model Section (UGMS)}
\label{subsec:ugmsps}
%The undirected graphical model selection (UGMS) problem is formally defined as follows.
\begin{definition}[UGMS]
\label{def:ugms}
The undirected graphical model selection (UGMS) problem is to estimate a graph $G^*$ such that the joint probability distribution $P_X$ is Markov on $G^*$, but not Markov on any subgraph of $G^*$.
\end{definition} 
The last statement in Definition~\ref{def:ugms} is important, since, if $P_X$ is Markov on $G^*$, then it is also Markov on any graph that contains $G^*$.  For example, all probability distributions are Markov on the complete graph.  Thus, the UGMS problem is to find the minimal graph that captures the Markov properties associated with a joint probability distribution.  In the literature, this is also known as finding the minimal I-map.  

Let $\Psi$ be an abstract UGMS algorithm that takes as inputs a set of $n$ i.i.d. observations $\Xf^n = \{X^{(1)},\ldots,X^{(n)}\}$ drawn from $P_X$ and a regularization parameter $\lambda_n$.  The output of $\Psi$ is a graph $\widehat{G}_n$, where $\lambda_n$ controls the number of edges estimated in $\widehat{G}_n$.  Note the dependence of the regularization parameter on $n$.  We assume $\Psi$ is consistent, which is formalized in the following assumption.

\begin{assumption}
\label{ass:cons}
There exists a $\lambda_n$ for which $P(\widehat{G}_n = G^*) \rightarrow 1$ as $n \rightarrow \infty$, where $\widehat{G}_n = \Psi(\Xf^n,\lambda_n)$.
\end{assumption}

We give examples of $\Psi$ in Appendix~\ref{app:examples}.  Assumption~\ref{ass:cons} also takes into account the high-dimensional case where $p$ depends on $n$ in such a way that $p,n \rightarrow \infty$.

\subsection{Junction Trees}
\label{subsec:junctiontrees}

Junction trees \citep{Robertson1986} are used extensively for efficiently solving various graph related problems, see \citet{ArnborgPro1989JT} for some examples.  Reference \citet{LauritzenSpiegelhalter1988} shows how junction trees can be used for exact inference (computing marginal distribution given a joint distribution) over graphical models.  We use junction trees as a tool for decomposing the UGMS problem into multiple subproblems.

\begin{definition}[Junction tree]
\label{def:junctiontree}
For an undirected graph $G = (V,E(G))$, a junction tree ${\cal J} = (\C,E({\cal J}))$ is a tree-structured graph over clusters of nodes in $V$ such that
\begin{enumerate}[(i)]
\item Each node in $V$ is associated with at least one cluster in $\C$.
\item For every edge $(i,j) \in E(G) $, there exists a cluster $C_k \in \C$ such that $i,j \in C_k$.
\item ${\cal J}$ satisfies the {running intersection property}:  For all clusters $C_u$, $C_v$, and $C_w$ such that $C_w$ separates $C_u$ and $C_v$ in the tree defined by $E({\cal J})$, $C_u \cap C_v \subset C_w$.
\end{enumerate}
\end{definition}

The first property in Definition~\ref{def:junctiontree} says that all nodes must be mapped to at least one cluster of the junction tree.  The second property states that each edge of the original graph must be contained within a cluster.  The third property, known as the running intersection property, is the most important since it restricts the clusters and the trees that can be be formed. For example, consider the graph in Figure~\ref{fig:jtexample}(a).  By simply clustering the nodes over edges, as done in Figure~\ref{fig:jtexample}(b), we can \textit{not} get a valid junction tree \citep{WainwrightThesis}.  By making appropriate clusters of size three, we get a valid junction tree in Fig.~\ref{fig:jtexample}(c).  In other words, the running intersection property says that for two clusters with a common node, all the clusters on the path between the two clusters must contain that common node.
%An important consequence of the third property is given by the following propoposition.

\begin{figure}[t!]
\begin{center}
\subfigure[]{
\begin{tikzpicture}[scale=0.7]
\tikzstyle{every node}=[draw,shape=circle,scale=0.6, fill=blue!20];
\path (0,2) node (x1) {$1$};
\path (-1,1) node (x2) {$2$};
\path (1,1) node (x3) {$3$};
\path (0,0) node (x4) {$4$};
\draw (x1) -- (x2) -- (x4) -- (x3) -- (x1);
\end{tikzpicture}
} \qquad
\subfigure[]{
\begin{tikzpicture}[scale=0.7]
\tikzstyle{every node}=[draw,scale=0.6, fill=blue!20];
\path (1,1) node[ellipse] (c1) {$1$ $3$};
\path (-1,1) node[ellipse] (c2) {$1$ $2$};
\path (-1,-1) node[ellipse] (c3) {$2$ $4$};
\path (1,-1) node[ellipse] (c4) {$3$ $4$};
\draw (c1) -- (c2) -- (c3) -- (c4);
\end{tikzpicture}
} \qquad
\subfigure[]{
\begin{tikzpicture}[scale=0.7]
\tikzstyle{every node}=[draw,scale=0.6, fill=blue!20];
\path (0,1) node[ellipse] (c1) {$1$ $2$ $3$};
\path (0,-1) node[ellipse] (c2) {$2$ $3$ $4$};
\draw (c1) -- (c2);
\end{tikzpicture}
} \qquad
\subfigure[]{
\begin{tikzpicture}[scale=0.7]
\tikzstyle{every node}=[draw,shape=circle,scale=0.6,fill=blue!20];
\path (-1,1) node (x1) {$1$};
\path (0,1) node (x2) {$2$};
\path (1,1) node (x3) {$3$};
\path (-1,0) node (x4) {$4$};
\path (0,0) node (x5) {$5$};
\path (1,0) node (x6) {$6$};\path (-1,-1) node (x7) {$7$};
\path (0,-1) node (x8) {$8$};
\path (1,-1) node (x9) {$9$};
\draw (x1) -- (x2) -- (x3) -- (x6) -- (x9) -- (x8) -- (x7);
\draw (x7) -- (x4) -- (x1) (x4) -- (x5) -- (x6);
\draw (x2) -- (x5) -- (x8);
\end{tikzpicture}
} \quad \quad
\subfigure[]{
\begin{tikzpicture}[scale=0.8]
\tikzstyle{every node}=[draw,scale=0.6,fill=blue!20];
\path (-1,1) node[ellipse] (c1) {$1$ $2$ $4$};
\path (-1,0) node[ellipse] (c2) {$2$ $4$ $5$ $8$};
\path (-1,-1) node[ellipse] (c3) {$4$ $7$ $8$};
\path (1,0) node[ellipse] (c4) {$2$ $5$ $6$ $8$};
\path (1,1) node[ellipse] (c5) {$2$ $3$ $6$};
\path (1,-1) node[ellipse] (c6) {$6$ $8$ $9$};
\draw (c1) -- (c2) -- (c3);
\draw (c2) -- (c4) -- (c5);
\draw (c4) -- (c6);
\end{tikzpicture}
}
\caption{{\small{(a) An undirected graph, (b) Invalid junction tree since $\{1,2\}$ separates $\{1,3\}$ and $\{3,4\}$,but $3 \notin \{1,2\}$. (c) Valid junction tree for the graph in (a). (d) A grid graph. (e) Junction tree representation of (d).}}}
\label{fig:jtexample}
\end{center}
\end{figure}
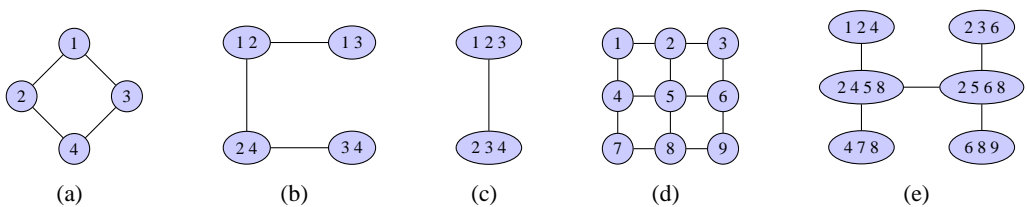

\begin{proposition}[\citep{Robertson1986}]
\label{prop:separation}
Let ${\cal J} = ({\cal C},E({\cal J}))$ be a junction tree of the graph $G$.  Let $S_{uv} = C_u \cap C_v$.  For each $(C_u,C_v) \in \Ec$, we have the following properties:
\begin{enumerate}
 \item $S_{uv} \ne \emptyset$.
 \item $S_{uv}$ separates $C_u \backslash S_{uv}$ and $C_v \backslash S_{uv}$.
\end{enumerate}
\end{proposition}

The set of nodes $S_{uv}$ on the edges are called the \textit{separators} of the junction tree.
Proposition~\ref{prop:separation} says that all clusters$  $ connected by an edge in the junction tree have at least one common node and the common nodes separate nodes in each cluster.  For example, consider the junction tree in Figure~\ref{fig:jtexample}(e) of the graph in Figure~\ref{fig:jtexample}(d).  We can infer that $1$ and $5$ are separated by $2$ and $4$.  Similarly, we can also infer that $4$ and $6$ are separated by $2$, $5$, and $8$.  It is clear that if a graphical model is defined on the graph, then the separators can be used to easily define conditional independence relationships.  For example, using Figure~\ref{fig:jtexample}(e), we can conclude that $X_1 \ind X_5$ given $X_2$ and $X_4$.  As we will see in later Sections, Proposition~\ref{prop:separation} allow the decomposition of UGMS into multiple subproblems over clusters and subsets of the separators in a junction tree.

\section{Overview of Region Graphs}
\label{sec:regiongraphs}
In this section, we show how junction trees can be represented as region graphs.  As we will see in Section~\ref{sec:ugmsjunctiontrees}, region graphs allow us to easily decompose the UGMS problem into multiple subproblems.  There are many different types of region graphs and we refer the readers to \citet{yedidia2005constructing} for a comprehensive discussion about region graphs and how they are useful for characterizing graphical models.  The region graph we present in this section differs slightly from the standard definition of region graphs.  This is mainly because our goal is to estimate edges, while the standard region graphs defined in the literature are used for computations over graphical models.

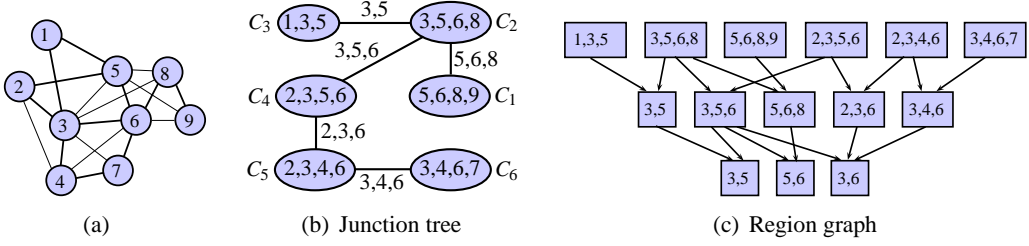
\begin{figure}
\centering
 \subfigure[]{
\scalebox{0.6} % Change this value to rescale the drawing.
{
\begin{pspicture}(0,-1.9600002)(4.48,1.9600002)
\definecolor{color1603b}{rgb}{0.8,0.8,1.0}
\pscircle[linewidth=0.04,dimen=outer,fillstyle=solid,fillcolor=color1603b](0.95,1.6100001){0.35}
\usefont{T1}{ptm}{m}{n}
\rput(0.9014062,1.6149998){\large 1}
\pscircle[linewidth=0.04,dimen=outer,fillstyle=solid,fillcolor=color1603b](0.35,0.46999985){0.35}
\usefont{T1}{ptm}{m}{n}
\rput(0.33375,0.47499985){\large 2}
\pscircle[linewidth=0.04,dimen=outer,fillstyle=solid,fillcolor=color1603b](1.35,-0.39000005){0.35}
\usefont{T1}{ptm}{m}{n}
\rput(1.3128124,-0.40500018){\large 3}
\pscircle[linewidth=0.04,dimen=outer,fillstyle=solid,fillcolor=color1603b](2.51,0.8100002){0.35}
\usefont{T1}{ptm}{m}{n}
\rput(2.4771874,0.81499976){\large 5}
\pscircle[linewidth=0.04,dimen=outer,fillstyle=solid,fillcolor=color1603b](2.93,-0.2700002){0.35}
\usefont{T1}{ptm}{m}{n}
\rput(2.896406,-0.28500023){\large 6}
\pscircle[linewidth=0.04,dimen=outer,fillstyle=solid,fillcolor=color1603b](2.53,-1.3900002){0.35}
\usefont{T1}{ptm}{m}{n}
\rput(2.5032814,-1.4050002){\large 7}
\pscircle[linewidth=0.04,dimen=outer,fillstyle=solid,fillcolor=color1603b](1.27,-1.6100001){0.35}
\usefont{T1}{ptm}{m}{n}
\rput(1.2523439,-1.6250002){\large 4}
\pscircle[linewidth=0.04,dimen=outer,fillstyle=solid,fillcolor=color1603b](3.63,0.7899999){0.35}
\usefont{T1}{ptm}{m}{n}
\rput(3.5762498,0.75500005){\large 8}
\pscircle[linewidth=0.04,dimen=outer,fillstyle=solid,fillcolor=color1603b](4.13,-0.2700001){0.35}
\usefont{T1}{ptm}{m}{n}
\rput(4.1071873,-0.2850001){\large 9}
\psline[linewidth=0.04cm](1.28,1.5400001)(2.26,1.0000001)
\psline[linewidth=0.04cm](3.42,0.5400001)(3.14,0.0)
\psline[linewidth=0.04cm](2.62,0.5200001)(2.78,0.03999995)
\psline[linewidth=0.04cm](1.66,-0.36000016)(2.6,-0.28000018)
\psline[linewidth=0.04cm](3.82,0.5000001)(4.02,0.0400001)
\psline[linewidth=0.04cm](2.84,-0.57999986)(2.66,-1.0799999)
\psline[linewidth=0.04cm](1.32,-0.7199999)(1.26,-1.2799999)
\psline[linewidth=0.04cm](0.58,0.2200001)(1.08,-0.19999996)
\psline[linewidth=0.04cm](1.6,-1.56)(2.22,-1.4399999)
\psline[linewidth=0.04cm](1.02,1.2800001)(1.28,-0.0799999)
\psline[linewidth=0.04cm](0.68,0.5000001)(2.2,0.7600001)
\psline[linewidth=0.02cm](2.28,0.5800001)(1.58,-0.1599999)
\psline[linewidth=0.02cm](3.3,0.8200001)(2.84,0.8400001)
\psline[linewidth=0.02cm](3.78,-0.2799999)(3.28,-0.2799999)
\psline[linewidth=0.02cm](0.44,0.1400001)(1.06,-1.3599999)
\psline[linewidth=0.02cm](2.66,-0.4399999)(1.5,-1.3599999)
\psline[linewidth=0.02cm](2.28,-1.1599998)(1.58,-0.6199999)
\psline[linewidth=0.02cm](3.34,0.6200001)(1.62,-0.2399999)
\psline[linewidth=0.02cm](3.86,-0.0799999)(2.76,0.6200001)
\end{pspicture} 
}}
\subfigure[Junction tree]{
\scalebox{0.7} % Change this value to rescale the drawing.
{
\begin{pspicture}(0,-1.8765625)(5.900781,1.8540626)
\definecolor{color1603b}{rgb}{0.8,0.8,1.0}
\psellipse[linewidth=0.04,dimen=outer,fillstyle=solid,fillcolor=color1603b](4.1771874,0.026562516)(0.78,0.41)
\usefont{T1}{ptm}{m}{n}
\rput(4.1596875,0.031562485){\large 5,6,8,9}
\psellipse[linewidth=0.04,dimen=outer,fillstyle=solid,fillcolor=color1603b](4.1971874,-1.3734375)(0.78,0.41)
\usefont{T1}{ptm}{m}{n}
\rput(4.1818748,-1.3684375){\large 3,4,6,7}
\psellipse[linewidth=0.04,dimen=outer,fillstyle=solid,fillcolor=color1603b](1.6171877,-1.3734375)(0.78,0.41)
\usefont{T1}{ptm}{m}{n}
\rput(1.613125,-1.3684375){\large 2,3,4,6}
\psellipse[linewidth=0.04,dimen=outer,fillstyle=solid,fillcolor=color1603b](1.4971877,1.4165626)(0.6,0.36)
\usefont{T1}{ptm}{m}{n}
\rput(1.41875,1.4115624){\large 1,3,5}
\psellipse[linewidth=0.04,dimen=outer,fillstyle=solid,fillcolor=color1603b](1.6371877,0.026562516)(0.78,0.41)
\usefont{T1}{ptm}{m}{n}
\rput(1.633125,0.031562485){\large 2,3,5,6}
\psellipse[linewidth=0.04,dimen=outer,fillstyle=solid,fillcolor=color1603b](4.1971874,1.4265625)(0.78,0.41)
\usefont{T1}{ptm}{m}{n}
\rput(4.17625,1.4315625){\large 3,5,6,8}
\psline[linewidth=0.04cm](4.1971874,0.4365625)(4.1971874,1.0365626)
\psline[linewidth=0.04cm](2.0771878,1.4165626)(3.4371874,1.4165626)
\psline[linewidth=0.04cm](2.1371877,0.3165625)(3.6771877,1.1565626)
\psline[linewidth=0.04cm](2.3771877,-1.3634375)(3.4371874,-1.3634375)
\psline[linewidth=0.04cm](1.6371877,-0.36343747)(1.6371877,-0.9834375)
\usefont{T1}{ptm}{m}{n}
\rput(2.7525,1.6515625){\large 3,5}
\usefont{T1}{ptm}{m}{n}
\rput(4.673125,0.7115625){\large 5,6,8}
\usefont{T1}{ptm}{m}{n}
\rput(2.4046876,0.89156246){\large 3,5,6}
\usefont{T1}{ptm}{m}{n}
\rput(2.153125,-0.66843754){\large 2,3,6}
\usefont{T1}{ptm}{m}{n}
\rput(2.8846877,-1.6284375){\large 3,4,6}
\usefont{T1}{ptm}{m}{n}
\rput(5.2518754,0.031562485){\large $C_1$}
\usefont{T1}{ptm}{m}{n}
\rput(5.269687,1.4315625){\large $C_2$}
\usefont{T1}{ptm}{m}{n}
\rput(0.5804688,1.4115624){\large $C_3$}
\usefont{T1}{ptm}{m}{n}
\rput(0.57109374,0.011562484){\large $C_4$}
\usefont{T1}{ptm}{m}{n}
\rput(0.5442188,-1.3884375){\large $C_5$}
\usefont{T1}{ptm}{m}{n}
\rput(5.2703123,-1.3884375){\large $C_6$}
\end{pspicture} }}
\subfigure[Region graph]{
% Generated with LaTeXDraw 2.0.8
% Tue Aug 07 16:57:15 CDT 2012
% \usepackage[usenames,dvipsnames]{pstricks}
% \usepackage{epsfig}
% \usepackage{pst-grad} % For gradients
% \usepackage{pst-plot} % For axes
\scalebox{0.55} % Change this value to rescale the drawing.
{
\begin{pspicture}(0,-2.11)(11.18,2.11)
\definecolor{color0b}{rgb}{0.8,0.8,1.0}
\psframe[linewidth=0.04,dimen=outer,fillstyle=solid,fillcolor=color0b](1.52,2.11)(0.0,1.23)
\usefont{T1}{ptm}{m}{n}
\rput(0.649375,1.66){\large 1,3,5}
\psframe[linewidth=0.04,dimen=outer,fillstyle=solid,fillcolor=color0b](3.46,2.11)(1.94,1.23)
\usefont{T1}{ptm}{m}{n}
\rput(2.666875,1.68){\large 3,5,6,8}
\psframe[linewidth=0.04,dimen=outer,fillstyle=solid,fillcolor=color0b](5.38,2.11)(3.86,1.23)
\usefont{T1}{ptm}{m}{n}
\rput(4.5903125,1.68){\large 5,6,8,9}
\psframe[linewidth=0.04,dimen=outer,fillstyle=solid,fillcolor=color0b](7.34,2.11)(5.82,1.23)
\usefont{T1}{ptm}{m}{n}
\rput(6.56375,1.68){\large 2,3,5,6}
\psframe[linewidth=0.04,dimen=outer,fillstyle=solid,fillcolor=color0b](9.28,2.11)(7.76,1.23)
\usefont{T1}{ptm}{m}{n}
\rput(8.50375,1.68){\large 2,3,4,6}
\psframe[linewidth=0.04,dimen=outer,fillstyle=solid,fillcolor=color0b](11.18,2.11)(9.66,1.23)
\usefont{T1}{ptm}{m}{n}
\rput(10.3925,1.68){\large 3,4,6,7}
\psframe[linewidth=0.04,dimen=outer,fillstyle=solid,fillcolor=color0b](4.4,0.43)(3.14,-0.45)
\usefont{T1}{ptm}{m}{n}
\rput(3.7353125,-0.04){\large 3,5,6}
\psframe[linewidth=0.04,dimen=outer,fillstyle=solid,fillcolor=color0b](6.08,0.43)(4.82,-0.45)
\usefont{T1}{ptm}{m}{n}
\rput(5.40375,-0.04){\large 5,6,8}
\psframe[linewidth=0.04,dimen=outer,fillstyle=solid,fillcolor=color0b](7.76,0.43)(6.5,-0.45)
\usefont{T1}{ptm}{m}{n}
\rput(7.10375,-0.04){\large 2,3,6}
\psframe[linewidth=0.04,dimen=outer,fillstyle=solid,fillcolor=color0b](9.42,0.43)(8.16,-0.45)
\usefont{T1}{ptm}{m}{n}
\rput(8.755313,-0.04){\large 3,4,6}
\psframe[linewidth=0.04,dimen=outer,fillstyle=solid,fillcolor=color0b](4.72,-1.23)(3.78,-2.11)
\usefont{T1}{ptm}{m}{n}
\rput(4.203125,-1.7){\large 3,5}
\psframe[linewidth=0.04,dimen=outer,fillstyle=solid,fillcolor=color0b](6.06,-1.23)(5.12,-2.11)
\usefont{T1}{ptm}{m}{n}
\rput(5.5521874,-1.7){\large 5,6}
\psframe[linewidth=0.04,dimen=outer,fillstyle=solid,fillcolor=color0b](7.38,-1.23)(6.44,-2.11)
\usefont{T1}{ptm}{m}{n}
\rput(6.8753123,-1.7){\large 3,6}
\psframe[linewidth=0.04,dimen=outer,fillstyle=solid,fillcolor=color0b](2.72,0.43)(1.78,-0.45)
\usefont{T1}{ptm}{m}{n}
\rput(2.203125,-0.04){\large 3,5}
\psline[linewidth=0.04cm,arrowsize=0.05291667cm 2.0,arrowlength=1.4,arrowinset=0.4]{->}(0.64,1.23)(2.0,0.43)
\psline[linewidth=0.04cm,arrowsize=0.05291667cm 2.0,arrowlength=1.4,arrowinset=0.4]{->}(2.44,1.25)(2.28,0.43)
\psline[linewidth=0.04cm,arrowsize=0.05291667cm 2.0,arrowlength=1.4,arrowinset=0.4]{->}(2.68,1.27)(3.56,0.41)
\psline[linewidth=0.04cm,arrowsize=0.05291667cm 2.0,arrowlength=1.4,arrowinset=0.4]{->}(3.12,1.23)(5.2,0.41)
\psline[linewidth=0.04cm,arrowsize=0.05291667cm 2.0,arrowlength=1.4,arrowinset=0.4]{->}(6.3,1.25)(4.0,0.43)
\psline[linewidth=0.04cm,arrowsize=0.05291667cm 2.0,arrowlength=1.4,arrowinset=0.4]{->}(4.62,1.27)(5.38,0.43)
\psline[linewidth=0.04cm,arrowsize=0.05291667cm 2.0,arrowlength=1.4,arrowinset=0.4]{->}(6.46,1.27)(6.94,0.43)
\psline[linewidth=0.04cm,arrowsize=0.05291667cm 2.0,arrowlength=1.4,arrowinset=0.4]{->}(8.22,1.27)(7.24,0.43)
\psline[linewidth=0.04cm,arrowsize=0.05291667cm 2.0,arrowlength=1.4,arrowinset=0.4]{->}(8.44,1.27)(8.64,0.43)
\psline[linewidth=0.04cm,arrowsize=0.05291667cm 2.0,arrowlength=1.4,arrowinset=0.4]{->}(10.18,1.25)(9.0,0.43)
\psline[linewidth=0.04cm,arrowsize=0.05291667cm 2.0,arrowlength=1.4,arrowinset=0.4]{->}(2.22,-0.41)(4.1,-1.25)
\psline[linewidth=0.04cm,arrowsize=0.05291667cm 2.0,arrowlength=1.4,arrowinset=0.4]{->}(3.56,-0.43)(4.34,-1.25)
\psline[linewidth=0.04cm,arrowsize=0.05291667cm 2.0,arrowlength=1.4,arrowinset=0.4]{->}(3.88,-0.43)(5.44,-1.25)
\psline[linewidth=0.04cm,arrowsize=0.05291667cm 2.0,arrowlength=1.4,arrowinset=0.4]{->}(5.48,-0.41)(5.62,-1.29)
\psline[linewidth=0.04cm,arrowsize=0.05291667cm 2.0,arrowlength=1.4,arrowinset=0.4]{->}(7.08,-0.43)(6.86,-1.23)
\psline[linewidth=0.04cm,arrowsize=0.05291667cm 2.0,arrowlength=1.4,arrowinset=0.4]{->}(8.74,-0.43)(7.0,-1.25)
\psline[linewidth=0.04cm,arrowsize=0.05291667cm 2.0,arrowlength=1.4,arrowinset=0.4]{->}(4.18,-0.43)(6.7,-1.23)
\end{pspicture} 
}
}
\caption{(a) An example of $H$. (b) A junction tree representation of $H$. (c) A region graph representation of (b) computed using Algorithm~\ref{alg:constructregion}.}
\label{fig:exampleregiongraph}
\end{figure}

\begin{algorithm}
\label{alg:constructregion}
\DontPrintSemicolon
\caption{Constructing region graphs}
\textbf{Input:} A junction tree ${\cal J} = ({\cal C},E({\cal J}))$ of a graph $H$. \;
\textbf{Output:} A region graph ${\cal G} = ({\cal R}, \vec{E}({\cal G}))$. \;
\nl ${\cal R}^1 = {\cal C}$, where ${\cal C}$ are the clusters of the junction tree ${\cal J}$. \;
\nl Let ${\cal R}^2$ be all the separators of ${\cal J}$, i.e., ${\cal R}^2 = \{S_{uv} = C_u \cap C_v : (C_u,C_v) \in E({\cal J}) \}$. \;
 \nl To construct ${\cal R}^3$, find all possible pairwise intersections of regions in ${\cal R}^2$.  Add all intersecting regions with cardinality greater than one to ${\cal R}^3$. \;
 \nl Repeat previous step to construct ${\cal R}^4,\ldots,{\cal R}^L$ until there are no more intersecting regions of cardinality greater than one. \;
 \nl For $R \in {\cal R}^{\ell}$ and $S \in {\cal R}^{\ell+1}$, add the edge $(R,S)$ to $\vec{E}({\cal G})$ if $S \subseteq R$. \;
 \nl Let ${\cal R} = \{ {\cal R}^1,\ldots,{\cal R}^L \}$.
\end{algorithm}

A \textit{region} is a collection of nodes, which in this paper can be the clusters of the junction tree, separators of the junction tree, or subsets of the separators.  A region graph ${\cal G} = ({\cal R},\vec{E}({\cal G}))$ is a directed graph where the vertices are regions and the edges represent directed edges from one region to another.  We use the notation $\vec{E}(\cdot)$ to emphasize that region graphs contain directed edges.  A description of region graphs is given as follows:
\begin{itemize}
 \item The set $\vec{E}({\cal G})$ contains directed edges so that if $(R,S) \in \vec{E}({\cal G})$, then there exists a directed edge from region $R$ to region $S$.
 \item Whenever $R \longrightarrow S$, then $S \subseteq R $.
\end{itemize}

Algorithm~\ref{alg:constructregion} outlines an algorithm to construct region graphs given a junction tree representation of a graph $H$.  We associate a label $l$ with every region in ${\cal R}$ and group regions with the same label to partition ${\cal R}$ into $L$ groups ${\cal R}^1,\ldots,{\cal R}^L$. In Algorithm~\ref{alg:constructregion}, we initialize ${\cal R}^1$ and ${\cal R}^2$ to be the clusters and separators of a junction tree ${\cal J}$, respectively, and then iteratively find ${\cal R}^3,\ldots,{\cal R}^L$ by computing all possible intersections of regions with the same label.  The edges in $\vec{E}({\cal G})$ are only drawn from a region in ${\cal R}^{l}$ to a region in ${\cal R}^{l+1}$.  Figure~\ref{fig:exampleregiongraph}(c) shows an example of a region graph computed using the junction tree in Figure~\ref{fig:exampleregiongraph}(b).

\begin{remark}
Note that the construction of the region graph depends on the junction tree.  Using methods in \citet{VatsMouraGraph2012}, we can always construct junction trees such that the region graph only has two sets of regions, namely the clusters of the junction tree and the separators of the junction tree.  However, in this case, the size of the regions or clusters may be too large.  This may not be desirable since the computational complexity of applying UGMS algorithms to region graphs, as shown in Section~\ref{sec:ugmsjunctiontrees}, depends on the size of the regions.
\end{remark}

\begin{remark}[Region graph vs. Junction tree]
For every junction tree, Algorithm~\ref{alg:constructregion} outputs a unique region graph.  The junction tree only characterizes the relationship between the clusters in a junction tree.  A region graph extends the junction tree representation to characterize the relationships between the clusters as well as the separators.  For example, in Figure~\ref{fig:exampleregiongraph}(c), the region $\{5,6\}$ is in the third row and is a subset of two separators of the junction tree.  Thus, the only difference between the region graph and the junction tree is the additional set of regions introduced in ${\cal R}^3, \ldots, {\cal R}^L$.  
\end{remark}

\begin{remark}
From the construction in Algorithm~\ref{alg:constructregion}, ${\cal R}$ may have two or more regions that are the same but have different labels.  For example, in Figure~\ref{fig:exampleregiongraph}(c), the region $\{3,5\}$ is in both ${\cal R}^2$ and ${\cal R}^3$.  We can avoid this situation by removing $\{3,5\}$ from ${\cal R}^2$ and adding an edge from the region $\{1,3,5\}$ in ${\cal R}^1$ to the region $\{3,5\}$ in ${\cal R}^3$.  For notational simplicity and for the purpose of illustration, we allow for duplicate regions.  This does not change the theory or the algorithms that we develop.
\end{remark}

\section{Applying UGMS to Region Graphs}
\label{sec:applyingUGMSSubset}

Before presenting our framework for decomposing UGMS into multiple subproblems, we first show how UGMS algorithms can be applied to estimate a subset of edges in a region of a region graph.  In particular, for a region graph ${\cal G} = ({\cal R},\vec{E}({\cal G}))$, \textit{we want to identify a set of edges in the induced subgraph $H[R]$ that can be estimated by applying a UGMS algorithm to either $R$ or a set of vertices that contains $R$}.  With this goal in mind, define the children $ch(R)$ of a region $R$ as follows:
\begin{align}
\text{Children: } ch(R) = \left\{S : (R,S) \in \vec{{\cal E}} \right\} \,.  \label{eq:children}
\end{align}
We say $R$ connects to $S$ if $(R,S) \in \vec{E}({\cal G})$.  Thus, the children in (\ref{eq:children}) consist of all regions that $R$ connects to.  For example, in Figure~\ref{fig:exampleregiongraph}(c), 
\[ch(\{2,3,4,6\}) = \{\{2,3,6\},\{3,4,6\}\}\,.\]  
If there exists a direct path from $S$ to $R$, we say $S$ is an \emph{ancestor} of $R$.  The set of all ancestors of $R$ is denoted by $an(R)$.  For example, in Figure~\ref{fig:exampleregiongraph}(c), 
\begin{align*}
an(\{5,6,8,9\}) &= \emptyset, \\
an(\{3,5,6\}) &= \{\{3,5,6,8\},\{2,3,5,6\}\}, \text{and} \\
an(\{3,6\}) &= \{\{3,5,6\},\{2,3,6\},\{3,4,6\},
\{2,3,5,6\},\{2,3,4,6\},\{3,4,6,7\},\{3,5,6,8\} \} \}.
\end{align*}
The notation $\overline{R}$ takes the union of all regions in $an(R)$ and $R$ so that
\begin{equation}
\overline{R} = \bigcup_{S \in \{an(R),R\}} S \,. \label{eq:rbar}
\end{equation}
Thus, $\overline{R}$\textit{ contains the union of all clusters in the junction tree that contain $R$}.  An illustration of some of the notations defined on region graphs is shown in Figure~\ref{fig:NotationRegionGraphs}.  Using $ch(R)$, define the subgraph $H'_R$ as
\footnote{For graphs $G_1$ and $G_2$, $E(G_1 \backslash G_2) = E(G_1) \backslash E(G_2)$ and $E(G_1 \cup G_2) = E(G_1) \cup E(G_2)$}
\begin{equation}
H_{R}' = H[R] \backslash \left\{\cup_{S \in ch(R)} K_S \right\} \,, \label{eq:haprime}
\end{equation}
where $H[R]$ is the induced subgraph that contains all edges in $H$ over the region $R$ and $K_S$ is the complete graph over $S$.  In words, \textit{$H'_R$ is computed by removing all edges from $H[R]$ that are contained in another separator}.  For example, in Figure~\ref{fig:exampleregiongraph}(c), when $R = \{5,6,8\}$, $E(H'_R) = \{ (5,8),(6,8) \}$.
The subgraph $H_R'$ is important since it identifies the edges that can be estimated when applying a UGMS algorithm to the set of vertices $\overline{R}$.

\begin{proposition}
\label{prop:regionest}
Suppose $E(G^*) \subseteq E(H)$.  All edges in $H'_R$ can be estimated by solving a UGMS problem over the vertices $\overline{R}$.
\end{proposition}
\begin{proof}
See Appendix~\ref{sub:proofpropregioest}.
\end{proof}

\begin{algorithm}[tb]
\caption{UGMS over regions in a region graph}
\label{alg:ugmsregion}
\begin{algorithmic}[1]
   	\STATE {\bfseries Input:} Region graph ${\cal G} = ({\cal R},\vec{E}({\cal G}))$, a region $R$, observations $\Xf^n$, and a UGMS algorithm $\Psi$.
	%\STATE \textbf{Output:} A set of edges $\widehat{E}_R$.
	\STATE Compute $H_R'$ using (\ref{eq:haprime}) and $\overline{R}$ using (\ref{eq:rbar}).
	\STATE Apply $\Psi$ to $\Xf^n_{\overline{R}}$ to estimate edges in $H'_R$. See Appendix~\ref{app:examples} for examples.
	\STATE \textbf{Return} the estimated edges $\widehat{E}_R$.
\end{algorithmic}
\end{algorithm}

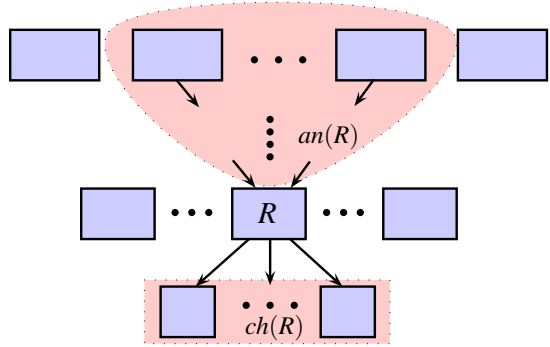
\begin{figure}
\centering
\scalebox{0.8} % Change this value to rescale the drawing.
{
\begin{pspicture}(0,-2.9292195)(11.690312,2.9392195)
\definecolor{color2875b}{rgb}{0.8,0.8,1.0}
\definecolor{color410b}{rgb}{1.0,0.8,0.8}
\psbezier[linewidth=0.02,fillstyle=solid,fillcolor=color410b,linestyle=dotted,dotsep=0.16cm](2.6525,2.2731726)(3.1213067,2.9231725)(7.903931,2.9292195)(8.3725,2.2531726)(8.841069,1.5771257)(6.29176,-0.32682744)(5.2925,-0.2980469)(4.29324,-0.26926637)(2.1836934,1.6231728)(2.6525,2.2731726)
\psframe[linewidth=0.02,linestyle=dotted,dotsep=0.16cm,dimen=outer,fillstyle=solid,fillcolor=color410b](7.3325,-1.8492196)(3.2325,-2.9292195)
\psframe[linewidth=0.04,dimen=outer,fillstyle=solid,fillcolor=color2875b](2.5125,2.3207805)(0.9925,1.4407805)
\psframe[linewidth=0.04,dimen=outer,fillstyle=solid,fillcolor=color2875b](4.5525,2.3207805)(3.0325,1.4407805)
\psframe[linewidth=0.04,dimen=outer,fillstyle=solid,fillcolor=color2875b](7.9325,2.3207805)(6.4125,1.4407805)
\psframe[linewidth=0.04,dimen=outer,fillstyle=solid,fillcolor=color2875b](9.9525,2.3207805)(8.4325,1.4407805)
\psdots[dotsize=0.12](5.0725,1.8007805)
\psdots[dotsize=0.12](5.4725,1.8007805)
\psdots[dotsize=0.12](5.8725,1.8007805)
\psframe[linewidth=0.04,dimen=outer,fillstyle=solid,fillcolor=color2875b](3.4325,-0.31921953)(2.1725,-1.1992195)
\psframe[linewidth=0.04,dimen=outer,fillstyle=solid,fillcolor=color2875b](5.9425,-0.31921953)(4.6825,-1.1992195)
\psframe[linewidth=0.04,dimen=outer,fillstyle=solid,fillcolor=color2875b](8.4525,-0.31921953)(7.1925,-1.1992195)
\psframe[linewidth=0.04,dimen=outer,fillstyle=solid,fillcolor=color2875b](4.4325,-1.9492195)(3.4925,-2.8292196)
\psframe[linewidth=0.04,dimen=outer,fillstyle=solid,fillcolor=color2875b](7.0925,-1.9492195)(6.1525,-2.8292196)
\psdots[dotsize=0.12](3.7525,-0.7392195)
\psdots[dotsize=0.12](4.0325,-0.7392195)
\psdots[dotsize=0.12](4.3125,-0.7392195)
\psdots[dotsize=0.12](6.2725,-0.7392195)
\psdots[dotsize=0.12](6.5525,-0.7392195)
\psdots[dotsize=0.12](6.8325,-0.7392195)
\usefont{T1}{ptm}{m}{n}
\rput(5.286523,-0.74421954){\Large $R$}
\psdots[dotsize=0.12](4.9325,-2.2792196)
\psdots[dotsize=0.12](5.3325,-2.2792196)
\psdots[dotsize=0.12](5.7325,-2.2792196)
\psdots[dotsize=0.12](5.3325,0.81078047)
\psdots[dotsize=0.12](5.3325,0.6041138)
\psdots[dotsize=0.12](5.3325,0.39744714)
\psdots[dotsize=0.12](5.3325,0.19078048)
\psline[linewidth=0.04cm,arrowsize=0.05291667cm 3.0,arrowlength=1.4,arrowinset=0.4]{->}(4.9925,-1.1692195)(4.0925,-1.9692194)
\psline[linewidth=0.04cm,arrowsize=0.05291667cm 3.0,arrowlength=1.4,arrowinset=0.4]{->}(5.3325,-1.1692195)(5.3325,-1.9492195)
\psline[linewidth=0.04cm,arrowsize=0.05291667cm 3.0,arrowlength=1.4,arrowinset=0.4]{->}(5.6525,-1.1692195)(6.5325,-1.9692194)
\usefont{T1}{ptm}{m}{n}
\rput(5.406387,-2.65){\large $ch(R)$}
\usefont{T1}{ptm}{m}{n}
\rput(6.3060355,0.5206175){\large ${an(R)}$}
\psline[linewidth=0.04cm,arrowsize=0.05291667cm 3.0,arrowlength=1.4,arrowinset=0.4]{->}(4.7125,0.13078047)(5.0925,-0.32921952)
\psline[linewidth=0.04cm,arrowsize=0.05291667cm 3.0,arrowlength=1.4,arrowinset=0.4]{->}(5.9925,0.11078046)(5.6725,-0.32921952)
\psline[linewidth=0.04cm,arrowsize=0.05291667cm 3.0,arrowlength=1.4,arrowinset=0.4]{->}(3.7725,1.4707805)(4.1525,1.0107805)
\psline[linewidth=0.04cm,arrowsize=0.05291667cm 3.0,arrowlength=1.4,arrowinset=0.4]{->}(7.0525,1.4707805)(6.7325,1.0307804)
%\usefont{T1}{ptm}{m}{n}
%\rput(10.469844,1.9206175){\large ${\cal R}_1$}
%\usefont{T1}{ptm}{m}{n}
%\rput(0.211875,1.8806175){\large \color{white}R1}
\end{pspicture} 
}
\caption{Notations defined on region graphs.  The children $ch(R)$ are the set of regions that $R$ connects to.  The ancestors $an(R)$ are all the regions that have a directed path to the region $R$.  The set $\overline{R}$ takes the union of all regions in $an(R)$ and $R$.}
\label{fig:NotationRegionGraphs}
\end{figure}

Proposition~\ref{prop:regionest} says that all edges in $H_R'$ can be estimated by applying a UGMS algorithm to the set of vertices $\overline{R}$.  The intuition behind the result is that only those edges in the region $R$ can be estimated whose Markov properties can be deduced using the vertices in $\overline{R}$.  Moreover, the edges \textit{not estimated} in $H[R]$ share an edge with another region that does not contain all the vertices in $R$.   Algorithm~\ref{alg:ugmsregion} summarizes the steps involved in estimating the edges in $H_R'$ using the UGMS algorithm $\Psi$ defined in Section~\ref{subsec:ugmsps}.  Some examples on how to use Algorithm~\ref{alg:ugmsregion} to estimate some edges of the graph in Figure~\ref{fig:exampleregiongraph}(a) using the region graph in  Figure~\ref{fig:exampleregiongraph}(c) are described as follows.
\begin{enumerate}
\item Let $R = \{1,3,5\}$.  This region only connects to $\{3,5\}$.  This means that all edges, except the edge $(3,5)$ in $H[R]$, can be estimated by applying $\Psi$ to $R$.
\item Let $R = \{3,5,6\}$.  The children of this region are $\{3,5\}$, $\{5,6\}$, and $\{3,6\}$.  This means that $H'_R = \emptyset$, i.e., no edge over $H[R]$ can be estimated by applying $\Psi$ to $\{3,5,6\}$.
\item Let $R = \{3,4,6\}$.  This region only connects to $\{3,6\}$.  Thus, all edges except $(3,6)$ can be estimated.  The regions $\{2,3,4,6\}$ and $\{3,4,6,7\}$ connect to $R$, so $\Psi$ needs to be applied to $\overline{R} = \{2,3,4,6,7\}$.
\end{enumerate}

\section{UGMS Using Junction Trees: A General Framework}
\label{sec:ugmsjunctiontrees}

\begin{table}
\centering
    \begin{tabular}{l|l}
        Notation & Description  \\ \hline \hline
        $G^* = (V,E(G^*))$ & Unknown graph that we want to estimate.  \\ 
        %$\Xf^n$ & $n$ i.i.d. observations from $P_X$, which is Markov on $G^*$.  \\
		$H$ & Known graph such that $E(G^*) \subseteq E(H)$.\\
% ${\cal J} = ({\cal C},E({\cal J}))$ & Junction tree. \\
		${\cal G} = ({\cal R},\vec{E}({\cal G}))$ & Region graph of $H$ constructed using Algorithm~\ref{alg:constructregion}.\\
		${\cal R} = ({\cal R}^1,\ldots,{\cal R}^L)$ & Partitioning of the regions in ${\cal R}$ into $L$ labels. \\
		$\overline{R}$ & The set of vertices used when applying $\Psi$ to estimate edges over $R$.  \\
		&See (\ref{eq:rbar}) for definition. \\
		$H_R'$ & Edges in $H[R]$ that can be estimated using Algorithm~\ref{alg:ugmsregion}.  \\
		&See (\ref{eq:haprime}) for definition. \\
    \hline
    \end{tabular}
\smallskip
\caption{A summary of some notations.}
\label{tab:notations}
\end{table}

In this section, we present the junction tree framework for UGMS using the results from Sections~\ref{sec:regiongraphs}-\ref{sec:applyingUGMSSubset}. Section~\ref{subsec:description} presents the junction tree framework.  Section~\ref{subsec:compcomplexity} discusses the computational complexity of the framework.  Section~\ref{subsec:whyrg} highlights the advantages of using junction trees for UGMS using some examples.  We refer to Table~\ref{tab:notations} for a summary of all the notations that we use in this section.  

\subsection{Description of Framework}
\label{subsec:description}

Recall that Algorithm~\ref{alg:ugmsregion} shows that to estimate a subset of edges in $H[R]$, where $R$ is a region in the region graph ${\cal G}$, the UGMS algorithm $\Psi$ in Assumption~\ref{ass:cons} needs to be applied to the set $\overline{R}$ defined in (\ref{eq:rbar}).  Given this result, a straightforward approach to decomposing the UGMS problem is to apply Algorithm~\ref{alg:ugmsregion} to each region $R$ and combine all the estimated edges.  This will work since for any $R,S \in {\cal R}$ such that $R \ne S$, $E(H_R') \cap E(H_S') = \emptyset$.  This means that each application of Algorithm~\ref{alg:ugmsregion} estimates a different set of edges in the graph.  However, for some edges, this may require applying a UGMS algorithm to a large set of nodes.  For example, in Figure~\ref{fig:exampleregiongraph}(c), when applying Algorithm~\ref{alg:ugmsregion}  to $R = \{3,6\}$, the UGMS algorithm needs to be applied to $\overline{R} = \{2,3,4,5,6,7,8\}$, which is almost the full set of vertices.  To reduce the problem size of the subproblems, we  apply Algorithms~\ref{alg:constructregion} and \ref{alg:ugmsregion}  in an iterative manner as outlined in Algorithm~\ref{alg:jtframework}.

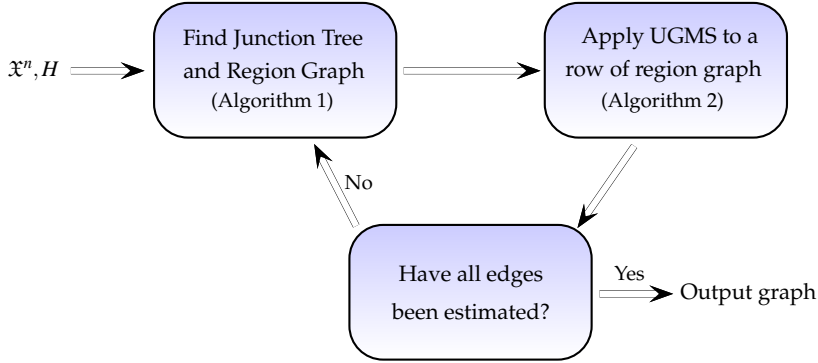
\begin{figure}[h]
\centering
\scalebox{0.7} % Change this value to rescale the drawing.
{
\begin{pspicture}(0,-3.41)(15.24,3.41)
\definecolor{color0b}{rgb}{0.8,0.8,1.0}
\psframe[linewidth=0.04,framearc=0.5,dimen=outer,fillstyle=gradient,gradlines=2000,gradbegin=color0b,gradend=white,gradmidpoint=1.0,fillcolor=color0b](14.78,3.41)(10.14,0.81)
\usefont{T1}{ppl}{m}{n}
\rput(12.54,2.715){\large Apply UGMS to a}
\usefont{T1}{ppl}{m}{n}
\rput(12.42,2.095){\large row of region graph}
\usefont{T1}{ppl}{m}{n}
\rput(12.4,1.49){(Algorithm 2)}
\psframe[linewidth=0.04,framearc=0.5,dimen=outer,fillstyle=gradient,gradlines=2000,gradbegin=color0b,gradend=white,gradmidpoint=1.0,fillcolor=color0b](7.36,3.41)(2.72,0.81)
\usefont{T1}{ppl}{m}{n}
\rput(5.0,2.705){\large Find Junction Tree}
\usefont{T1}{ppl}{m}{n}
\rput(5.02,2.065){\large and Region Graph}
\usefont{T1}{ppl}{m}{n}
\rput(4.99,1.5){(Algorithm 1)}
\psframe[linewidth=0.04,framearc=0.5,dimen=outer,fillstyle=gradient,gradlines=2000,gradbegin=color0b,gradend=white,gradmidpoint=1.0,fillcolor=color0b](11.07,-0.81)(6.43,-3.41)
\usefont{T1}{ppl}{m}{n}
\rput(8.76,-1.795){\large Have all edges}
\usefont{T1}{ppl}{m}{n}
\rput(8.76,-2.475){\large been estimated?}
\psline[linewidth=0.0080cm,arrowsize=0.05291667cm 2.0,arrowlength=1.4,arrowinset=0.4,doubleline=true,doublesep=0.12]{->}(7.48,2.1)(10.08,2.1)
\psline[linewidth=0.0080cm,arrowsize=0.05291667cm 2.0,arrowlength=1.4,arrowinset=0.4,doubleline=true,doublesep=0.12]{->}(11.86,0.67)(10.78,-0.89)
\psline[linewidth=0.0080cm,arrowsize=0.05291667cm 2.0,arrowlength=1.4,arrowinset=0.4,doubleline=true,doublesep=0.12]{->}(6.58,-0.85)(5.78,0.71)
\psline[linewidth=0.0080cm,arrowsize=0.05291667cm 2.0,arrowlength=1.4,arrowinset=0.4,doubleline=true,doublesep=0.12]{->}(1.16,2.1)(2.56,2.1)
\usefont{T1}{ptm}{m}{n}
\rput(0.45,2.1){\large $\Xf^n,H$}
\usefont{T1}{ppl}{m}{n}
\rput(14.03,-2.14){\large Output graph}
\psline[linewidth=0.0080cm,arrowsize=0.05291667cm 2.0,arrowlength=1.4,arrowinset=0.4,doubleline=true,doublesep=0.12]{->}(11.2,-2.14)(12.6,-2.14)
\usefont{T1}{ppl}{m}{n}
\rput(6.63,-0.02){No}
\usefont{T1}{ppl}{m}{n}
\rput(11.77,-1.74){Yes}
\end{pspicture} 
}
\caption{A high level overview of the junction tree framework for UGMS in Algorithm~\ref{alg:jtframework}.}
\label{fig:description}
\end{figure}

\begin{algorithm}[h]
\label{alg:jtframework}
\caption{Junction Tree Framework for UGMS}
See Table~\ref{tab:notations} for notations.
\begin{enumerate}[Step 1.]
 \item Initialize $\widehat{G}$ so that $E(\widehat{G}) = \emptyset$ and find the region graph ${\cal G}$ of $H$.
 \item Find the smallest $\ell$ such that there exists a region $R \in {\cal R}^{\ell}$ such that $E(H'_R) \ne \emptyset$.
 \item Apply Algorithm~\ref{alg:ugmsregion} to each region in ${\cal R}^{\ell}$.
 \item Add all estimated edges to $\widehat{G}$ and \textit{remove edges} from $H$ that have been estimated.  Now $H \cup \widehat{G}$ contains all the edges in $G^*$.
 \item Compute a new junction tree and region graph ${\cal G}$ using the graph $\widehat{G} \cup H$.
 \item If $E(H) = \emptyset$, stop the algorithm, else go to Step 2.
\end{enumerate}
\end{algorithm}

Figure~\ref{fig:description} shows a high level description of Algorithm~\ref{alg:jtframework}.  We first find a junction tree and then a region graph of the graph $H$ using Algorithm~\ref{alg:constructregion}.  We then find the row in the region graph over which edges can be estimated and apply Algorithm~\ref{alg:ugmsregion} to each region in that row. We note that when estimating edges over a region, we use model selection algorithms to choose an appropriate regularization parameter to select the number of edges to estimate.   Next, all estimated edges are added to $\widehat{G}$ and all edges that are estimated are removed from $H$.  Thus, $H$ now represents all the edges that are left to be estimated and $\widehat{G} \cup H$ contains all the edges in $G^*$.  We repeat the above steps on a new region graph computed using $\widehat{G} \cup H$ and stop the algorithm when $H$ is an empty graph.

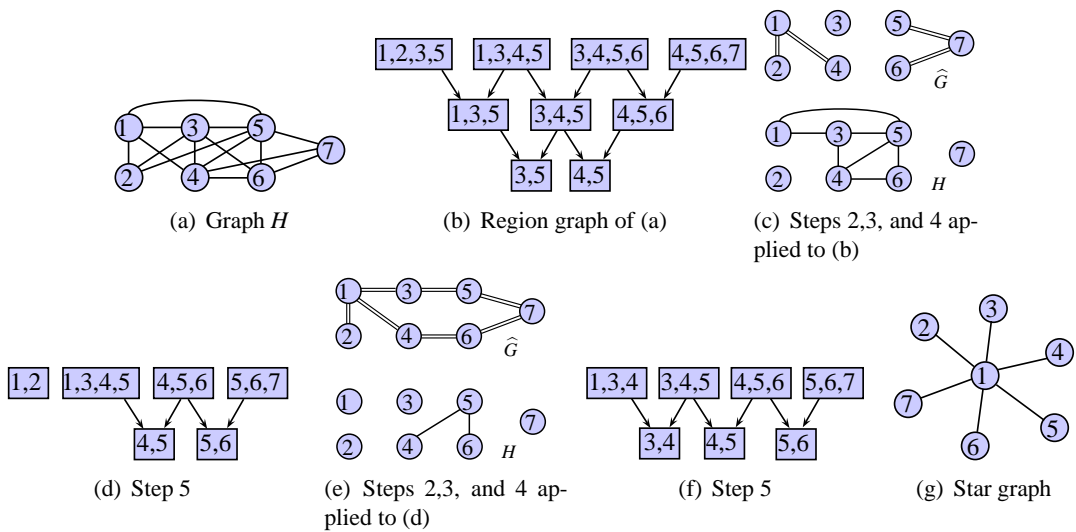
\begin{figure}
\centering
\subfigure[Graph $H$]{
\scalebox{0.55} % Change this value to rescale the drawing.
{
\begin{pspicture}(0,-1.165)(5.58,1.185)
\definecolor{color2875b}{rgb}{0.8,0.8,1.0}
\pscircle[linewidth=0.04,dimen=outer,fillstyle=solid,fillcolor=color2875b](0.35,0.385){0.35}
\usefont{T1}{ptm}{m}{n}
\rput(0.18859375,0.3949999){\LARGE 1}
\pscircle[linewidth=0.04,dimen=outer,fillstyle=solid,fillcolor=color2875b](0.35,-0.81499994){0.35}
\usefont{T1}{ptm}{m}{n}
\rput(0.31093752,-0.8249999){\LARGE 2}
\pscircle[linewidth=0.04,dimen=outer,fillstyle=solid,fillcolor=color2875b](1.95,-0.815){0.35}
\usefont{T1}{ptm}{m}{n}
\rput(1.9064062,-0.805){\LARGE 4}
\pscircle[linewidth=0.04,dimen=outer,fillstyle=solid,fillcolor=color2875b](1.95,0.38499996){0.35}
\usefont{T1}{ptm}{m}{n}
\rput(1.8943751,0.395){\LARGE 3}
\pscircle[linewidth=0.04,dimen=outer,fillstyle=solid,fillcolor=color2875b](3.55,-0.815){0.35}
\usefont{T1}{ptm}{m}{n}
\rput(3.480781,-0.825){\LARGE 6}
\pscircle[linewidth=0.04,dimen=outer,fillstyle=solid,fillcolor=color2875b](3.55,0.38499996){0.35}
\usefont{T1}{ptm}{m}{n}
\rput(3.4821875,0.395){\LARGE 5}
\psline[linewidth=0.04cm](0.68,0.395)(1.62,0.395)
\psline[linewidth=0.04cm](2.28,0.395)(3.22,0.395)
\psline[linewidth=0.04cm](2.28,-0.8249999)(3.22,-0.8249999)
\psline[linewidth=0.04cm](1.94,0.05500004)(1.94,-0.48499998)
\psline[linewidth=0.04cm](3.54,0.05500004)(3.54,-0.48499998)
\pscircle[linewidth=0.04,dimen=outer,fillstyle=solid,fillcolor=color2875b](5.23,-0.15500002){0.35}
\usefont{T1}{ptm}{m}{n}
\rput(5.168906,-0.16500002){\LARGE 7}
\psline[linewidth=0.04cm](0.34,0.05500004)(0.34,-0.48499998)
\psline[linewidth=0.04cm](0.58,0.15500003)(1.66,-0.62499994)
\psline[linewidth=0.04cm](0.56,-0.56499994)(1.7,0.19500002)
\psline[linewidth=0.04cm](0.64,-0.66499996)(3.22,0.29500005)
\psline[linewidth=0.04cm](2.18,0.17500003)(3.28,-0.60499996)
\psline[linewidth=0.04cm](2.24,-0.66499996)(4.9,-0.12499997)
\psline[linewidth=0.04cm](2.14,-0.54499996)(3.3,0.17500003)
\psline[linewidth=0.04cm](3.86,0.29500005)(4.94,-0.00499998)
\psline[linewidth=0.04cm](3.86,-0.68499994)(4.92,-0.28499997)
\psbezier[linewidth=0.04](0.3404207,0.7040723)(0.34,1.1639334)(3.5395792,1.165)(3.5399997,0.70513886)
\end{pspicture} 
}
}
\subfigure[Region graph of (a)]{
\scalebox{0.55} % Change this value to rescale the drawing.
{
\begin{pspicture}(0,-1.86)(8.98,1.86)
\definecolor{color2875b}{rgb}{0.8,0.8,1.0}
\psframe[linewidth=0.04,dimen=outer,fillstyle=solid,fillcolor=color2875b](1.88,1.86)(0.0,1.08)
\usefont{T1}{ptm}{m}{n}
\rput(0.8720313,1.46){\LARGE 1,2,3,5}
\psframe[linewidth=0.04,dimen=outer,fillstyle=solid,fillcolor=color2875b](4.2466664,1.86)(2.3666666,1.08)
\usefont{T1}{ptm}{m}{n}
\rput(3.238698,1.46){\LARGE 1,3,4,5}
\psframe[linewidth=0.04,dimen=outer,fillstyle=solid,fillcolor=color2875b](6.613333,1.86)(4.733333,1.08)
\usefont{T1}{ptm}{m}{n}
\rput(5.636615,1.46){\LARGE 3,4,5,6}
\psframe[linewidth=0.04,dimen=outer,fillstyle=solid,fillcolor=color2875b](8.98,1.86)(7.1,1.08)
\usefont{T1}{ptm}{m}{n}
\rput(8.009063,1.46){\LARGE 4,5,6,7}
\psframe[linewidth=0.04,dimen=outer,fillstyle=solid,fillcolor=color2875b](3.24,0.39)(1.76,-0.39)
\usefont{T1}{ptm}{m}{n}
\rput(2.4220312,-0.01){\LARGE 1,3,5}
\psframe[linewidth=0.04,dimen=outer,fillstyle=solid,fillcolor=color2875b](5.23,0.39)(3.75,-0.39)
\usefont{T1}{ptm}{m}{n}
\rput(4.455,-0.01){\LARGE 3,4,5}
\psframe[linewidth=0.04,dimen=outer,fillstyle=solid,fillcolor=color2875b](7.22,0.39)(5.74,-0.39)
\usefont{T1}{ptm}{m}{n}
\rput(6.460781,-0.01){\LARGE 4,5,6}
\psframe[linewidth=0.04,dimen=outer,fillstyle=solid,fillcolor=color2875b](4.28,-1.08)(3.28,-1.86)
\usefont{T1}{ptm}{m}{n}
\rput(3.755,-1.48){\LARGE 3,5}
\psframe[linewidth=0.04,dimen=outer,fillstyle=solid,fillcolor=color2875b](5.7,-1.08)(4.7,-1.86)
\usefont{T1}{ptm}{m}{n}
\rput(5.1825,-1.48){\LARGE 4,5}
\psline[linewidth=0.04cm,arrowsize=0.05291667cm 3.0,arrowlength=1.4,arrowinset=0.4]{->}(3.12,1.1)(2.72,0.4)
\psline[linewidth=0.04cm,arrowsize=0.05291667cm 3.0,arrowlength=1.4,arrowinset=0.4]{->}(3.68,1.1)(4.14,0.38)
\psline[linewidth=0.04cm,arrowsize=0.05291667cm 3.0,arrowlength=1.4,arrowinset=0.4]{->}(5.2,1.08)(4.8,0.38)
\psline[linewidth=0.04cm,arrowsize=0.05291667cm 3.0,arrowlength=1.4,arrowinset=0.4]{->}(6.1,1.1)(6.56,0.38)
\psline[linewidth=0.04cm,arrowsize=0.05291667cm 3.0,arrowlength=1.4,arrowinset=0.4]{->}(7.38,1.08)(6.98,0.38)
\psline[linewidth=0.04cm,arrowsize=0.05291667cm 3.0,arrowlength=1.4,arrowinset=0.4]{->}(1.62,1.08)(2.08,0.36)
\psline[linewidth=0.04cm,arrowsize=0.05291667cm 3.0,arrowlength=1.4,arrowinset=0.4]{->}(3.04,-0.38)(3.5,-1.1)
\psline[linewidth=0.04cm,arrowsize=0.05291667cm 3.0,arrowlength=1.4,arrowinset=0.4]{->}(4.58,-0.38)(5.04,-1.1)
\psline[linewidth=0.04cm,arrowsize=0.05291667cm 3.0,arrowlength=1.4,arrowinset=0.4]{->}(5.92,-0.38)(5.52,-1.08)
\psline[linewidth=0.04cm,arrowsize=0.05291667cm 3.0,arrowlength=1.4,arrowinset=0.4]{->}(4.4,-0.38)(4.0,-1.08)
\end{pspicture}}} 
\subfigure[Steps 2,3, and 4 applied to (b)]{
% Generated with LaTeXDraw 2.0.8
% Thu Aug 23 20:05:03 CDT 2012
% \usepackage[usenames,dvipsnames]{pstricks}
% \usepackage{epsfig}
% \usepackage{pst-grad} % For gradients
% \usepackage{pst-plot} % For axes
\scalebox{0.5} % Change this value to rescale the drawing.
{
\begin{pspicture}(0,-2.42)(5.6,2.42)
\definecolor{color2875b}{rgb}{0.8,0.8,1.0}
\pscircle[linewidth=0.04,dimen=outer,fillstyle=solid,fillcolor=color2875b](0.35,2.07){0.35}
\usefont{T1}{ptm}{m}{n}
\rput(0.28515625,2.06){\LARGE 1}
\pscircle[linewidth=0.04,dimen=outer,fillstyle=solid,fillcolor=color2875b](0.35,0.87000006){0.35}
\usefont{T1}{ptm}{m}{n}
\rput(0.333125,0.8600001){\LARGE 2}
\pscircle[linewidth=0.04,dimen=outer,fillstyle=solid,fillcolor=color2875b](1.95,0.87){0.35}
\usefont{T1}{ptm}{m}{n}
\rput(1.9160937,0.88){\LARGE 4}
\pscircle[linewidth=0.04,dimen=outer,fillstyle=solid,fillcolor=color2875b](1.95,2.07){0.35}
\usefont{T1}{ptm}{m}{n}
\rput(1.9403125,2.06){\LARGE 3}
\pscircle[linewidth=0.04,dimen=outer,fillstyle=solid,fillcolor=color2875b](3.55,0.87){0.35}
\usefont{T1}{ptm}{m}{n}
\rput(3.5117185,0.86){\LARGE 6}
\pscircle[linewidth=0.04,dimen=outer,fillstyle=solid,fillcolor=color2875b](3.55,2.07){0.35}
\usefont{T1}{ptm}{m}{n}
\rput(3.53375,2.06){\LARGE 5}
\pscircle[linewidth=0.04,dimen=outer,fillstyle=solid,fillcolor=color2875b](5.23,1.53){0.35}
\usefont{T1}{ptm}{m}{n}
\rput(5.200156,1.5){\LARGE 7}
\psline[linewidth=0.03cm](0.3,1.74)(0.3,1.2)
\psline[linewidth=0.03cm](0.62,1.88)(1.7,1.1)
\psline[linewidth=0.03cm](3.86,2.0)(4.94,1.7)
\psline[linewidth=0.03cm](3.88,0.94)(4.94,1.34)
\psline[linewidth=0.03cm](0.38,1.74)(0.38,1.2)
\psline[linewidth=0.03cm](0.56,1.82)(1.64,1.04)
\psline[linewidth=0.03cm](3.86,1.02)(4.92,1.4200001)
\psline[linewidth=0.03cm](3.82,1.9200001)(4.9,1.62)
\pscircle[linewidth=0.04,dimen=outer,fillstyle=solid,fillcolor=color2875b](0.37,-0.87){0.35}
\usefont{T1}{ptm}{m}{n}
\rput(0.30515626,-0.8800001){\LARGE 1}
\pscircle[linewidth=0.04,dimen=outer,fillstyle=solid,fillcolor=color2875b](0.37,-2.07){0.35}
\usefont{T1}{ptm}{m}{n}
\rput(0.353125,-2.08){\LARGE 2}
\pscircle[linewidth=0.04,dimen=outer,fillstyle=solid,fillcolor=color2875b](1.97,-2.07){0.35}
\usefont{T1}{ptm}{m}{n}
\rput(1.9360937,-2.06){\LARGE 4}
\pscircle[linewidth=0.04,dimen=outer,fillstyle=solid,fillcolor=color2875b](1.97,-0.87000006){0.35}
\usefont{T1}{ptm}{m}{n}
\rput(1.9603126,-0.88){\LARGE 3}
\pscircle[linewidth=0.04,dimen=outer,fillstyle=solid,fillcolor=color2875b](3.57,-2.07){0.35}
\usefont{T1}{ptm}{m}{n}
\rput(3.5317185,-2.08){\LARGE 6}
\pscircle[linewidth=0.04,dimen=outer,fillstyle=solid,fillcolor=color2875b](3.57,-0.87000006){0.35}
\usefont{T1}{ptm}{m}{n}
\rput(3.55375,-0.88){\LARGE 5}
\psline[linewidth=0.04cm](0.7,-0.86)(1.64,-0.86)
\psline[linewidth=0.04cm](2.3,-0.86)(3.24,-0.86)
\psline[linewidth=0.04cm](2.3,-2.08)(3.24,-2.08)
\psline[linewidth=0.04cm](1.96,-1.1999999)(1.96,-1.74)
\psline[linewidth=0.04cm](3.56,-1.1999999)(3.56,-1.74)
\pscircle[linewidth=0.04,dimen=outer,fillstyle=solid,fillcolor=color2875b](5.25,-1.41){0.35}
\usefont{T1}{ptm}{m}{n}
\rput(5.220156,-1.44){\LARGE 7}
\psline[linewidth=0.04cm](2.16,-1.8)(3.32,-1.0799999)
\psbezier[linewidth=0.04](0.40042067,-0.5609277)(0.4,-0.1010666)(3.599579,-0.1)(3.5999997,-0.5598612)
\usefont{T1}{ptm}{m}{n}
\rput(4.6815624,0.585){\Large $\widehat{G}$}
\usefont{T1}{ptm}{m}{n}
\rput(4.644375,-2.215){\Large $H$}
\end{pspicture} 
}}
\subfigure[Step 5]{
% Generated with LaTeXDraw 2.0.8
% Thu Aug 23 19:37:27 CDT 2012
% \usepackage[usenames,dvipsnames]{pstricks}
% \usepackage{epsfig}
% \usepackage{pst-grad} % For gradients
% \usepackage{pst-plot} % For axes
\scalebox{0.55} % Change this value to rescale the drawing.
{
\begin{pspicture}(0,-1.13)(6.78,1.13)
\definecolor{color2875b}{rgb}{0.8,0.8,1.0}
\psframe[linewidth=0.04,dimen=outer,fillstyle=solid,fillcolor=color2875b](3.1933331,1.13)(1.3133334,0.35)
\usefont{T1}{ptm}{m}{n}
\rput(2.130365,0.73){\LARGE 1,3,4,5}
\psframe[linewidth=0.04,dimen=outer,fillstyle=solid,fillcolor=color2875b](4.9866667,1.13)(3.5066667,0.35)
\usefont{T1}{ptm}{m}{n}
\rput(4.2111974,0.73){\LARGE 4,5,6}
\psframe[linewidth=0.04,dimen=outer,fillstyle=solid,fillcolor=color2875b](1.0,1.13)(0.0,0.35)
\usefont{T1}{ptm}{m}{n}
\rput(0.42734376,0.73){\LARGE 1,2}
\psframe[linewidth=0.04,dimen=outer,fillstyle=solid,fillcolor=color2875b](4.04,-0.35)(3.04,-1.13)
\usefont{T1}{ptm}{m}{n}
\rput(3.4979687,-0.75){\LARGE 4,5}
\psline[linewidth=0.04cm,arrowsize=0.05291667cm 3.0,arrowlength=1.4,arrowinset=0.4]{->}(2.82,0.35)(3.28,-0.37)
\psline[linewidth=0.04cm,arrowsize=0.05291667cm 3.0,arrowlength=1.4,arrowinset=0.4]{->}(4.36,0.35)(4.82,-0.37)
\psline[linewidth=0.04cm,arrowsize=0.05291667cm 3.0,arrowlength=1.4,arrowinset=0.4]{->}(5.7,0.35)(5.3,-0.35)
\psline[linewidth=0.04cm,arrowsize=0.05291667cm 3.0,arrowlength=1.4,arrowinset=0.4]{->}(4.18,0.35)(3.78,-0.35)
\psframe[linewidth=0.04,dimen=outer,fillstyle=solid,fillcolor=color2875b](6.78,1.13)(5.3,0.35)
\usefont{T1}{ptm}{m}{n}
\rput(5.993125,0.73){\LARGE 5,6,7}
\psframe[linewidth=0.04,dimen=outer,fillstyle=solid,fillcolor=color2875b](5.56,-0.35)(4.56,-1.13)
\usefont{T1}{ptm}{m}{n}
\rput(5.0165625,-0.75){\LARGE 5,6}
\end{pspicture}}} \quad
\subfigure[Steps 2,3, and 4 applied to (d)]{
% Generated with LaTeXDraw 2.0.8
% Thu Aug 23 20:27:37 CDT 2012
% \usepackage[usenames,dvipsnames]{pstricks}
% \usepackage{epsfig}
% \usepackage{pst-grad} % For gradients
% \usepackage{pst-plot} % For axes
\scalebox{0.5} % Change this value to rescale the drawing.
{
\begin{pspicture}(0,-2.4365625)(6.1775,2.42)
\definecolor{color2875b}{rgb}{0.8,0.8,1.0}
\pscircle[linewidth=0.04,dimen=outer,fillstyle=solid,fillcolor=color2875b](0.35,2.07){0.35}
\usefont{T1}{ptm}{m}{n}
\rput(0.2134375,2.06){\LARGE 1}
\pscircle[linewidth=0.04,dimen=outer,fillstyle=solid,fillcolor=color2875b](0.35,0.87000006){0.35}
\usefont{T1}{ptm}{m}{n}
\rput(0.31421876,0.8600001){\LARGE 2}
\pscircle[linewidth=0.04,dimen=outer,fillstyle=solid,fillcolor=color2875b](1.95,0.87){0.35}
\usefont{T1}{ptm}{m}{n}
\rput(1.9009374,0.88){\LARGE 4}
\pscircle[linewidth=0.04,dimen=outer,fillstyle=solid,fillcolor=color2875b](1.95,2.07){0.35}
\usefont{T1}{ptm}{m}{n}
\rput(1.9032812,2.06){\LARGE 3}
\pscircle[linewidth=0.04,dimen=outer,fillstyle=solid,fillcolor=color2875b](3.55,0.87){0.35}
\usefont{T1}{ptm}{m}{n}
\rput(3.4871871,0.86){\LARGE 6}
\pscircle[linewidth=0.04,dimen=outer,fillstyle=solid,fillcolor=color2875b](3.55,2.07){0.35}
\usefont{T1}{ptm}{m}{n}
\rput(3.4995313,2.06){\LARGE 5}
\pscircle[linewidth=0.04,dimen=outer,fillstyle=solid,fillcolor=color2875b](5.23,1.53){0.35}
\usefont{T1}{ptm}{m}{n}
\rput(5.175781,1.5){\LARGE 7}
\psline[linewidth=0.03cm](0.3,1.74)(0.3,1.2)
\psline[linewidth=0.03cm](0.62,1.88)(1.7,1.1)
\psline[linewidth=0.03cm](3.86,2.0)(4.94,1.7)
\psline[linewidth=0.03cm](3.88,0.94)(4.94,1.34)
\psline[linewidth=0.03cm](0.38,1.74)(0.38,1.2)
\psline[linewidth=0.03cm](0.56,1.82)(1.64,1.04)
\psline[linewidth=0.03cm](3.86,1.02)(4.92,1.4200001)
\psline[linewidth=0.03cm](3.82,1.9200001)(4.9,1.62)
\pscircle[linewidth=0.04,dimen=outer,fillstyle=solid,fillcolor=color2875b](0.37,-0.87){0.35}
\usefont{T1}{ptm}{m}{n}
\rput(0.23343751,-0.8800001){\LARGE 1}
\pscircle[linewidth=0.04,dimen=outer,fillstyle=solid,fillcolor=color2875b](0.37,-2.07){0.35}
\usefont{T1}{ptm}{m}{n}
\rput(0.33421874,-2.08){\LARGE 2}
\pscircle[linewidth=0.04,dimen=outer,fillstyle=solid,fillcolor=color2875b](1.97,-2.07){0.35}
\usefont{T1}{ptm}{m}{n}
\rput(1.9209374,-2.06){\LARGE 4}
\pscircle[linewidth=0.04,dimen=outer,fillstyle=solid,fillcolor=color2875b](1.97,-0.87000006){0.35}
\usefont{T1}{ptm}{m}{n}
\rput(1.9232813,-0.88){\LARGE 3}
\pscircle[linewidth=0.04,dimen=outer,fillstyle=solid,fillcolor=color2875b](3.57,-2.07){0.35}
\usefont{T1}{ptm}{m}{n}
\rput(3.5071874,-2.08){\LARGE 6}
\pscircle[linewidth=0.04,dimen=outer,fillstyle=solid,fillcolor=color2875b](3.57,-0.87000006){0.35}
\usefont{T1}{ptm}{m}{n}
\rput(3.5195312,-0.88){\LARGE 5}
\psline[linewidth=0.03cm](2.28,2.14)(3.22,2.14)
\psline[linewidth=0.04cm](3.56,-1.1999999)(3.56,-1.74)
\pscircle[linewidth=0.04,dimen=outer,fillstyle=solid,fillcolor=color2875b](5.25,-1.41){0.35}
\usefont{T1}{ptm}{m}{n}
\rput(5.195781,-1.44){\LARGE 7}
\psline[linewidth=0.04cm](2.16,-1.8)(3.32,-1.0799999)
\usefont{T1}{ptm}{m}{n}
\rput(4.656875,0.585){\Large $\widehat{G}$}
\usefont{T1}{ptm}{m}{n}
\rput(4.6196876,-2.215){\Large $H$}
\psline[linewidth=0.03cm](2.28,2.06)(3.22,2.06)
\psline[linewidth=0.03cm](0.66,2.14)(1.6,2.14)
\psline[linewidth=0.03cm](0.66,2.06)(1.6,2.06)
\psline[linewidth=0.03cm](2.28,0.9)(3.22,0.9)
\psline[linewidth=0.03cm](2.28,0.82)(3.22,0.82)
\end{pspicture}}}
\subfigure[Step 5]{
% Generated with LaTeXDraw 2.0.8
% Thu Aug 23 14:40:35 CDT 2012
% \usepackage[usenames,dvipsnames]{pstricks}
% \usepackage{epsfig}
% \usepackage{pst-grad} % For gradients
% \usepackage{pst-plot} % For axes
\scalebox{0.55} % Change this value to rescale the drawing.
{
\begin{pspicture}(0,-1.13)(6.713333,1.13)
\definecolor{color2875b}{rgb}{0.8,0.8,1.0}
\psframe[linewidth=0.04,dimen=outer,fillstyle=solid,fillcolor=color2875b](4.9688888,1.13)(3.488889,0.35)
\usefont{T1}{ptm}{m}{n}
\rput(4.17717,0.73){\LARGE 4,5,6}
\psframe[linewidth=0.04,dimen=outer,fillstyle=solid,fillcolor=color2875b](2.2733333,-0.33)(1.2733333,-1.11)
\usefont{T1}{ptm}{m}{n}
\rput(1.7686458,-0.73){\LARGE 3,4}
\psline[linewidth=0.04cm,arrowsize=0.05291667cm 3.0,arrowlength=1.4,arrowinset=0.4]{->}(1.0933332,0.37)(1.5533333,-0.35)
\psline[linewidth=0.04cm,arrowsize=0.05291667cm 3.0,arrowlength=1.4,arrowinset=0.4]{->}(2.5933332,0.37)(3.0533333,-0.35)
\psline[linewidth=0.04cm,arrowsize=0.05291667cm 3.0,arrowlength=1.4,arrowinset=0.4]{->}(4.0133333,0.37)(3.6133332,-0.33)
\psline[linewidth=0.04cm,arrowsize=0.05291667cm 3.0,arrowlength=1.4,arrowinset=0.4]{->}(2.4133334,0.37)(2.0133333,-0.33)
\psframe[linewidth=0.04,dimen=outer,fillstyle=solid,fillcolor=color2875b](6.713333,1.13)(5.233333,0.35)
\usefont{T1}{ptm}{m}{n}
\rput(5.8988023,0.73){\LARGE 5,6,7}
\psframe[linewidth=0.04,dimen=outer,fillstyle=solid,fillcolor=color2875b](3.8433332,-0.33)(2.8433332,-1.11)
\usefont{T1}{ptm}{m}{n}
\rput(3.2753646,-0.73){\LARGE 4,5}
\psframe[linewidth=0.04,dimen=outer,fillstyle=solid,fillcolor=color2875b](1.48,1.13)(0.0,0.35)
\usefont{T1}{ptm}{m}{n}
\rput(0.6589057,0.73){\LARGE 1,3,4}
\psframe[linewidth=0.04,dimen=outer,fillstyle=solid,fillcolor=color2875b](3.2244444,1.13)(1.7444445,0.35)
\usefont{T1}{ptm}{m}{n}
\rput(2.405538,0.73){\LARGE 3,4,5}
\psframe[linewidth=0.04,dimen=outer,fillstyle=solid,fillcolor=color2875b](5.5733333,-0.35)(4.5733333,-1.13)
\usefont{T1}{ptm}{m}{n}
\rput(5.003958,-0.75){\LARGE 5,6}
\psline[linewidth=0.04cm,arrowsize=0.05291667cm 3.0,arrowlength=1.4,arrowinset=0.4]{->}(4.4533334,0.37)(4.9133334,-0.35)
\psline[linewidth=0.04cm,arrowsize=0.05291667cm 3.0,arrowlength=1.4,arrowinset=0.4]{->}(5.633333,0.37)(5.233333,-0.33)
\end{pspicture} 
}
}
\subfigure[Star graph]{
% Generated with LaTeXDraw 2.0.8
% Tue Aug 28 23:00:11 CDT 2012
% \usepackage[usenames,dvipsnames]{pstricks}
% \usepackage{epsfig}
% \usepackage{pst-grad} % For gradients
% \usepackage{pst-plot} % For axes
\scalebox{0.55} % Change this value to rescale the drawing.
{
\begin{pspicture}(0,-2.01)(4.36,2.01)
\definecolor{color2875b}{rgb}{0.8,0.8,1.0}
\pscircle[linewidth=0.04,dimen=outer,fillstyle=solid,fillcolor=color2875b](2.21,0.04000002){0.35}
\usefont{T1}{ptm}{m}{n}
\rput(2.136875,0.02999992){\LARGE 1}
\pscircle[linewidth=0.04,dimen=outer,fillstyle=solid,fillcolor=color2875b](0.75,1.22){0.35}
\usefont{T1}{ptm}{m}{n}
\rput(0.7320313,1.2100002){\LARGE 2}
\pscircle[linewidth=0.04,dimen=outer,fillstyle=solid,fillcolor=color2875b](4.01,0.6){0.35}
\usefont{T1}{ptm}{m}{n}
\rput(3.97125,0.61){\LARGE 4}
\pscircle[linewidth=0.04,dimen=outer,fillstyle=solid,fillcolor=color2875b](2.41,1.66){0.35}
\usefont{T1}{ptm}{m}{n}
\rput(2.377344,1.65){\LARGE 3}
\pscircle[linewidth=0.04,dimen=outer,fillstyle=solid,fillcolor=color2875b](1.97,-1.66){0.35}
\usefont{T1}{ptm}{m}{n}
\rput(1.9162498,-1.67){\LARGE 6}
\pscircle[linewidth=0.04,dimen=outer,fillstyle=solid,fillcolor=color2875b](3.89,-1.22){0.35}
\usefont{T1}{ptm}{m}{n}
\rput(3.8479688,-1.25){\LARGE 5}
\pscircle[linewidth=0.04,dimen=outer,fillstyle=solid,fillcolor=color2875b](0.35,-0.64){0.35}
\usefont{T1}{ptm}{m}{n}
\rput(0.324531,-0.65){\LARGE 7}
\psline[linewidth=0.04cm](2.24,0.38000003)(2.36,1.36)
\psline[linewidth=0.04cm](2.5,0.18000002)(3.72,0.46)
\psline[linewidth=0.04cm](2.48,-0.17999998)(3.64,-1.04)
\psline[linewidth=0.04cm](2.16,-0.27999997)(2.02,-1.3199999)
\psline[linewidth=0.04cm](1.9,-0.03999998)(0.66,-0.49999997)
\psline[linewidth=0.04cm](1.0,1.04)(1.94,0.26000002)
\end{pspicture} 
}
}
\caption{Example to illustrate the junction tree framework in Algorithm~\ref{alg:jtframework}.}
\label{fig:exampleFinalFramework}
\end{figure}

An example illustrating the junction tree framework is shown in Figure~\ref{fig:exampleFinalFramework}.  The region graph in Figure~\ref{fig:exampleFinalFramework}(b) is constructed using the graph $H$ in Figure~\ref{fig:exampleFinalFramework}(a). The true graph $G^*$ we want to estimate is shown in Figure~\ref{fig:FirstExample}(a).  The top and bottom in Figure~\ref{fig:exampleFinalFramework}(c) show the graphs $\widehat{G}$ and $H$, respectively, after estimating all the edges in ${\cal R}^1$ of Figure~\ref{fig:exampleFinalFramework}(b).  The edges in $\widehat{G}$ are represented by double lines to distinguish them from the edges in $H$.  Figure~\ref{fig:exampleFinalFramework}(d) shows the region graph of $\widehat{G} \cup H$.  Figure~\ref{fig:exampleFinalFramework}(e) shows the updated $\widehat{G}$ and $H$ where only the edges $(4,5)$ and $(5,6)$ are left to be estimated.  This is done by applying Algorithm~\ref{alg:ugmsregion}  to the regions in ${\cal R}^2$ of Figure~\ref{fig:exampleFinalFramework}(f).  Notice that we did not include the region $\{1,2\}$ in the last region graph since we know all edges in this region have already been estimated.  In general, if $E(H[R]) = \emptyset$ for any region $R$, we can remove this region and thereby reduce the computational complexity of constructing region graphs.

\subsection{Computational Complexity}
\label{subsec:compcomplexity}

In this section, we discuss the computational complexity of the junction tree framework.  It is difficult to write down a closed form expression since the computational complexity depends on the structure of the junction tree.  Moreover, merging clusters in the junction tree can easily control the computations.  With this in mind, the main aim in this section is to show that the complexity of the framework is roughly the same as that of applying a standard UGMS algorithm.  Consider the following observations.

\begin{enumerate}

\item \textit{Computing $H$:}  Assuming no prior knowledge about $H$ is given, this graph needs to be computed from the observations.  This can be done using standard screening algorithms, such as those in \citet{fan2008sure,VatsMuG2012}, or by applying a UGMS algorithm with a regularization parameter that selects a larger number of edges (than that computed by using a standard UGMS algorithm).  Thus, the complexity of computing $H$ is roughly the same as that of applying a UGMS algorithm to all the vertices in the graph.

\item \textit{Applying UGMS to regions:} Recall from Algorithm~\ref{alg:ugmsregion} that we apply a UGMS algorithm to observations over $\overline{R}$ to estimate edges over the vertices $R$, where $R$ is a region in a region graph representation of $H$.  Since $|\overline{R}| \le p$, it is clear that the complexity of Algorithm~\ref{alg:ugmsregion} is less than that of applying a UGMS algorithm to estimate all edges in the graph.

\item \textit{Computing junction trees:}  For a given graph, there exists several junction tree representations.  The computational complexity of applying UGMS algorithms to a junction tree depends on the size of the clusters, the size of the separators, and the degree of the junction tree.  In theory, it is useful to select a junction tree so that the overall computational complexity of the framework is as small as possible.  However, this is hard since there can be an exponential number of possible junction tree representations.  Alternatively, we can select a junction tree so that the maximum size of the clusters is as small as possible.  Such junction trees are often referred to as \textit{optimal junction trees} in the literature.
Although finding optimal junction trees is also hard \citep{arnborg1987complexity}, there exists several computationally tractable heuristics for finding close to optimal junction trees \citep{Kjaerulff1990,berry2003minimum}.  The complexity of such algorithms range from $O(p^2)$ to $O(p^3)$, depending on the degree of approximation.  We note that this time complexity is less than that of standard UGMS algorithms.
\end{enumerate}

It is clear that the complexity of all the intermediate steps in the framework is less than that of applying a standard UGMS algorithm.  The overall complexity of the framework depends on the number of clusters in the junction tree and the size of the separators in the junction tree.  The size of the separators in a junction tree can be controlled by merging clusters that share a large separator.  This step can be done in linear time.  Removing large separators also reduces the total number of clusters in a junction tree.  In the worst case, if all the separators in $H$ are too large, the junction tree will only have one cluster that contains all the vertices.  In this case, using the junction tree framework will be no different than using a standard UGMS algorithm.  

%Finally, it is clear that if $G^*$ contains at least one small separator and $H$ can identify this small separator, the junction tree framework can decompose the UGMS problem into multiple subproblems.

\subsection{Advantages of using Junction Trees and Region Graphs}
\label{subsec:whyrg}

An alternative approach to estimating $G^*$ using $H$ is to modify some current UGMS algorithms (see Appendix~\ref{app:examples} for some concrete examples).  For example, neighborhood selection based algorithms first estimate the neighborhood of each vertex and then combine all the estimated neighborhoods to construct an estimate $\widehat{G}$ of $G^*$ \citep{NicolaiPeter2006,BreslerMosselSly2008,NetrapalliSanghaviAllerton2010,RavikumarWainwrightLafferty2010}.  Two ways in which these algorithms can be modified when given $H$ are described as follows:

\begin{enumerate}
\item A straightforward approach is to decompose the UGMS problem into $p$ different subproblems of estimating the neighborhood of each vertex.  The graph $H$ can be used to restrict the estimated neighbors of each vertex to be subsets of the neighbors in $H$.  For example, in Figure~\ref{fig:exampleFinalFramework}(a), the neighborhood of $1$ is estimated from the set $\{2,3,4,5\}$ and the neighborhood of $3$ is estimated from the set $\{1,4,5,6\}$.  This approach can be compared to independently applying Algorithm~\ref{alg:ugmsregion}  to each region in the region graph.  For example, when using the region graph, the edge $(1,4)$ can be estimated by applying a UGMS algorithm to $\{1,3,4,5\}$.  In comparison, when not using region graphs, the edge $(1,4)$ is estimated by applying a UGMS algorithm to $\{1,2,3,4,5\}$.  In general, using region graphs results in smaller subproblems.  A good example to illustrate this is the star graph in Figure~\ref{fig:exampleFinalFramework}(g).  A junction tree representation of the star graph can be computed so that all clusters will have size two.  Subsequently, the junction tree framework will only require applying a UGMS algorithm to a pair of nodes.  On the other hand, neighborhood selection needs to be applied to all the nodes to estimate the neighbors of the central node $1$ which is connected to all other nodes.

\item An alternative approach is to estimate the neighbors of each vertex in an iterative manner.  However, it is not clear what ordering should be chosen for the vertices.  The region graph approach outlined in Section~\ref{subsec:description} leads to a natural choice for choosing which edges to estimate in the graph so as to reduce the problem size of subsequent subproblems.  Moreover, iteratively applying neighborhood selection may still lead to large subproblems.  For example, suppose the star graph in Figure~\ref{fig:exampleFinalFramework}(g) is in fact the true graph.  In this case, using neighborhood selection always leads to applying UGMS to all the nodes in the graph.
\end{enumerate}

From the above discussion, it is clear that using junction trees for UGMS leads to smaller subproblems and a natural choice of an ordering for estimating edges in the graph.  We will see in Section~\ref{sec:theoreticalPC} that the smaller subproblems lead to weaker conditions on the number of observations required for consistent graph estimation.  Moreover, our numerical simulations in Section~\ref{sec:numericalSim} empirically show the advantages of using junction tree over neighborhood selection based algorithms.

\section{PC-Algorithm for UGMS}
\label{sec:fPC}

So far, we have presented the junction tree framework using an abstract undirected graphical model selection (UMGS) algorithm.  This shows that our framework can be used in conjunction with any UGMS algorithm.  In this section, we review the PC-Algorithm, since we use it to analyze the junction tree framework in Section~\ref{sec:theoreticalPC}.  The PC-Algorithm was originally proposed in the literature for learning directed graphical models \citep{PCAlgorithm}.  The first stage of the PC-Algorithm, which we refer to as $\fPC$, estimates an undirected graph using conditional independence tests.  The second stage orients the edges in the undirected graph to estimate a directed graph.  We use the first stage of the PC-Algorithm for UGMS.  Algorithm~\ref{alg:fPC} outlines $\fPC$.  Variants of the PC-Algorithm for learning undirected graphical models have recently been analyzed in \citet{AnimaTanWillsky2011a,AnimaTanWillsky2011b}.  The main property used in $\fPC$ is the global Markov property of undirected graphical models which states that if a set of vertices $S$ separates $i$ and $j$, then $X_i \ind X_j | X_S$.  As seen in Line~5 of Algorithm~\ref{alg:fPC}, $\fPC$ deletes an edge $(i,j)$ if it identifies a conditional independence relationship.  Some properties of $\fPC$ are summarized as follows:

\begin{algorithm}[t]
\label{alg:fPC}
\DontPrintSemicolon
\caption{ PC-Algorithm for UGMS: $\fPC(\kappa,\Xf^n,H,L$)}
\textbf{Inputs:} \\
$\quad \kappa$: An integer that controls the computational complexity of $\fPC$. \\
$\quad \Xf^n$: $n$ i.i.d. observations. \\
$\quad H$: A graph that contains all the true edges $G^*$. \\
$\quad L$: A graph that contains the edges that need to be estimated. \\
\textbf{Output:} A graph $\widehat{G}$ that contains edges in $L$ that are estimated to be in $G^*$.\;
\nl $\widehat{G} \gets L$ \;
\nl \For{each $k \in \{0,1,\ldots,\kappa\}$}{
\nl \For{each $(i,j) \in E(\widehat{G})$}{
\nl ${\cal S}_{ij} \gets$ Neighbors of $i$ or $j$ in $H$ depending on which one has lower cardinality. \;
\nl \If{$\exists$ $S \subset {\cal S}_{ij}$, $|S| = k$, s.t. $X_i \ind X_j | X_S$ (computed using $\Xf^n$)}{
\nl Delete edge $(i,j)$ from $\widehat{G}$ and $H$. \;
}}}
\nl Return $\widehat{G}$.
\end{algorithm}

\begin{enumerate}
 \item \textbf{Parameter $\kappa$}:  $\fPC$ iteratively searches for separators for an edge $(i,j)$ by searching for separators of size $0,1,\ldots,\kappa$.  This is reflected in Line~2 of Algorithm~\ref{alg:fPC}.  Theoretically, the algorithm can automatically stop after searching for all possible separators for each edge in the graph.  However, this may not be computationally tractable, which is why $\kappa$ needs to be specified.

 \item \textbf{Conditional Independence Test: } Line~5 of Algorithm~\ref{alg:fPC} uses a conditional independence test to determine if an edge $(i,j)$ is in the true graph.  This makes $\fPC$ extremely flexible since nonparametric independence tests may be used, see \citet{hoeffding1948non,rasch2kernel,zhang2012kernel} for some examples.  In this paper, for simplicity, we only consider Gaussian graphical models.  In this case, conditional independence can be tested using the conditional correlation coefficient defined as
\begin{equation}
 \text{Conditional correlation coefficient: } \rho_{ij|S} = \frac{\Sigma_{ij} - \Sigma_{i,S} \Sigma_{S,S}^{-1} \Sigma_{S,j}}{\sqrt{\Sigma_{i,i|S} \Sigma_{j,j|S}}} \,,
\end{equation}
where $P_X \sim {\cal N}(0,\Sigma)$, $\Sigma_{A,B}$ is the covariance matrix of $X_A$ and $X_B$, and $\Sigma_{A,B|S}$ is the conditional covariance defined by 
\begin{equation}
 \Sigma_{A,B|S} = \Sigma_{A,B} - \Sigma_{A,S} \Sigma_{S,S}^{-1} \Sigma_{B,S} \,.
\end{equation}
Whenever $X_i \ind X_j | X_S$, then $\rho_{ij|S} = 0$.  This motivates the following test for independence:
\begin{equation}
 \text{Conditional Independence Test: } |\widehat{\rho}_{ij|S}| < \lambda_n \Longrightarrow X_i \ind X_j | X_S \,, \label{eq:cit}
\end{equation}
where $\widehat{\rho}_{ij|S}$ is computed using the empirical covariance matrix from the observations $\Xf^n$.  The regularization parameter $\lambda_n$ controls the number of edges estimated in $\widehat{G}$.

\item \textbf{The graphs $H$ and $L$:} Recall that $H$ contains all the edges in $G^*$.  The graph $L$ contains edges that need to be estimated since, as seen in Algorithm~\ref{alg:ugmsregion}, we apply UGMS to only certain parts of the graph instead of the whole graph.  As an example, to estimate edges in a region $R$ of a region graph representation of $H$, we apply Algorithm~\ref{alg:fPC} as follows:
\begin{equation}
 \widehat{G}_R = \fPC\left(\eta,\Xf^n,H,H_R' \right) \,, \label{eq:tt1}
\end{equation}
where $H_R'$ is defined in (\ref{eq:haprime}).  Notice that we do not use $\overline{R}$ in (\ref{eq:tt1}).  This is because Line~4 of Algorithm~\ref{alg:fPC} automatically finds the set of vertices to apply the $\fPC$ algorithm to.  Alternatively, we can apply Algorithm~\ref{alg:fPC} using $\overline{R}$ as follows:
\begin{equation}
 \widehat{G}_R = \fPC\left(\eta,\Xf^n_{\overline{R}},K_{\overline{R}},H_R' \right) \,, \label{eq:tt2}
\end{equation}
where $K_{\overline{R}}$ is the complete graph over $\overline{R}$.

\item \textbf{The set ${\cal S}_{ij}$:} An important step in Algorithm~\ref{alg:fPC} is specifying the set ${\cal S}_{ij}$ in Line~4 to restrict the search space for finding separators for an edge $(i,j)$.  This step significantly reduces the computational complexity of $\fPC$ and differentiates $\fPC$ from the first stage of the SGS-Algorithm \citep{SGSAlgorithm}, which specifies ${\cal S}_{ij} = V \backslash \{i,j\}$.
\end{enumerate}

%An example illustrating $\fPC$ is shown in Figure~\ref{fig:fPC}.

\section{Theoretical Analysis of Junction Tree based $\fPC$}
\label{sec:theoreticalPC}

We use the PC-algorithm to analyze the junction tree based UGMS algorithm.  Our main result, stated in Theorem~\ref{thm:mainResult1}, shows that when using the PC-Algorithm with the junction tree framework, we can potentially estimate the graph using fewer number of observations than what is required by the standard PC-Algorithm.  As we shall see in Theorem~\ref{thm:mainResult1}, the particular gain in performance depends on the structure of the graph.

Section~\ref{subsec:assumptions} discusses the assumptions we place on the graphical model.  Section~\ref{subsec:theory} presents the main theoretical result highlighting the advantages of using junction trees.  Throughout this Section, we use standard asymptotic notation so that $f(n) = \Omega(g(n))$ implies that there exists an $N$ and a constant $c$ such that for all $n \ge N$, $f(n) \ge c g(n)$.  For $f(n) = O(g(n))$, replace $\ge$ by $\le$. 

\subsection{Assumptions}
\label{subsec:assumptions}

\begin{enumerate}[({A}1)]
 \item \textbf{Gaussian graphical model:} We assume $X = (X_1,\ldots,X_p) \sim P_X$, where $P_X$ is a multivariate normal distribution with mean zero and covariance $\Sigma$.  Further, $P_X$ is Markov on $G^*$ and not Markov on any subgraph of $G^*$.  It is well known that this is assumption translates into the fact that $\Sigma_{ij}^{-1} = 0$ if and only if $(i,j) \notin G^*$ \citep{SpeedKiiveri1986}.

\item \textbf{Faithfulness:}  If $X_i \ind X_j | X_S$, then $i$ and $j$ are separated by\footnote{If $S$ is the empty set, this means that there is no path between $i$ and $j$.} $S$.  This assumption is important for the $\fPC$ algorithm to output the correct graph.  Further, note that the Markov assumption is different since it goes the other way: if $i$ and $j$ are separated by $S$, then $X_i \ind X_j | X_S$.  Thus, when both (A1) and (A2) hold, we have that $X_i \ind X_j | X_S \Longleftrightarrow (i,j) \notin G^*$.

\item \textbf{Separator Size $\eta$:} For all $(i,j) \notin G^*$, there exists a subset of nodes $S \subset V \backslash \{i,j\}$, where $|S| \le \eta$, such that $S$ is a separator for $i$ and $j$ in $G^*$.  This assumption allows us to use $\kappa = \eta$ when using $\fPC$.

\item \textbf{Conditional Correlation Coefficient $\rho_{ij|S}$ and $\Sigma$: }  Under (A3), we assume that $\rho_{ij|S}$ satisfies
\begin{equation}
 \sup \{ |\rho_{ij|S}|: i,j \in V, S \subset V, |S| \le \eta \} \} \le M < 1 \,,
\end{equation}
where $M$ is a constant.  Further, we assume that $\max_{i,S, |S| \le \eta} \Sigma_{i,i|S} \le L < \infty$.

\item \textbf{High-Dimensionality} We assume that the number of vertices in the graph $p$ scales with $n$ so that $p \rightarrow \infty$ as $n \rightarrow \infty$.  Furthermore, both $\rho_{ij|S}$ and $\eta$ are assumed to be functions of $n$ and $p$ unless mentioned otherwise.

\item \textbf{Structure of $G^*$: } Under (A3), we assume that there \textit{exists} a set of vertices $V_1$, $V_2$, and $T$ such that $T$ separates $V_1$ and $V_2$ in $G^*$ and $|T| < \eta$.
Figure~\ref{fig:genstruc}(a) shows the general structure of this assumption.  
%As we will see in the next Section, this structure of the graph will allow us to apply the junction tree framework to the region graph representation in Figure~\ref{fig:genstruc}(b).
\end{enumerate}

Assumptions (A1)-(A5) are standard conditions for proving high-dimensional consistency of the PC-Algorithm for Gaussian graphical models.  The structural constraints on the graph in Assumption~(A6) are required for showing the advantages of the junction tree framework.  We note that although (A6) appears to be a strong assumption, there are several graph families that satisfy this assumption.  For example, the graph in Figure~\ref{fig:FirstExample}(a) satisfies (A6) with $V_1 = \{1,2\}$, $V_2 = \{1,3,4,5,6,7\}$, and $T = \{1\}$.  In general, if there exists a separator in the graph of size less than $\eta$, then (A6) is clearly satisfied.  Further, we remark that we only assume the \textit{existence} of the sets $V_1$, $V_2$, and $T$ and do \textit{not} assume that these sets are known \textit{a priori}.  We refer to Remark~\ref{rem:ext} for more discussions about (A6) and some extensions of this assumption.

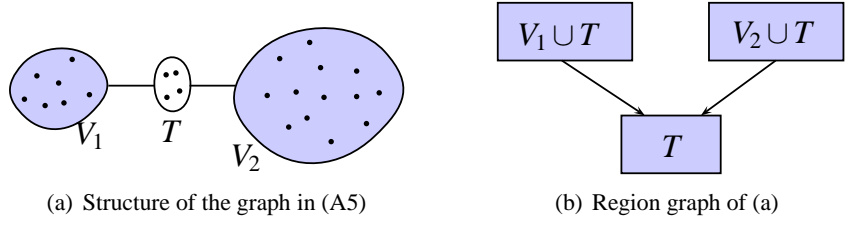
\begin{figure}
\centering
 \subfigure[Structure of the graph in (A5)]{
\scalebox{0.6}{
\begin{pspicture}(0,-1.7426543)(9.106651,1.7426542)
\definecolor{color30b}{rgb}{0.8,0.8,1.0}
\psbezier[linewidth=0.04,fillstyle=solid,fillcolor=color30b](0.38826162,0.6918938)(0.77652323,1.1330954)(1.6471232,1.1826543)(2.0335617,0.7397271)(2.42,0.2967998)(2.3162072,0.011489006)(1.8954945,-0.39631054)(1.4747819,-0.80411005)(0.88718396,-0.8773457)(0.45729518,-0.4800186)(0.027406394,-0.08269151)(0.0,0.25069222)(0.38826162,0.6918938)
\psbezier[linewidth=0.04,fillstyle=solid,fillcolor=color30b](5.554909,0.9018673)(6.229819,1.6397681)(7.743169,1.7226542)(8.41491,0.9818674)(9.086651,0.2410806)(8.906227,-0.23609608)(8.174909,-0.91813266)(7.443591,-1.6001692)(6.422179,-1.7226542)(5.6749096,-1.0581326)(4.92764,-0.39361107)(4.8799996,0.1639665)(5.554909,0.9018673)
\psbezier[linewidth=0.04](3.3149095,0.26186737)(3.3149095,-0.5381326)(4.1149096,-0.55813265)(4.1149096,0.24186738)(4.1149096,1.0418674)(3.3149095,1.0618674)(3.3149095,0.26186737)
\psdots[dotsize=0.12](1.4949093,0.8018674)
\psdots[dotsize=0.12](1.1949093,0.28186736)
\psdots[dotsize=0.12](0.89490926,-0.21813263)
\psdots[dotsize=0.12](0.43490934,-0.07813263)
\psdots[dotsize=0.12](0.7149093,0.44186735)
\psdots[dotsize=0.12](1.8749093,0.02186736)
\psdots[dotsize=0.12](1.3149093,-0.15813264)
\psdots[dotsize=0.12](5.8149095,0.00186737)
\psdots[dotsize=0.12](6.9949102,0.6218674)
\psdots[dotsize=0.12](7.794909,0.5418674)
\psdots[dotsize=0.12](8.2949095,0.06186737)
\psdots[dotsize=0.12](8.014909,-0.61813265)
\psdots[dotsize=0.12](7.234909,-1.0181326)
\psdots[dotsize=0.12](7.0749097,-0.03813262)
\psdots[dotsize=0.12](6.4549103,0.08186738)
\psdots[dotsize=0.12](6.2949095,-0.69813263)
\psdots[dotsize=0.12](6.694909,-0.49813262)
\psdots[dotsize=0.12](7.794909,-0.01813262)
\psdots[dotsize=0.12](6.094909,0.8218674)
\psdots[dotsize=0.12](6.754909,1.1818672)
\psdots[dotsize=0.12](3.574909,0.46186736)
\psdots[dotsize=0.12](3.874909,0.10186737)
\psdots[dotsize=0.12](3.594909,-0.03813262)
\psdots[dotsize=0.12](3.7949092,0.5018674)
\usefont{T1}{ptm}{m}{n}
\rput(3.7089717,-0.7831326){\huge $T$}
\usefont{T1}{ptm}{m}{n}
\rput(1.8961593,-0.8831326){\huge $V_1$}
\usefont{T1}{ptm}{m}{n}
\rput(5.3631907,-1.3831327){\huge $V_2$}
\psline[linewidth=0.04cm](2.2949092,0.22186737)(3.3149095,0.22186737)
\psline[linewidth=0.04cm](4.094909,0.22186737)(5.1149096,0.22186737)
\end{pspicture}
}} \qquad
\subfigure[Region graph of (a)]{
% Generated with LaTeXDraw 2.0.8
% Wed Aug 29 20:57:00 CDT 2012
% \usepackage[usenames,dvipsnames]{pstricks}
% \usepackage{epsfig}
% \usepackage{pst-grad} % For gradients
% \usepackage{pst-plot} % For axes
\scalebox{0.6} % Change this value to rescale the drawing.
{
\begin{pspicture}(0,-1.91)(7.7,1.91)
\definecolor{color0b}{rgb}{0.8,0.8,1.0}
\psframe[linewidth=0.04,dimen=outer,fillstyle=solid,fillcolor=color0b](3.02,1.91)(0.0,0.59)
\psframe[linewidth=0.04,dimen=outer,fillstyle=solid,fillcolor=color0b](7.7,1.91)(4.68,0.59)
\usefont{T1}{ptm}{m}{n}
\rput(1.4,1.14){\huge $V_1 \cup T$}
\usefont{T1}{ptm}{m}{n}
\rput(6.15,1.18){\huge $V_2 \cup T$}
\psframe[linewidth=0.04,dimen=outer,fillstyle=solid,fillcolor=color0b](4.98,-0.59)(2.74,-1.91)
\usefont{T1}{ptm}{m}{n}
\rput(3.9,-1.3){\huge $T$}
\psline[linewidth=0.04cm,arrowsize=0.05291667cm 3.0,arrowlength=1.4,arrowinset=0.4]{->}(1.44,0.61)(3.28,-0.61)
\psline[linewidth=0.04cm,arrowsize=0.05291667cm 3.0,arrowlength=1.4,arrowinset=0.4]{->}(6.18,0.61)(4.52,-0.61)
\end{pspicture} 
}}
\caption{General Structure of the graph we use in showing the advantages of the junction tree framework.}
\label{fig:genstruc}
\end{figure}

\subsection{Theoretical Result and Analysis}
\label{subsec:theory}

Recall $\fPC$ in Algorithm~\ref{alg:fPC}.  Since we assume (A1), the conditional independence test in (\ref{eq:cit}) can be used in Line~5 of Algorithm~\ref{alg:fPC}.  To analyze the junction tree framework, consider the following steps to construct $\widehat{G}$ using $\fPC$ when given $n$ i.i.d. observations $\Xf^n$:

\begin{enumerate}[Step 1.]
 \item Compute $H$: Apply $\fPC$ using a regularization parameter $\lambda_n^0$ such that 
\begin{equation}
H = \fPC(|T|,\Xf^n,K_V,K_V)\,, 
\end{equation}
where $K_V$ is the complete graph over the nodes $V$.  In the above equation, we apply $\fPC$ to remove all edges for which there exists a separator of size less than or equal to $|T|$.

\item Estimate a subset of edges over $V_1 \cup T$ and $V_2 \cup T$ using regularization parameters $\lambda^1_n$ and $\lambda^2_n$, respectively, such that
\begin{equation}
 \widehat{G}_{V_k} = \fPC\left(\eta, \Xf^n, H[V_k \cup T] \cup K_T, 
 H'_{V_k \cup T} \right), \text{for $k = 1,2$},
\end{equation}
where $H'_{V_k \cup T} = H[V_k \cup T] \backslash K_T$ as defined in (\ref{eq:haprime}).

\item Estimate edges over $T$ using a regularization parameter $\lambda_n^T$:
\begin{equation}
 \widehat{G}_{T} = \fPC\left(\eta, \Xf^n, H[T \cup ne_{\widehat{G}_{V_1} \cup \widehat{G}_{V_2}}(T)], H[T] \right) \,.
\end{equation}

\item Final estimate is $\widehat{G} = \widehat{G}_{V_1} \cup \widehat{G}_{V_2} \cup \widehat{G}_T$.
\end{enumerate}

Step~1 is the screening algorithm used to eliminate some edges from the complete graph.  For the region graph in Figure~\ref{fig:genstruc}(b), Step~2 corresponds to applying $\fPC$ to the regions $V_1 \cup T$ and $V_2 \cup T$.  Step~3 corresponds to applying $\fPC$ to the region $T$ and all neighbors of $T$ estimated so far.  Step~4 merges all the estimated edges.  Although the neighbors of $T$ are sufficient to estimate all the edges in $T$, in general, depending on the graph, a smaller set of vertices is required to estimate edges in $T$.  The main result is stated using the following terms defined on the graphical model:
\begin{align}
p_1 &= |V_1| + |T| \,, \;
p_2 = |V_2| + |T| \,, \;
p_T = |T \cup ne_{G^*}(T) |  \,, \;
\eta_T = |T| \label{eq:p1} \\
\rho_{0} &= \inf\{|\rho_{ij|S}|: i,j \;s.t.\;|S| \le \eta_T \;\&\; |\rho_{ij|S}| > 0 \} \\
\rho_{1} &= \inf\{|\rho_{ij|S}|: i \in V_1, j \in V_1 \cup T \;s.t.\; (i,j) \in E(G^*), S \subseteq V_1 \cup T, |S| \le \eta\}  \\
\rho_{2} &= \inf\{|\rho_{ij|S}|: i \in V_2, j \in V_2 \cup T \;s.t.\; (i,j) \in E(G^*), S \subseteq V_2 \cup T, |S| \le \eta \}  \\
\rho_{T} &= \inf\{|\rho_{ij|S}|: i,j \in T\;s.t.\; (i,j) \in E, S \subseteq T \cup ne_{G^*}(T), \eta_T < |S| \le \eta \} \,, \label{eq:rhoTdef}
\end{align}
The term $\rho_0$ is a measure of how hard it is to learn the graph $H$ in Step 1 so that $E(G^*) \subseteq E(H)$ and all edges that have a separator of size less than $|T|$ are deleted in $H$.  The terms $\rho_{1}$ and $\rho_2$ are measures of how hard it is learn the edges in $G^*[{V_1 \cup T}] \backslash K_T$ and $G^*[V_2 \cup T] \backslash K_T$ (Step 2), respectively, given that $E(G^*) \subseteq E(H)$.  The term $\rho_T$ is a measure of how hard it is learn the graph over the nodes ${T}$ given that we know the edges that connect $V_1$ to $T$ and $V_2$ to $T$.

\begin{theorem}
\label{thm:mainResult1}
Under Assumptions~(A1)-(A6), there exists a conditional independence test such that if
\begin{align}
n = \Omega\left(\max\left\{ 
\rho_0^{-2}\eta_T \log(p), 
\rho_1^{-2}\eta \log(p_1),
\rho_2^{-2}\eta \log(p_2),
\rho_T^{-2}\eta \log(p_T)
\right\} \right) \,, \label{eq:pcref}
\end{align}
then $P(\widehat{G} \ne G) \rightarrow 0$ as $n \rightarrow \infty$.
\end{theorem}
\begin{proof}
 See Appendix~\ref{app:mainResult1}.
\end{proof}
We now make several remarks regarding Theorem~\ref{thm:mainResult1} and its consequences.

\begin{remark}[Comparison to Necessary Conditions]
Using results from \citet{WangWainwright2010}, it follows that a necessary condition for any algorithm to recover the graph $G^*$ that satisfies Assumptions (A1) and (A6) is that $n = \Omega( \max \{ \theta_{1}^{-2} \log(p_1-d), \theta_{2}^{-2} \log(p_2-d) \}$, where $d$ is the maximum degree of the graph and $\theta_1$ and $\theta_2$ are defined as follows:
\begin{align}
\theta_k &= \min_{(i,j) \in G^*[V_k \cup T] \backslash G^*[T]} \frac{|\Sigma_{ij}^{-1}|}{\sqrt{|\Sigma_{ii}^{-1} \Sigma_{jj}^{-1}|}} \,, k = 1,2 \,.
\end{align}
If $\eta$ is a constant and $\rho_1$ and $\rho_2$ are chosen so that the corresponding expressions dominate all other expressions, then (\ref{eq:pcref}) reduces to $n = \Omega( \max\{ \rho_1^{-2} \log(p_1), \rho_2^{-2} \log(p_2) \})$.  Furthermore, for certain classes of Gaussian graphical models, namely walk summable graphical models \citep{malioutov2006walk}, the results in \citet{AnimaTanWillsky2011b} show that there exists conditions under which $\rho_1 = \Omega(\theta_1)$ and $\rho_2 = \Omega(\theta_2)$.  In this case, (\ref{eq:pcref}) is equivalent to $n = \Omega( \max\{ \theta_1^{-2} \log(p_1), \theta_2^{-2} \log(p_2)\})$.  Thus, as long as $p_1,p_2 \gg d$, there exists a family of graphical models for which the sufficient conditions in Theorem~\ref{thm:mainResult1} nearly match the necessary conditions for asymptotically reliable estimation of the graph.  We note that the particular family of graphical models is quite broad, and includes forests, scale-free graphs, and some random graphs.  We refer to \citet{AnimaTanWillsky2011b} for a characterization of such graphical models.

\end{remark}

\begin{remark}[Choice of Regularization Parameters]
We use the conditional independence test in (\ref{eq:cit}) that thresholds the conditional correlation coefficient.  From the proof in  Appendix~\ref{app:mainResult1}, the thresholds, which we refer to as the regularization parameter, are chosen as follows: 
 \begin{align}
 \lambda_n^0 &= O(\rho_0) \text{ and } \rho_0 = \Omega\left( \sqrt{\eta_T \log(p)/n} \right) \\
 \lambda_n^k &= O(\rho_k) \text{ and } \rho_k = \Omega\left( \sqrt{ \eta \log(p_k)/n} \right), k = 1,2 \\
 \lambda_n^T &= O(\rho_T) \text{ and } \rho_T = \Omega\left( \sqrt{ \eta \log(p_T)/n} \right) \,.
 \end{align} 
We clearly see that different regularization parameters are used to estimate different parts of the graph.  Furthermore, just like in the traditional analysis of UGMS algorithms, the optimal choice of the regularization parameter depends on unknown parameters of the graphical model.  In practice, we use model selection algorithms to select regularization parameters.  We refer to Section~\ref{sec:numericalSim} for more details.
\end{remark}

\begin{remark}[Weaker Condition]
If we do not use the junction tree based approach outlined in Steps~1-4, and instead directly apply $\fPC$, the sufficient condition on the number of observations will be $n = \Omega( \rho_{min}^{-2} \eta \log(p))$, where \begin{equation}
\rho_{min} := \inf\{|\rho_{ij|S}|: (i,j) \in E(G^*), |S| \le \eta \} \,.  
\end{equation}
This result is proven in Appendix~\ref{app:GeneralResult} using results from \citet{KalischBuhlmann2007,AnimaTanWillsky2011b}.  Since $\rho_{min} \le \min\{\rho_0,\rho_1,\rho_2,\rho_T\}$, it is clear that (\ref{eq:pcref}) is a weaker condition.  The main reason for this difference is that the junction tree approach defines an \emph{ordering} on the edges to test if an edge belongs to the true graph.  This ordering allows for a reduction in separator search space (see ${\cal S}_{ij}$ in Algorithm~\ref{alg:fPC}) for testing edges over the set $T$.  Standard analysis of $\fPC$ assumes that the edges are tested randomly, in which case, the separator search space is always upper bounded by the full set of nodes.
\end{remark}

\begin{remark}[Reduced Sample Complexity]
\label{rm:prev}
Suppose $\eta$, $\rho_0$, and $\rho_T$ are constants and $\rho_1 < \rho_2$.  In this case, (\ref{eq:pcref}) reduces to
\begin{align}
n = \Omega\left( \max\left\{ \log(p),\rho_1^{-2} \log(p_1), \rho_2^{-2} \log(p_2) \right\}  \right) \,. \label{eq:ssim}
\end{align}
If $\rho_1^{-2} = \Omega\left( \max\left\{ \rho_2^{-2} \log(p_2)/\log(p_1), \log(p)  \right\} \right)$, then (\ref{eq:ssim}) reduces to
\begin{equation}
n = \Omega\left( \rho_1^{-2} \log(p_1) \right) \,.
\end{equation}
On the other hand, if we do not use junction trees, $n = \Omega\left( \rho_{min}^{-2} \log(p)\right)$, where $\rho_{min} \le \rho_1$.
Thus, if $p_1 \ll p$, for example $p_1 = \log(p)$, then using the junction tree based $\fPC$ requires lower number of observations for consistent UGMS.  Informally, the above condition says that if the graph structure in (A6) is easy to identify, $p_1 \ll p_2$, and the minimal conditional correlation coefficient over the true edges lies in the smaller cluster (but not over the separator),  the junction tree framework may accurately learn the graph using significantly less number of observations.
\end{remark}

\begin{remark}[Learning Weak Edges]
We now analyze Theorem~\ref{thm:mainResult1} to see how the conditional correlation coefficients scale for high-dimensional consistency.  Under the assumption in Remark~\ref{rm:prev}, it is easy to see that the minimal conditional correlation coefficient scales as $\Omega(\sqrt{\log(p_1)/n})$ when using junction trees and as $\Omega(\sqrt{\log(p)/n})$ when not using junction trees.  This suggests that when $p_1 \ll p$, it may be possible to learn edges with weaker conditional correlation coefficients when using junction trees.  Our numerical simulations in Section~\ref{sec:numericalSim} empirically show this property of the junction tree framework.
\end{remark}

\begin{remark}[Computational complexity]
It is easy to see that the worst case computational complexity of the PC-Algorithm is $O(p^{\eta+2})$ since there are $O(p^2)$ edges and testing for each edge requires a search over at most $O(p^{\eta})$ separators.  The worst case computational complexity of Steps~1-4 is roughly $O\left(p^{|T|+2} + p_1^{\eta+2} + p_2^{\eta+2} + p_T^{\eta+2}\right)$.  Under the conditions in Remark~8.3 and when $p_1 \ll p$, this complexity is roughly $O(p^{\eta+2})$, which is the same as the standard PC-Algorithm.  In practice, especially when the graph is sparse, the computational complexity is much less than $O(p^{\eta+2})$ since the PC-Algorithm restricts the search space for finding separators.
\end{remark}

\begin{figure}
\centering
\scalebox{0.6} % Change this value to rescale the drawing.
{
\begin{pspicture}(0,-2.152845)(15.512709,2.162845)
\definecolor{color114b}{rgb}{0.8,0.8,1.0}
\psellipse[linewidth=0.04,dimen=outer,fillstyle=solid,fillcolor=color114b](1.4905808,-1.482845)(0.95,0.51)
\psellipse[linewidth=0.04,dimen=outer,fillstyle=solid,fillcolor=color114b](1.5305808,0.7371551)(0.95,0.51)
\psellipse[linewidth=0.04,dimen=outer,fillstyle=solid,fillcolor=color114b](4.8705807,0.7371551)(0.95,0.51)
\psellipse[linewidth=0.04,dimen=outer,fillstyle=solid,fillcolor=color114b](4.8705807,-1.482845)(0.95,0.51)
\psellipse[linewidth=0.04,dimen=outer,fillstyle=solid,fillcolor=color114b](8.23058,0.71715504)(0.95,0.51)
\psframe[linewidth=0.04,dimen=outer](3.4405808,1.067155)(2.9605808,0.50715506)
\psframe[linewidth=0.04,dimen=outer](6.7405806,1.0071551)(6.2605805,0.4471551)
\rput{90.0}(4.487736,-5.2734256){\psframe[linewidth=0.04,dimen=outer](5.1205807,-0.11284491)(4.640581,-0.6728449)}
\rput{90.0}(1.1277359,-1.9134257){\psframe[linewidth=0.04,dimen=outer](1.7605808,-0.11284491)(1.2805808,-0.6728449)}
\psline[linewidth=0.04cm](2.4805808,0.7671551)(2.9805808,0.7671551)
\psline[linewidth=0.04cm](3.4205809,0.7671551)(3.9405808,0.7671551)
\psline[linewidth=0.04cm](5.800581,0.7471551)(6.280581,0.7471551)
\psline[linewidth=0.04cm](6.7205806,0.7471551)(7.300581,0.7471551)
\psline[linewidth=0.04cm](4.900581,0.24715507)(4.900581,-0.15284492)
\psline[linewidth=0.04cm](4.880581,-0.6128449)(4.880581,-0.97284496)
\psline[linewidth=0.04cm](1.5205808,0.24715507)(1.5205808,-0.15284492)
\psline[linewidth=0.04cm](1.5005808,-0.6128449)(1.5005808,-0.97284496)
\usefont{T1}{ptm}{m}{n}
\rput(1.4568552,0.7371551){\huge $V_2$}
\usefont{T1}{ptm}{m}{n}
\rput(1.4568552,-1.482845){\huge $V_1$}
\usefont{T1}{ptm}{m}{n}
\rput(4.8168554,0.7571551){\huge $V_3$}
\usefont{T1}{ptm}{m}{n}
\rput(4.7968554,-1.482845){\huge $V_4$}
\usefont{T1}{ptm}{m}{n}
\rput(8.156856,0.71715504){\huge $V_5$}
\psbezier[linewidth=0.02,linestyle=dashed,dash=0.16cm 0.16cm](0.5019791,1.3070685)(1.0039582,2.152845)(9.5954485,1.8617712)(9.848015,0.9162697)(10.100581,-0.02923172)(5.8366313,-2.0729175)(4.834047,-2.0928812)(3.8314626,-2.112845)(2.2728229,-0.9671744)(1.5448843,-0.2952066)(0.8169457,0.37676117)(0.0,0.461292)(0.5019791,1.3070685)
\psline[linewidth=0.08cm,arrowsize=0.05291667cm 2.0,arrowlength=1.4,arrowinset=0.4]{->}(8.94058,-0.5528449)(10.260581,-0.5528449)
\psellipse[linewidth=0.04,dimen=outer,fillstyle=solid,fillcolor=color114b](12.170581,-1.642845)(0.95,0.51)
\rput{90.0}(11.647737,-12.753425){\psframe[linewidth=0.04,dimen=outer](12.440581,-0.27284402)(11.960581,-0.832844)}
\psline[linewidth=0.04cm](12.18058,-0.7728449)(12.18058,-1.1328449)
\usefont{T1}{ptm}{m}{n}
\rput(12.136855,-1.642845){\huge $V_1$}
\psellipse[linewidth=0.04,dimen=outer,fillstyle=solid,fillcolor=color114b](12.205291,0.88500005)(1.7847095,0.802155)
\psline[linewidth=0.04cm](12.200581,0.08715508)(12.200581,-0.31284493)
\usefont{T1}{ptm}{m}{n}
\rput(12.196855,0.93715507){\huge $\cup_{i=2}^{5} V_i$}
\end{pspicture} 
}
\caption{Junction tree representation with clusters $V_1,\ldots,V_5$ and separators denotes by rectangular boxes.  We can cluster vertices in the junction tree to get a two cluster representation as in Figure~\ref{fig:genstruc}.}
\label{fig:seg}
\end{figure}
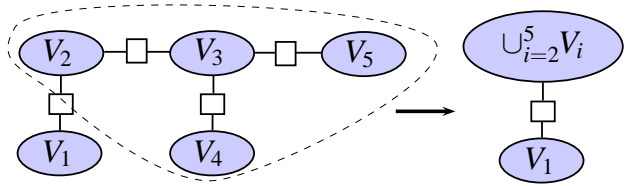

\begin{remark}[Using other UGMS Algorithms]
Although our analysis used the PC-Algorithm to derive sufficient conditions for accurately estimating the graph, we can easily use other algorithms, such as the graphical Lasso or the neighborhood selection based Lasso, for analysis.  The main difference will be in the assumptions imposed on the graphical model.
\end{remark}

\begin{remark}[Extensions]
\label{rem:ext}
We have analyzed the junction tree framework assuming that the junction tree of $H$ only has two clusters.  One way to generalize our analysis to junction trees with multiple clusters is to merge clusters so that the resulting junction tree admits the structure in Figure~\ref{fig:genstruc}.  For example, suppose the graph $G^*$ has a junction tree representation as in Figure~\ref{fig:seg} with five clusters.  If $|V_1 \cap V_2| < \eta$, then we can merge the clusters $V_2,V_3,\ldots,V_5$ so that the resulting junction tree admits the two cluster representation in Figure~\ref{fig:genstruc}.  Furthermore, we can also generalize Theorem~\ref{thm:mainResult1} to cases when $|T| = \eta$. The main change in the analysis will be in the definition of $\rho_0$. For example, if the graph is a chain so that the first $p_1$ vertices are associated with ``weak edges", we can get similar results as in Theorem~\ref{thm:mainResult1}.  Finally, we note that a full analysis of the junction tree framework, that also incorporates the step of updating the junction tree in Algorithm~\ref{alg:jtframework}, is challenging and will be addressed in future work.
\end{remark}

\section{Numerical Simulations}
\label{sec:numericalSim}

In this section, we present numerical simulations that highlight the advantages of using the junction tree framework for UGMS.  Throughout this Section, we assume a Gaussian graphical model such that $P_X \sim {\cal N}(0,\Theta^{-1})$ is Markov on $G^*$.  It is well known that this implies that $(i,j) \notin G^* \Longleftrightarrow \Theta_{ij} = 0$ \citep{SpeedKiiveri1986}.  Some algorithmic details used in the simulations are described as follows.

\medskip

\noindent
\textbf{Computing $H$:}  We apply Algorithm~\ref{alg:fPC} with a suitable value of $\kappa$ in such a way that the separator search space ${\cal S}_{ij}$ (see Line~4) is restricted to be small.  In other words, we do not test for all possible conditional independence tests so as to restrict the computational complexity of the screening algorithm.  We use the conditional partial correlation to test for conditional independence and choose a \textit{separate threshold} to test for each edge in the graph.  The thresholds for the conditional independence test are computed using $5$-fold cross-validation.  The computational complexity of this step is roughly $O(p^2)$ since there are $O(p^2)$ edges to be tested.  Note that this method for computing $H$ is equivalent to Step~1 in Section~\ref{subsec:theory} with $|T| = \kappa$.  Finally, we note that the above method does not guarantee that all edges in $G^*$ will be included in $H$.  This can result in false edges being included in the junction tree estimated graphs.  To avoid this situation, once a graph estimate $\widehat{G}$ has been computed using the junction tree based UGMS algorithm, we apply conditional independence tests again to prune the estimated edge set.

\medskip

\noindent
\textbf{Computing the junction tree:} We use standard algorithms in the literature for computing close to optimal junction trees\footnote{We use the GreedyFillin heuristic.  This is known to give good results with reasonable computational time \citep{Kjaerulff1990}.}.  Once the junction tree is computed, we merge clusters so that the maximum size of the separator is at most $\kappa + 1$, where $\kappa$ is the parameter used when computing the graph $H$.  For example, in Figure~\ref{fig:seg}, if the separator associated with $V_2$ and $V_3$ has cardinality greater than $\kappa+1$, then we merge $V_2$ and $V_3$ and resulting junction tree is such that $V_1$, $V_4$, and $V_5$ all connect to the cluster $V_2 \cup V_3$.  

\medskip

\noindent
\textbf{UGMS Algorithms:}  We apply the junction tree framework in conjunction with graphical Lasso ($\gL$) \citep{BanerjeeGhaoui2008}, neighborhood selection using Lasso ($\nL$) \citep{NicolaiPeter2006}, and the PC-Algorithm ($\fPC$) \citep{PCAlgorithm}.  See Appendix~\ref{app:examples} for a review of $\gL$ and $\nL$ and Algorithm~\ref{alg:fPC} for $\fPC$.  When using $\nL$, we use the intersection rule to combine neighborhood estimates.  Further, we use the adaptive Lasso \citep{Zou2006} for finding neighbors of a vertex since this is known to give superior results for variable selection \citep{van2011adaptive}.

\medskip

\noindent
\textbf{Choosing Regularization Parameters:}  An important step when applying UGMS algorithms is to choose a suitable regularization parameter.  It is now well known that classical methods, such as cross-validation and information criterion based methods, tend to choose a much larger number of edges when compared to an oracle estimator for high-dimensional problems \citep{meinshausen2010stability,liu2010stability}.  Several alternative methods have been proposed in the literature; see for example stability selection \citep{meinshausen2010stability,liu2010stability} and extended Bayesian information (EBIC) criterion \citep{chen2008extended,foygel2010extended}.  In all our simulations, we use EBIC since it is much faster than stability based methods when the distribution is Gaussian.  EBIC selects a regularization parameter $\widehat{\lambda}_n$ as follows:
\begin{equation}
\widehat{\lambda}_n = \max_{\lambda_n > 0}  \left\{ n \left[ \log \det \widehat{ \Theta }_{\lambda_n} - \text{trace} (\widehat{S} \Theta)  \right] + |E(\widehat{G}_{\lambda_n})| \log n + 4 \gamma |E(\widehat{G}_{\lambda_n})| \log p \right\} \,,
\end{equation}
where $\widehat{S}$ is the empirical covariance matrix, $\widehat{\Theta}_{\lambda_n}$ is the estimate of the inverse covariance matrix and $|E(\widehat{G}_{\lambda_n})|$ is the number of edges in the estimated graph.  The estimate $\widehat{\lambda}_n$  depends on a parameter $\gamma \in [0,1]$ such that $\gamma = 0$ results in the BIC estimate and increasing $\gamma$ produces sparser graphs.  The authors in reference \citet{foygel2010extended} suggest that $\gamma = 0.5$ is a reasonable choice for high-dimensional problems.  When solving subproblems using Algorithm~\ref{alg:ugmsregion}, the $\log p$ term is replaced by $\log |\overline{R}|$, $\widehat{\Theta}_{\lambda_n}$ is replaced by the inverse covariance over the vertices $\overline{R}$, and $|\widehat{G}_{\lambda_n}|$ is replaced by the number of edges estimated from the graph $H'_R$.

\medskip

\noindent
\textbf{Small subproblems:}  Whenever $|\overline{R}|$ is small (less than 8 in our simulations), we independently test whether each edge is in $G^*$ using hypothesis testing.   This shows the application of using different algorithms to learn different parts of the graph.

\subsection{Results on Synthetic Graphs}
\label{subsec:syn}

We assume that $\Theta_{ii} = 1$ for all $i = 1,\ldots,p$.  We refer to all edges connected to the first $p_1$ vertices as \textit{weak edges} and the rest of the edges are referred to as \textit{strong edges}.  The different types of synthetic graphical models we study are described as follows:

\begin{itemize}
\item Chain ($\CHo$ and $\CHt$): $\Theta_{i,i+1} = \rho_1$ for $i = 1,\ldots,p_1-1$ (weak edges) and $\Theta_{i,i+1} = \rho_2$ for $i = p_1,p-1$ (strong edges).  For $\CHo$, $\rho_1 = 0.15$ and $\rho_2 = 0.245$.  For $\CHt$, $\rho_1 = 0.075$ and $\rho_2 = 0.245$.  Let $\Theta_{ij} = \Theta_{ji}$.

\item Cycle ($\CYo$ and $\CYt$): $\Theta_{i,i+1} = \rho_1$ for $i = 1,\ldots,p_1-1$ (weak edges) and $\Theta_{i,i+1} = \rho_2$ for $i = p_1,p-1$ (strong edges).  In addition, $\Theta_{i,i+3} = \rho_1$ for $i = 1,\ldots,p_1-3$ and $\Theta_{i,i+3} = \rho_2$ for $i = p_1,p_1+1,\ldots,p-3$.  This introduces multiple cycles in the graph.  For $\CYo$, $\rho_1 = 0.15$ and $\rho_2 = 0.245$.  For $\CYt$, $\rho_1 = 0.075$ and $\rho_2 = 0.245$.

\item Hub ($\HBo$ and $\HBt$): For the first $p_1$ vertices, construct as many star\footnote{A star is a tree where one vertex is connected all other vertices.} graphs of size $d_1$ as possible.  For the remaining vertices, construct star graphs of size $d_2$ (at most one may be of size less than $d_2$).  The hub graph $G^*$ is constructed by taking a union of all star graphs. For $(i,j) \in G^*$ s.t. $i,j \le p_1$, let $\Theta_{i,j} = 1/d_1$.  For the remaining edges, let $\Theta_{ij} = 1/d_2$. For $\HBo$, $d_1 = 8$ and $d_2 = 5$.  For $\HBt$, $d_1 = 12$ and $d_2 = 5$.

\item Neighborhood graph ($\NBo$ and $\NBt$): Randomly place vertices on the unit square at coordinates $y_1,\ldots,y_p$.  Let $\Theta_{ij} = 1/\rho_1$ with probability $(\sqrt{2\pi})^{-1} \exp(-4||y_i-y_j||_2^2)$, otherwise $\Theta_{ij} = 0$ for all $i,j \in \{1,\ldots,p_1\}$ such that $i > j$.  For all $i,j \in \{p_1+1,\ldots,p\}$ such that $i > j$, $\Theta_{ij} = \rho_2$.  For edges over the first $p_1$ vertices, delete edges so that each vertex is connected to at most $d_1$ other vertices.  For the vertices $p_1+1,\ldots,p$, delete edges such that the neighborhood of each vertex is at most $d_2$.  Finally, randomly add four edges from a vertex in $\{1,\ldots,p_1\}$ to a vertex in $\{p_1,p_1+1,\ldots,p\}$ such that for each such edge, $\Theta_{ij} = \rho_1$.  We let $\rho_2 = 0.245$, $d_1 = 6$, and $d_2 = 4$.  For $\NBo$, $\rho_1 = 0.15$ and for $\NBt$, $\rho_2 = 0.075$.
\end{itemize}

Notice that the parameters associated with the weak edges are lower than the parameters associated with the strong edges.  Some comments regarding notation and usage of various algorithms is given as follows.

\begin{itemize}
\item The junction tree versions of the UGMS algorithms are denoted by $\JgL$, $\JfPC$, and $\JnL$.
\item  We use EBIC with $\gamma = 0.5$ to choose regularization parameters when estimating graphs using $\JgL$ and $\JfPC$.  To objectively compare $\JgL$ ($\JfPC$) and $\gL$ ($\fPC$), we make sure that the number of edges estimated by $\gL$ ($\fPC$) is roughly the same as the number of edges estimated by $\JgL$ ($\JfPC$).
\item The $\nL$ and $\JnL$ estimates are computed differently since it is difficult to control the number of edges estimated using both these algorithms\footnote{Recall that both these algorithms use different regularization parameters, so there may exist multiple different estimates with the same number of edges}.  We apply both $\nL$ and $\JnL$ for multiple different values of $\gamma$ (the parameter for EBIC) and choose graphs so that the  number of edges estimated is closest to the number of edges estimated by $\gL$.
\item When applying $\fPC$ and $\JfPC$, we choose $\kappa$ as $1$, $2$, $1$, and $3$ for Chain, Cycle, Hub, and Neighborhood graphs, respectively.  When computing $H$, we choose $\kappa$ as $0$, $1$, $0$, and $2$ for Chain, Cycle, Hub, and Neighborhood graphs, respectively.
\end{itemize}

Tables~\ref{tab:chain}-\ref{tab:nb300} summarize the results for the different types of synthetic graphical models.  For an estimate $\widehat{G}$ of $G^*$, we evaluate $\widehat{G}$ using the weak edge discovery rate (WEDR), false discovery rate (FDR), true positive rate (TPR), and the edit distance (ED).  
\begin{align}
\text{WEDR} &= \frac{\text{ $\#$ weak edges in $\widehat{G}$ }}{\text{$\#$ of weak edges in $G^*$}} \\
\text{FDR} &= \frac{\text{$\#$ of edges in $\widehat{G} \backslash G^*$}}{\text{$\#$ of edges in $\widehat{G}$}} \\
\text{TPR} &= \frac{\text{$\#$ of edges in $\widehat{G} \cap G^*$}}{\text{$\#$ of edges in $G^*$}} \\
\text{ED} &= \{\text{$\#$ edges in $\widehat{G} \backslash G^*$}\} + \{\text{$\#$ edges in $ G^* \backslash \widehat{G}$}\} \,,
\end{align}
Recall that the weak edges are over the first $p_1$ vertices in the graph.  Naturally, we want WEDR and TPR to be large and FDR and ED to be small.  Each entry in the table shows the mean value and standard error (in brackets) over $50$ observations.  We now make some remarks regarding the results.

\begin{table}
\centering
\caption{Results for Chain graphs: $p = 100$ and $p_1 = 20$}
{\scriptsize{
   \begin{tabular}{llllllll} \toprule
      Model & $n$ & Alg & WEDR & FDR & TPR &ED & $|\widehat{G}|$ \\
\toprule
%& \multicolumn{2}{c}{$\CHo$, $p = 100$} \\
$\CHo$ & \multirow{6}{*}{} $300$ & $\JgL$  & $\textbf{0.305}$ ($0.005$) & $0.048$ ($0.001$) & $0.767$ ($0.002$) & $27.0$ ($0.176$) & $79.8$ \\
{$p = 100$} & & $\gL$  & ${0.180}$ ($0.004$) & $0.061$ ($0.001$) & $0.757$ ($0.001$) & $29.0$ ($0.153$) & $79.8$ \\ \cdashline{3-8} 
\multirow{4}{*}{} && $\JfPC$  & $\textbf{0.312}$ ($0.004$) & $0.047$ ($0.001$) & $0.775$ ($0.001$) & $26.0$ ($0.162$) & $80.5$ \\
&& $\fPC$  & $0.264$ ($0.005$) & $0.047$ ($0.001$) & $0.781$ ($0.001$) & $25.6$ ($0.169$) & $81.2$ \\ \cdashline{3-8} 
&&$\JnL$  & $\textbf{0.306}$ ($0.005$) & $0.072$ ($0.001$) & $0.769$ ($0.002$) & $28.8$ ($0.188$) & $82.1$ \\
&&$\nL$ & $0.271$ ($0.005$) & $0.073$ ($0.001$) & $0.757$ ($0.001$) & $30.0$ ($0.197$) & $80.9$ \\
\hline
$\CHt$ & \multirow{6}{*}{} $300$ & $\JgL$  & $\textbf{0.052}$ ($0.002$) & $0.067$ ($0.001$) & $0.727$ ($0.001$) & $32.2$ ($0.173$) & $77.3$ \\
{$p = 100$} & & $\gL$  & $0.009$ ($0.001$) & $0.062$ ($0.001$) & $0.733$ ($0.002$) & $31.3$ ($0.162$) & $77.4$ \\ \cdashline{3-8} 
\multirow{4}{*}{} && $\JfPC$  & $\textbf{0.048}$ ($0.002$) & $0.064$ ($0.001$) & $0.735$ ($0.001$) & $31.2$ ($0.169$) & $77.8$ \\
&& $\fPC$  & $0.0337$ ($0.002$) & $0.055$ ($0.001$) & $0.748$ ($0.001$) & $29.3$ ($0.144$) & $78.4$ \\ \cdashline{3-8} 
&&$\JnL$  & $\textbf{0.052}$ ($0.002$) & $0.077$ ($0.001$) & $0.733$ ($0.001$) & $32.5$ ($0.186$) & $78.7$ \\
&&$\nL$ & $0.039$ ($0.002$) & $0.086$ ($0.001$) & $0.723$ ($0.001$) & $34.2$ ($0.216$) & $78.4$ \\
\hline
$\CHo$ & \multirow{6}{*}{} $500$ & $\JgL$  & $\textbf{0.596}$ ($0.006$) & $0.021$ ($0.001$) & $0.916$ ($0.001$) & $10.2$ ($0.133$) & $92.6$ \\
{$p = 100$} & & $\gL$  & ${0.44}$ ($0.005$) & $0.050$ ($0.001$) & $0.889$ ($0.001$) & $15.6$ ($0.132$) & $92.7$ \\ \cdashline{3-8} 
\multirow{4}{*}{} && $\JfPC$  & $\textbf{0.612}$ ($0.005$) & $0.022$ ($0.001$) & $0.921$ ($0.001$) & $9.86$ ($0.128$) & $93.2$ \\
&& $\fPC$  & $0.577$ ($0.005$) & $0.032$ ($0.001$) & $0.916$ ($0.001$) & $11.4$ ($0.124$) & $93.7$ \\ \cdashline{3-8} 
&&$\JnL$  & $\textbf{0.623}$ ($0.005$) & $0.059$ ($0.001$) & $0.922$ ($0.001$) & $13.5$ ($0.133$) & $97.0$ \\
&&$\nL$ & $0.596$ ($0.005$) & $0.069$ ($0.001$) & $0.918$ ($0.001$) & $14.9$ ($0.164$) & $97.6$ \\
\hline
$\CHt$ & \multirow{6}{*}{} $500$ & $\JgL$  & $\textbf{0.077}$ ($0.002$) & $0.044$ ($0.001$) & $0.816$ ($0.001$) & $22.0$ ($0.107$) & $84.5$ \\
{$p = 100$} & & $\gL$  & $0.0211$ ($0.001$) & $0.053$ ($0.001$) & $0.808$ ($0.000$) & $23.5$ ($0.082$) & $84.6$ \\ \cdashline{3-8} 
\multirow{4}{*}{} && $\JfPC$  & $\textbf{0.073}$ ($0.002$) & $0.042$ ($0.001$) & $0.817$ ($0.001$) & $21.7$ ($0.082$) & $84.5$ \\
&& $\fPC$  & $0.0516$ ($0.002$) & $0.049$ ($0.001$) & $0.815$ ($0.001$) & $22.5$ ($0.092$) & $84.9$ \\ \cdashline{3-8} 
&&$\JnL$  & $\textbf{0.076}$ ($0.002$) & $0.070$ ($0.001$) & $0.818$ ($0.001$) & $24.2$ ($0.102$) & $87.2$ \\
&&$\nL$ & $0.066$ ($0.002$) & $0.077$ ($0.001$) & $0.815$ ($0.001$) & $25.1$ ($0.126$) & $87.5$ \\
    \bottomrule
    \end{tabular}
}}
\label{tab:chain}
\end{table}

\begin{table}
\centering
\caption{Results for Cycle graphs, $p = 100$ and $p_1 = 20$}
{\scriptsize{
   \begin{tabular}{llllllll} \toprule
      Model & $n$ & Alg & WEDR & FDR & TPR &ED & $|\widehat{G}|$ \\
\toprule
%& \multicolumn{2}{c}{$\CHo$, $p = 100$} \\
$\CYo$ & \multirow{6}{*}{} $300$ & $\JgL$  & $\textbf{0.314}$ ($0.003$) & $0.036$ ($0.001$) & $0.814$ ($0.001$) & $28.5$ ($0.142$) & $111$ \\
{$p = 100$} & & $\gL$  & $0.105$ ($0.003$) & $0.057$ ($0.001$) & $0.798$ ($0.001$) & $32.9$ ($0.16$) & $112$ \\  \cdashline{3-8} 
\multirow{4}{*}{} && $\JfPC$  & $\textbf{0.326}$ ($0.004$) & $0.030$ ($0.001$) & $0.819$ ($0.001$) & $27.2$ ($0.18$) & $112$ \\
&& $\fPC$  & $0.307$ ($0.004$) & $0.027$ ($0.001$) & $0.826$ ($0.001$) & $26$ ($0.169$) & $112$ \\\cdashline{3-8} 
&&$\JnL$  & $\textbf{0.342}$ ($0.004$) & $0.043$ ($0.001$) & $0.813$ ($0.001$) & $29.5$ ($0.175$) & $112$ \\  
&&$\nL$ & $0.299$ ($0.004$) & $0.044$ ($0.001$) & $0.793$ ($0.001$) & $32.3$ ($0.192$) & $110$ \\
\hline
$\CYt$ & \multirow{6}{*}{} $300$ & $\JgL$  & $\textbf{0.047}$ ($0.002$) & $0.045$ ($0.001$) & $0.762$ ($0.001$) & $36.2$ ($0.163$) & $105$ \\
{$p = 100$} & & $\gL$  & $0.001$ ($0.001$) & $0.049$ ($0.001$) & $0.759$ ($0.001$) & $37.0$ ($0.172$) & $105$ \\  \cdashline{3-8} 
\multirow{4}{*}{} && $\JfPC$  & $\textbf{0.043}$ ($0.002$) & $0.042$ ($0.001$) & $0.764$ ($0.001$) & $35.6$ ($0.174$) & $105$ \\
&& $\fPC$  & $0.027$ ($0.002$) & $0.036$ ($0.001$) & $0.773$ ($0.001$) & $33.7$ ($0.137$) & $106$ \\  \cdashline{3-8} 
&&$\JnL$  & $\textbf{0.042}$ ($0.002$) & $0.058$ ($0.002$) & $0.754$ ($0.001$) & $38.6$ ($0.210$) & $106$ \\
&&$\nL$ & $0.035$ ($0.002$) & $0.057$ ($0.002$) & $0.743$ ($0.001$) & $39.9$ ($0.228$) & $104$ \\
\hline
$\CYo$ & \multirow{6}{*}{} $500$ & $\JgL$  & $\textbf{0.532}$ ($0.005$) & $0.022$ ($0.001$) & $0.907$ ($0.001$) & $15.1$ ($0.139$) & $122$ \\
{$p = 100$} & & $\gL$  & $0.278$ ($0.001$) & $0.071$ ($0.001$) & $0.862$ ($0.001$) & $26.9$ ($0.178$) & $122$ \\  \cdashline{3-8} 
\multirow{4}{*}{} && $\JfPC$  & $\textbf{0.61}$ ($0.004$) & $0.012$ ($0.001$) & $0.925$ ($0.001$) & $11.9$ ($0.150$) & $124$ \\
&& $\fPC$  & $0.609$ ($0.004$) & $0.020$ ($0.001$) & $0.925$ ($0.001$) & $12.5$ ($0.134$) & $125$ \\  \cdashline{3-8} 
&&$\JnL$  & $\textbf{0.612}$ ($0.005$) & $0.028$ ($0.001$) & $0.924$ ($0.001$) & $13.6$ ($0.151$) & $125$ \\
&&$\nL$ & $0.584$ ($0.005$) & $0.041$ ($0.001$) & $0.919$ ($0.001$) & $15.9$ ($0.171$) & $126$ \\
\hline
$\CYt$ & \multirow{6}{*}{} $500$ & $\JgL$  & $\textbf{0.086}$ ($0.003$) & $0.039$ ($0.001$) & $0.821$ ($0.001$) & $28.1$ ($0.116$) & $113$ \\
{$p = 100$} & & $\gL$  & $0.004$ ($0.001$) & $0.058$ ($0.001$) & $0.805$ ($0.000$) & $32.3$ ($0.088$) & $113$ \\  \cdashline{3-8} 
\multirow{4}{*}{} && $\JfPC$  & $\textbf{0.087}$ ($0.002$) & $0.034$ ($0.001$) & $0.825$ ($0.001$) & $27.0$ ($0.099$) & $113$ \\
&& $\fPC$  & $0.074$ ($0.002$) & $0.040$ ($0.001$) & $0.823$ ($0.001$) & $27.9$ ($0.010$) & $113$ \\  \cdashline{3-8} 
&&$\JnL$  & $\textbf{0.085}$ ($0.003$) & $0.045$ ($0.001$) & $0.824$ ($0.001$) & $28.4$ ($0.147$) & $114$ \\
&&$\nL$ & $0.069$ ($0.003$) & $0.053$ ($0.001$) & $0.821$ ($0.001$) & $29.8$ ($0.158$) & $114$ \\
    \bottomrule
    \end{tabular}
}}
\label{tab:cycle}
\end{table}

\begin{table}
\centering
\caption{Results for Hub graphs:$p = 100$ and $p_1 = 20$}
{\scriptsize{
   \begin{tabular}{llllllll} \toprule
      Model & $n$ & Alg & WEDR & FDR & TPR &ED & $|\widehat{G}|$ \\
\toprule
%& \multicolumn{2}{c}{$\CHo$, $p = 100$} \\
$\HBo$ & \multirow{6}{*}{} $300$ & $\JgL$  & $\textbf{0.204}$ ($0.004$) & $0.039$ ($0.001$) & $0.755$ ($0.002$) & $22.3$ ($0.151$) & $63.7$ \\
{$p = 100$} & & $\gL$  & $0.154$ ($0.004$) & $0.038$ ($0.001$) & $0.758$ ($0.002$) & $22.1$ ($0.130$) & $63.8$ \\  \cdashline{3-8} 
\multirow{4}{*}{} && $\JfPC$  & $\textbf{0.204}$ ($0.004$) & $0.038$ ($0.001$) & $0.753$ ($0.002$) & $22.4$ ($0.160$) & $63.4$ \\
&& $\fPC$  & $0.193$ ($0.004$) & $0.038$ ($0.001$) & $0.762$ ($0.002$) & $21.7$ ($0.143$) & $64.2$ \\  \cdashline{3-8} 
&&$\JnL$  & $0.245$ ($0.005$) & $0.089$ ($0.001$) & $0.750$ ($0.002$) & $26.2$ ($0.174$) & $66.7$ \\
&&$\nL$ & $\textbf{0.247}$ ($0.005$) & $0.098$ ($0.002$) & $0.752$ ($0.002$) & $26.8$ ($0.198$) & $67.6$ \\
\hline
$\HBt$ & \multirow{6}{*}{} $300$ & $\JgL$  & $\textbf{0.044}$ ($0.002$) & $0.047$ ($0.001$) & $0.710$ ($0.001$) & $26.7$ ($0.116$) & $61.2$ \\
{$p = 100$} & & $\gL$  & $0.013$ ($0.002$) & $0.043$ ($0.001$) & $0.716$ ($0.001$) & $26.0$ ($0.121$) & $61.4$ \\  \cdashline{3-8} 
\multirow{4}{*}{} && $\JfPC$  & $\textbf{0.048}$ ($0.002$) & $0.043$ ($0.001$) & $0.709$ ($0.001$) & $26.5$ ($0.108$) & $60.8$ \\
&& $\fPC$  & $0.029$ ($0.002$) & $0.038$ ($0.001$) & $0.718$ ($0.001$) & $25.5$ ($0.121$) & $61.3$ \\  \cdashline{3-8} 
&&$\JnL$  & $\textbf{0.054}$ ($0.003$) & $0.083$ ($0.001$) & $0.704$ ($0.001$) & $29.6$ ($0.146$) & $63.0$ \\
&&$\nL$ & $0.0467$ ($0.002$) & $0.096$ ($0.001$) & $0.700$ ($0.001$) & $30.7$ ($0.138$) & $63.5$ \\
\hline
$\HBo$ & \multirow{6}{*}{} $500$ & $\JgL$  & $\textbf{0.413}$ ($0.007$) & $0.026$ ($0.001$) & $0.870$ ($0.002$) & $12.4$ ($0.156$) & $72.4$ \\
{$p = 100$} & & $\gL$  & $0.364$ ($0.007$) & $0.035$ ($0.001$) & $0.863$ ($0.002$) & $13.7$ ($0.144$) & $72.5$ \\  \cdashline{3-8} 
\multirow{4}{*}{} && $\JfPC$  & $0.438$ ($0.007$) & $0.027$ ($0.001$) & $0.878$ ($0.002$) & $11.9$ ($0.148$) & $73.1$ \\
&& $\fPC$  & $\textbf{0.448}$ ($0.007$) & $0.027$ ($0.001$) & $0.882$ ($0.001$) & $11.6$ ($0.141$) & $73.4$ \\  \cdashline{3-8} 
&&$\JnL$  & $0.507$ ($0.006$) & $0.076$ ($0.001$) & $0.890$ ($0.001$) & $14.9$ ($0.152$) & $78.2$ \\
&&$\nL$ & $\textbf{0.52}$ ($0.007$) & $0.091$ ($0.001$) & $0.893$ ($0.002$) & $15.9$ ($0.191$) & $79.6$ \\
\hline
$\HBt$ & \multirow{6}{*}{} $500$ & $\JgL$  & $\textbf{0.086}$ ($0.003$) & $0.042$ ($0.001$) & $0.794$ ($0.001$) & $19.8$ ($0.086$) & $68.0$ \\
{$p = 100$} & & $\gL$  & $0.050$ ($0.002$) & $0.047$ ($0.001$) & $0.789$ ($0.001$) & $20.6$ ($0.098$) & $68.0$ \\  \cdashline{3-8} 
\multirow{4}{*}{} && $\JfPC$  & $\textbf{0.097}$ ($0.003$) & $0.040$ ($0.001$) & $0.798$ ($0.001$) & $19.3$ ($0.109$) & $68.2$ \\
&& $\fPC$  & $0.087$ ($0.003$) & $0.044$ ($0.001$) & $0.797$ ($0.001$) & $19.7$ ($0.111$) & $68.4$ \\  \cdashline{3-8} 
&&$\JnL$  & $\textbf{0.123}$ ($0.004$) & $0.084$ ($0.002$) & $0.804$ ($0.001$) & $22.2$ ($0.15$) & $72.1$ \\
&&$\nL$ & $0.106$ ($0.003$) & $0.105$ ($0.002$) & $0.801$ ($0.001$) & $24.1$ ($0.143$) & $73.4$ \\
    \bottomrule
    \end{tabular}
}}
\label{tab:hub}
\end{table}

\begin{table}
\centering
\caption{Results for Neighborhood graph, $p = 300$ and $p_1 = 30$}
{\scriptsize{
   \begin{tabular}{llllllll} \toprule
      Model & $n$ & Alg & WEDR & FDR & TPR &ED & $|\widehat{G}|$ \\
\toprule
%& \multicolumn{2}{c}{$\CHo$, $p = 100$} \\
$\NBo$ & \multirow{6}{*}{} $300$ & $\JgL$  & $\textbf{0.251}$ ($0.002$) & $0.030$ ($0.000$) & $0.813$ ($0.000$) & $126$ ($0.329$) & $498$ \\
{$p = 100$} & & $\gL$  & $0.102$ ($0.0015$) & $0.039$ ($0.000$) & $0.806$ ($0.001$) & $135$ ($0.345$) & $498$ \\  \cdashline{3-8} 
\multirow{4}{*}{} && $\JfPC$  & $\textbf{0.259}$ ($0.002$) & $0.031$ ($0.000$) & $0.814$ ($0.000$) & $126$ ($0.260$) & $499$ \\
&& $\fPC$  & $0.255$ ($0.002$) & $0.036$ ($0.000$) & $0.813$ ($0.000$) & $129$ ($0.330$) & $501$ \\  \cdashline{3-8} 
&&$\JnL$  & $\textbf{0.254}$ ($0.002$) & $0.035$ ($0.000$) & $0.812$ ($0.001$) & $129$ ($0.461$) & $500$ \\
&&$\nL$ & $0.226$ ($0.002$) & $0.039$ ($0.000$) & $0.804$ ($0.001$) & $136$ ($0.458$) & $497$ \\
\hline
$\NBo$ & \multirow{6}{*}{} $300$ & $\JgL$  & $\textbf{0.005}$ ($0.000$) & $0.043$ ($0.000$) & $0.784$ ($0.001$) & $149$ ($0.385$) & $486$ \\
{$p = 100$} & & $\gL$  & $0.000$ ($0.000$) & $0.036$ ($0.000$) & $0.790$ ($0.000$) & $142$ ($0.259$) & $487$ \\  \cdashline{3-8} 
\multirow{4}{*}{} && $\JfPC$  & $\textbf{0.004}$ ($0.000$) & $0.042$ ($0.000$) & $0.784$ ($0.001$) & $148$ ($0.376$) & $486$ \\
&& $\fPC$  & $0.003$ ($0.000$) & $0.048$ ($0.000$) & $0.782$ ($0.000$) & $153$ ($0.239$) & $488$ \\  \cdashline{3-8} 
&&$\JnL$  & $\textbf{0.005}$ ($0.000$) & $0.046$ ($0.000$) & $0.783$ ($0.000$) & $151$ ($0.356$) & $488$ \\
&&$\nL$ & $0.003$ ($0.000$) & $0.050$ ($0.000$) & $0.775$ ($0.000$) & $158$ ($0.374$) & $485$ \\
\hline
$\NBo$ & \multirow{6}{*}{} $500$ & $\JgL$  & $\textbf{0.449}$ ($0.001$) & $0.018$ ($0.000$) & $0.921$ ($0.000$) & $57.1$ ($0.199$) & $557$ \\
{$p = 100$} & & $\gL$  & $0.319$ ($0.002$) & $0.035$ ($0.000$) & $0.905$ ($0.000$) & $75.8$ ($0.242$) & $557$ \\  \cdashline{3-8} 
\multirow{4}{*}{} && $\JfPC$  & $0.489$ ($0.002$) & $0.019$ ($0.000$) & $0.925$ ($0.000$) & $52.8$ ($0.189$) & $558$ \\
&& $\fPC$  & $\textbf{0.496}$ ($0.002$) & $0.023$ ($0.000$) & $0.920$ ($0.000$) & $60.2$ ($0.214$) & $559$ \\  \cdashline{3-8} 
&&$\JnL$  & $\textbf{0.508}$ ($0.003$) & $0.027$ ($0.000$) & $0.929$ ($0.000$) & $57.9$ ($0.348$) & $567$ \\
&&$\nL$ & $0.494$ ($0.003$) & $0.033$ ($0.000$) & $0.927$ ($0.000$) & $62.3$ ($0.400$) & $570$ \\
\hline
$\NBt$ & \multirow{6}{*}{} $500$ & $\JgL$  & $\textbf{0.008}$ ($0.000$) & $0.033$ ($0.000$) & $0.870$ ($0.000$) & $95.0$ ($0.206$) & $534$ \\
{$p = 100$} & & $\gL$  & $0.000$ ($0.000$) & $0.034$ ($0.000$) & $0.869$ ($0.000$) & $96.0$ ($0.214$) & $534$ \\  \cdashline{3-8} 
\multirow{4}{*}{} && $\JfPC$  & $\textbf{0.009}$ ($0.000$) & $0.032$ ($0.000$) & $0.870$ ($0.000$) & $94.2$ ($0.215$) & $534$ \\
&& $\fPC$  & ${0.005}$ ($0.000$) & $0.040$ ($0.000$) & $0.865$ ($0.000$) & $102$ ($0.207$) & $536$ \\  \cdashline{3-8} 
&&$\JnL$  & $\textbf{0.001}$ ($0.000$) & $0.038$ ($0.000$) & $0.871$ ($0.000$) & $97.3$ ($0.220$) & $538$ \\
&&$\nL$ & $0.005$ ($0.000$) & $0.043$ ($0.000$) & $0.870$ ($0.000$) & $101$ ($0.234$) & $540$ \\
    \bottomrule
    \end{tabular}
}}
\label{tab:nb300}
\end{table}

\begin{remark}[Graphical Lasso]
Of all the algorithms, graphical Lasso ($\gL$) performs the worst.  On the other hand, junction tree based $\gL$ significantly improves the performance of $\gL$.  Moreover, the performance of $\JgL$ is comparable, and sometimes even better, when compared to $\JfPC$ and $\JnL$.  This suggests that when using $\gL$ in practice, it is beneficial to apply a screening algorithm to remove some edges and then use the junction tree framework in conjunction with $\gL$.
\end{remark}

\begin{remark}[PC-Algorithm and Neighborhood Selection]
Although using junction trees in conjunction with the PC-Algorithm ($\fPC$) and neighborhood selection ($\nL$) does improve the graph estimation performance, the difference is not as significant as $\gL$.  The reason is because both $\fPC$ and $\nL$ make use of the local Markov property in the graph $H$.  The junction tree framework further improves the performance of these algorithms by making use of the global Markov property, in addition to the local Markov property.
\end{remark}

\begin{remark}[Chain Graph]
Although the chain graph does not satisfy the conditions in (A6), the junction tree estimates still outperforms the non-junction tree estimates.  This suggests the advantages of using junction trees beyond the graphs considered in (A6).  We suspect that correlation decay properties, which have been studied extensively in \citet{AnimaTanWillsky2011a,AnimaTanWillsky2011b}, can be used to weaken the assumption in (A6).
\end{remark}

\begin{remark}[Hub Graph]
For the hub graph $\HBo$, the junction tree estimate does not result in a significant difference in performance, especially for the $\fPC$ and $\nL$ algorithms.  This is mainly because this graph is extremely sparse with multiple components.  For the number of observations considered, $H$ removes a significant number of edges.  However, for $\HBt$, the junction tree estimate, in general, performs slightly better.  This is because the parameters associated with the weak edges in $\HBt$ are smaller than that of $\HBo$.
\end{remark}

\begin{remark}[General Conclusion]
We see that, in general, the WEDR and TPR are higher, while the FDR and ED are lower, for junction tree based algorithms.  This clearly suggests that using junction trees results in more accurate graph estimation.  Moreover, the higher WEDR suggest that the main differences between the two algorithms are over the weak edges, i.e., junction tree based algorithms are estimating more weak edges when compared to a non junction tree based algorithm.
\end{remark}

\begin{figure}
\centering
\subfigure[Junction tree based graphical Lasso]{
\includegraphics[scale=0.55]{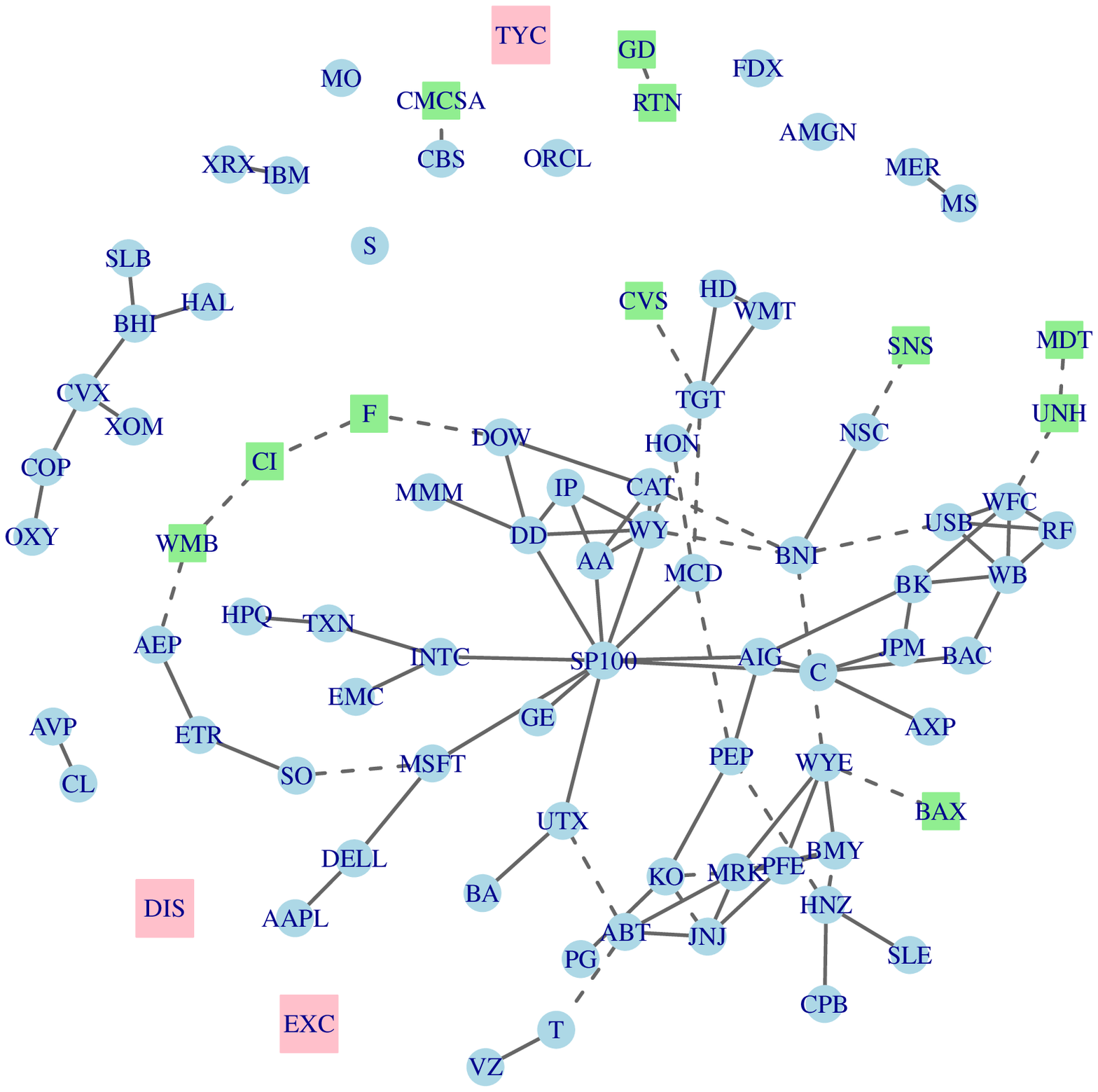}
} 
\subfigure[Graphical Lasso]{
\includegraphics[scale=0.55]{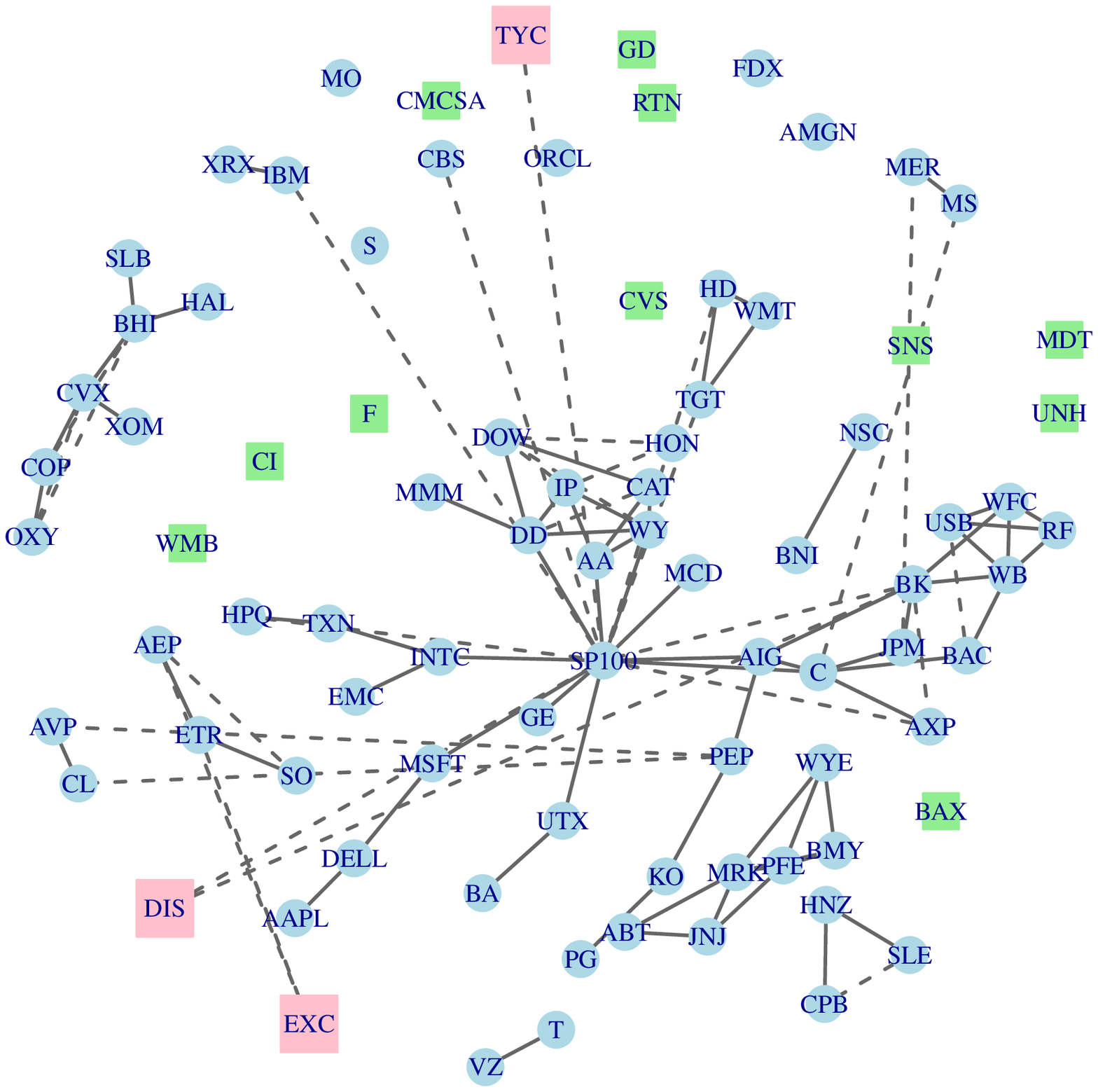}
}
\caption{Graph over a subset of companies in the S\&P 100.  The positioning of the vertices is chosen so that the\textit{ junction tree based graph is aesthetically pleasing}.  The edges common in (a) and (b) are marked by bold lines and the remaining edges are marked by dashed lines}
\label{fig:stockJT}
\end{figure}

\begin{figure}
\centering
\subfigure[Junction tree based graphical Lasso]{
\includegraphics[scale=0.55]{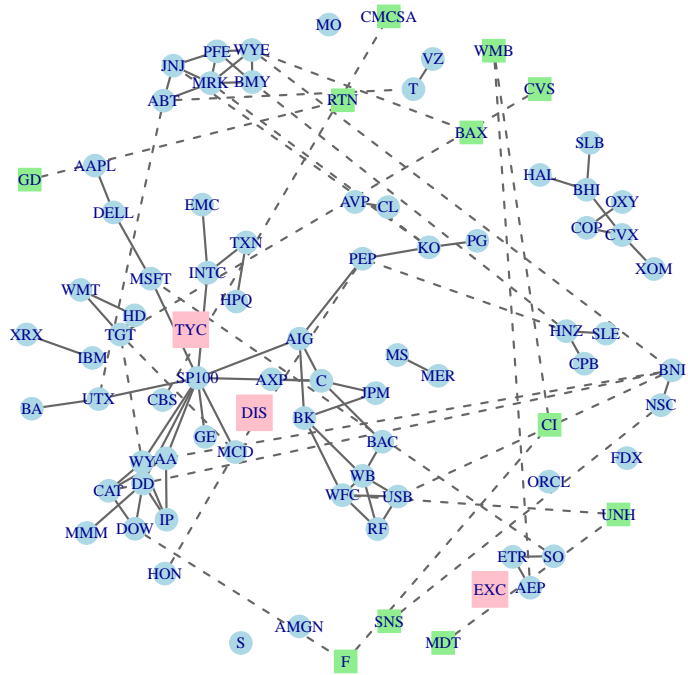}
} 
\subfigure[Graphical Lasso]{
\includegraphics[scale=0.55]{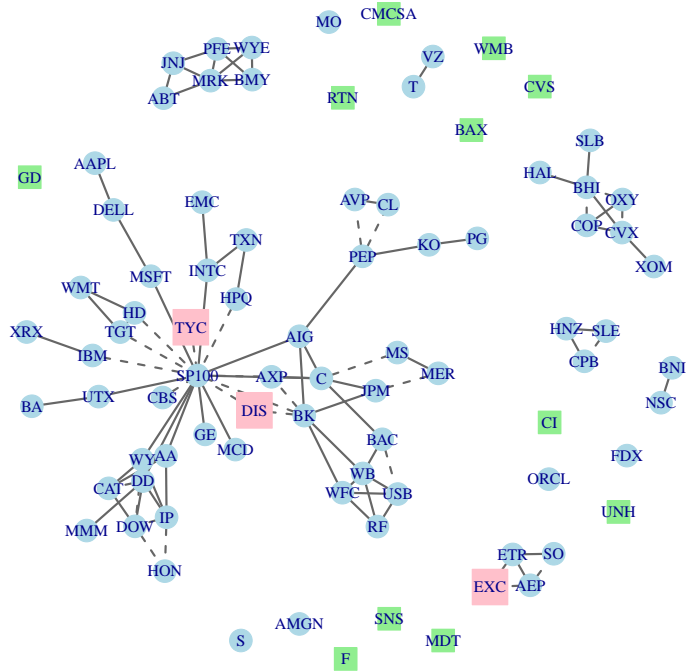}
}
\caption{Graph over a subset of companies in the S\&P 100.  The positioning of the vertices is chosen so that the \textit{graphical Lasso based graph is aesthetically pleasing}. The edges common in (a) and (b) are marked by bold lines and the remaining edges are marked by dashed lines.}
\label{fig:stockNOJT}
\end{figure}

\subsection{Analysis of Stock Returns Data}
\label{sub:realdata}

We applied our methods to the data set in \citet{ChoiTanAnimaWillsky2011} of $n = 216$ monthly stock returns of $p = 85$ companies in the S\&P 100.  We computed $H$ using $\kappa = 1$.  We applied $\JgL$ using EBIC with $\gamma = 0.5$ and applied $\gL$ so that both graphs have the same number of edges.  This allows us to objectively compare the $\gL$ and $\JgL$ graphs.  Figure~\ref{fig:stockJT} shows the two estimated graphs in such a way that the vertices are positioned so that the $\JgL$ graph looks aesthetically pleasing.  In Figure~\ref{fig:stockNOJT}, the vertices are positioned so that $\gL$ looks aesthetically pleasing.  In each graph, we mark the common edges by bold lines and the remaining edges by dashed lines.  Some conclusions that we draw from the estimated graphs are summarized as follows:

\begin{itemize}
\item The $\gL$ graph in Figure~\ref{fig:stockNOJT}(b) seems well structured with multiple different clusters of nodes with companies that seem to be related to each other.  A similar clustering is seen for the $\JgL$ graph in Figure~\ref{fig:stockJT}(a) with the exception that there are now connections between the clusters.  As observed in \citet{ChoiTanAnimaWillsky2011,chandrasekaran2010latent}, it has been hypothesized that the ``actual" graph over the companies is dense since there are several unobserved companies that induce conditional dependencies between the observed companies.  These induced conditional dependencies can be considered to be the weak edges of the ``actual" graph.  Thus, our results suggest that the junction tree based algorithm is able to detect such weak edges.

\item We now focus on some specific edges and nodes in the graphs.  The $11$ vertices represented by smaller squares and shaded in green are not connected to any other vertex in $\gL$.  On the other hand, all these $11$ vertices are connected to at least one other vertex in $\JgL$ (see Figure~\ref{fig:stockJT}).  Moreover, several of these edges are meaningful.  For example, CBS and CMCSA are in the television industry, TGT and CVS are stores, AEP and WMB are energy companies, GD and RTN are defense companies, and MDT and UNH are in the healthcare industry.
Finally, the three vertices represented by larger squares and shaded in pink, are not connected to any vertex in $\JgL$ and are connected to at least one other vertex in $\gL$.  Only the edges associated with EXC seem to be meaningful.

\end{itemize}

\subsection{Analysis of Gene Expression Data}
\label{sub:realdata2}

Graphical models have been used extensively for studying gene interactions using gene expression data \citep{Nevins04sparsegraphical,wille2004sparse}.  The gene expression data we study is the Lymph node status data which contains $n = 148$ expression values from $p = 587$ genes \citep{li2010inexact}.  Since there is no ground truth available, the main aim in this section is to highlight the differences between the estimates $\JgL$ (junction tree estimate) and $\gL$ (non junction tree estimate).  Just like in the stock returns data, we compute the graph $H$ using $\kappa = 1$. Both the $\JgL$ and $\gL$ graphs contain $831$ edges.  Figure~\ref{fig:gene} shows the graphs $\JgL$ and $\gL$ under different placements of the vertices.  We clearly see significant differences between the estimated graphs.  This suggests that using the junction tree framework may lead to new scientific interpretations when studying biological data.  

\begin{figure}
\centering
\subfigure[Junction tree based graphical Lasso]{
\includegraphics[scale=0.44]{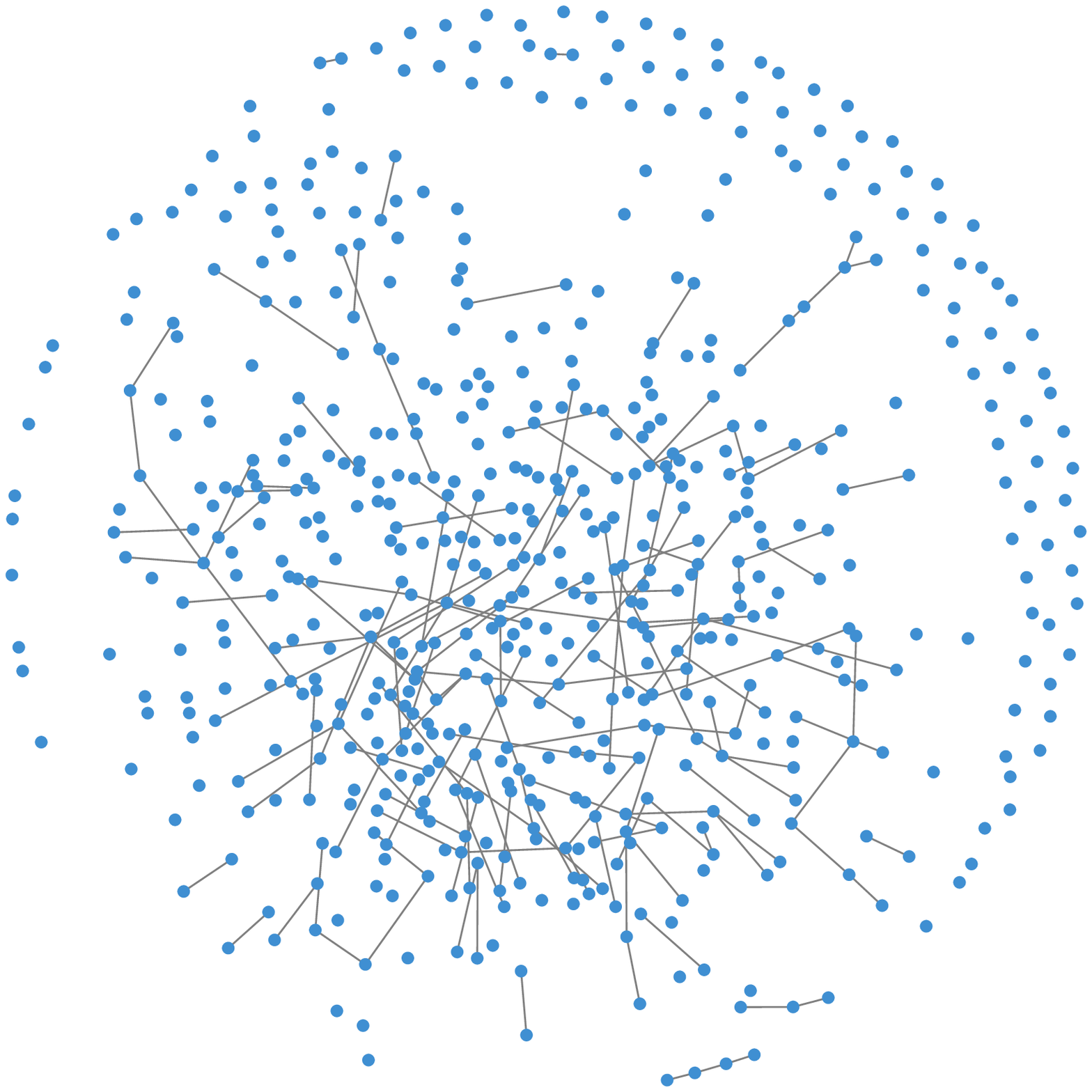}
} 
\subfigure[Graphical Lasso]{
\includegraphics[scale=0.44]{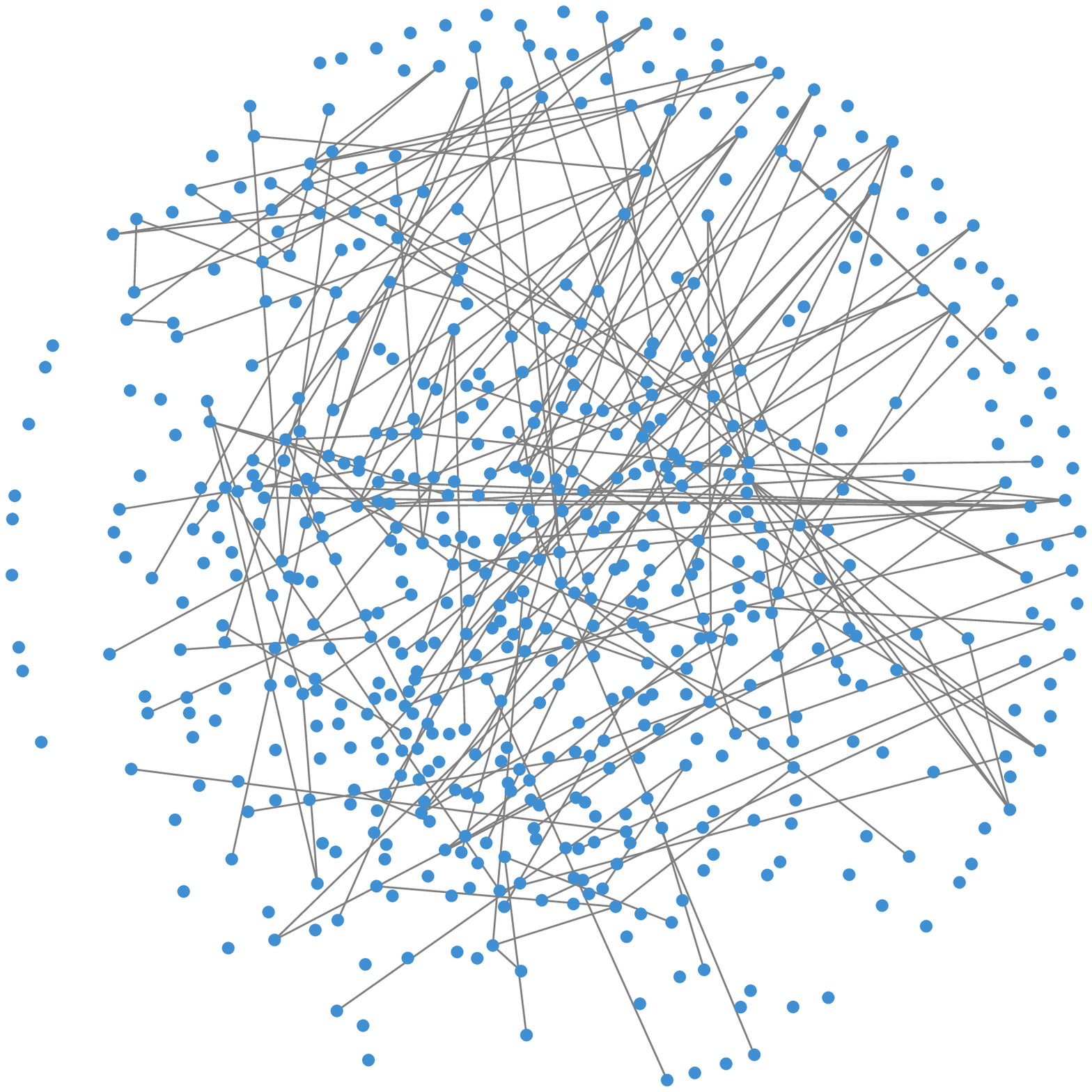}
}
\subfigure[Junction tree based graphical Lasso]{
\includegraphics[scale=0.44]{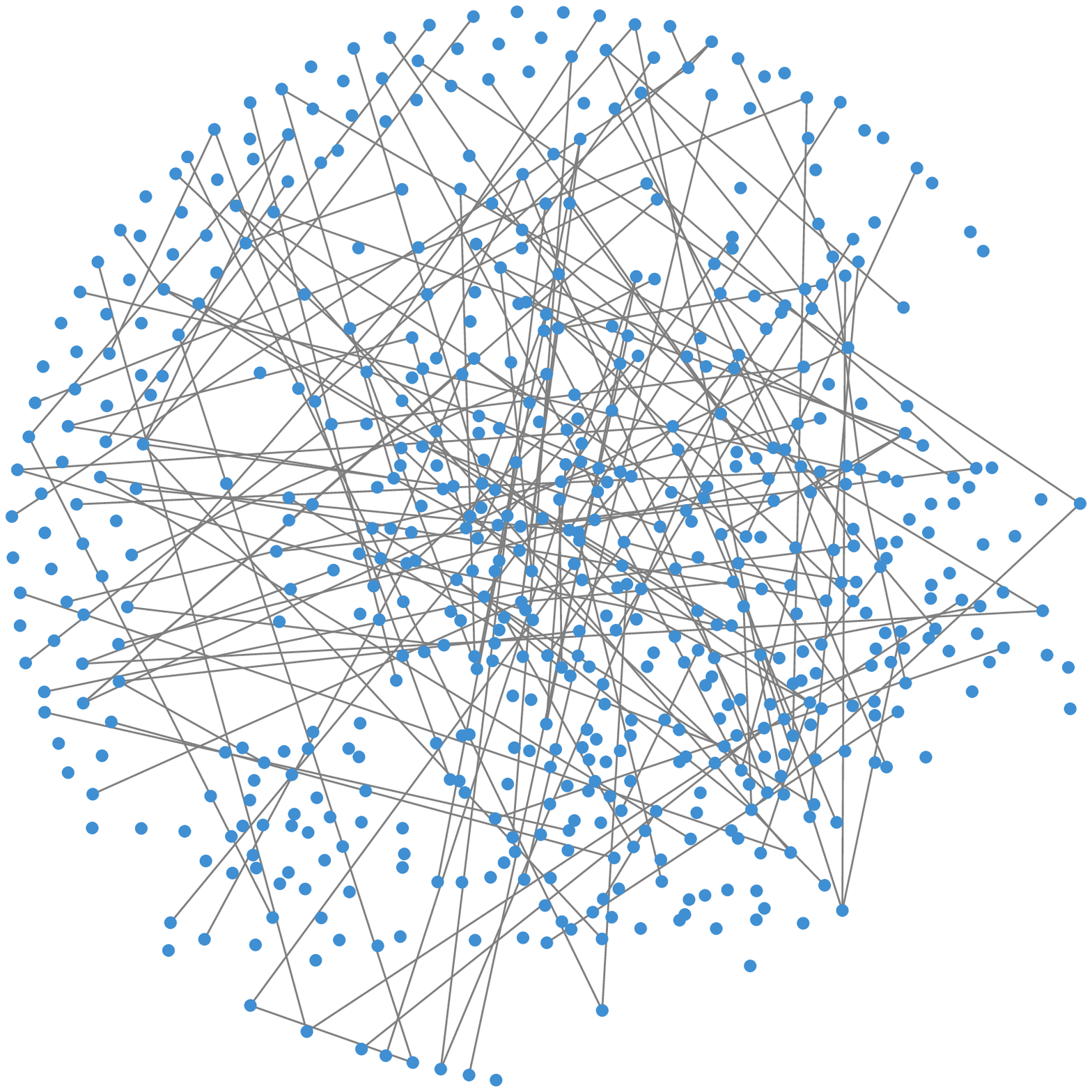}
} 
\subfigure[Graphical Lasso]{
\includegraphics[scale=0.44]{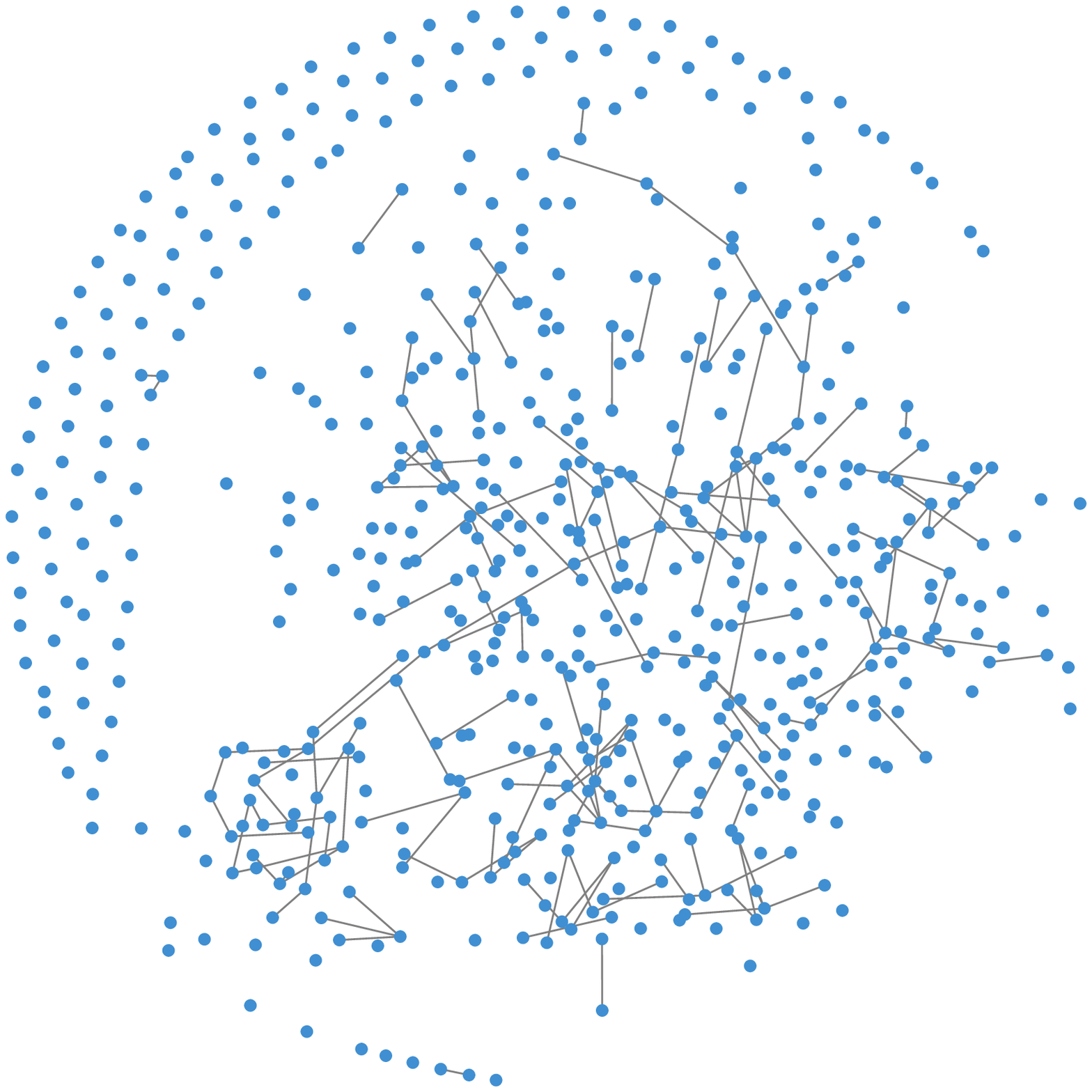}
}
\caption{Graph over genes computed using gene expression data.  For (a) and (b), the vertices are chosen so that the junction tree estimate is aesthetically pleasing.  For (c) and (d), the vertices are chosen so that the graphical Lasso estimate is aesthetically pleasing.  Further, in (a) and (c), we only show edges that are estimated in the junction tree estimate, but not estimated using graphical Lasso.  Similarly, for (b) and (c), we only show edges that are estimated by graphical Lasso, but not by the junction tree estimate.
}
\label{fig:gene}
\end{figure}

%\begin{figure}
%\centering
%\includegraphics[scale=0.4]{JTAraNewDashed.eps}
%\caption{Graph estimated using the junction tree framework when used in conjunction with the graphical Lasso.  The vertices with the same color refer to genes that are in the same pathway.  The dashed edges are not estimated when using the graphical Lasso without the junction tree framework.}
%\label{fig:jtgene}
%\end{figure}
%\begin{figure}
%\centering
%\includegraphics[scale=0.4]{NoJTAraNewDashed.eps}
%\caption{Graph estimated using graphical Lasso without using the junction tree framework.  The vertices with the same color refer to genes that are in the same pathway.  The dashed are not estimated when using the junction tree framework.}
%\label{fig:nojtgene}
%\end{figure}

\section{Summary and Future Work}
\label{sec:summary}

We have outlined a general framework that can be used as a wrapper around any arbitrary undirected graphical model selection (UGMS) algorithm for improved graph estimation.  Our framework takes as input a graph $H$ that contains all (or most of) the edges in $G^*$, decomposes the UGMS problem into multiple subproblems using a junction tree representation of $H$, and then solves subprolems iteratively to estimate a graph.  Our theoretical results show that certain weak edges, which cannot be estimated using standard algorithms, can be estimated when using the junction tree framework.  We supported the theory with numerical simulations on both synthetic and real world data.  All the data and code used in our numerical simulations can be found at \url{http://www.ima.umn.edu/~dvats/JunctionTreeUGMS.html}.

Our work motivates several interesting future research directions.  In our framework, we used a graph $H$ to decompose the UGMS problem into multiple subproblems.  Alternatively, we can also focus on directly finding such decompositions.  Another interesting research direction is to use the decompositions to develop parallel algorithms for UGMS for estimating extremely large graphs.  Finally, motivated by the differences in the graphs obtained using gene expression data, another research problem of interest is to study the scientific consequences of using the junction tree framework on various computational biology data sets.

\section*{Acknowledgement}
The first author thanks the Institute for Mathematics and its Applications (IMA) for financial support in the form of a postdoctoral fellowship.
The authors thank Vincent Tan for discussions and comments on an earlier version of the paper.  The authors thank the anonymous reviewers for comments which significantly improved this manuscript.

\appendix

\section{Marginal Graph}
\label{app:marginal}

\begin{definition}
\label{def:marginalgraph}
The marginal graph $G^{*,m}[A]$ of a graph $G^*$ over the nodes $A$ is defined as a graph with the following properties
\begin{enumerate}
\item $E(G^*[A]) \subseteq E(G^{*,m}[A])$.
\item For an edge $(i,j) \in E(K_A) \backslash E(G^*[A])$, if all paths from $i$ to $j$ in $G^*$ pass through a subset of the nodes in $A$, then $(i,j) \notin G^{*,m}[A]$.
\item For an edge $(i,j) \in E(K_A) \backslash E(G^*[A])$, if there exists a path from $i$ to $j$ in $G^*$ such that all nodes in the path, except $i$ and $j$, are in $V \backslash A$, then $(i,j) \in G^{*,m}[A]$.
\end{enumerate}
\end{definition}

The graph $K_A$ is the complete graph over the vertices $A$.  The first condition in Definition~\ref{def:marginalgraph} says that the marginal graph contains all edges in the induced subgraph over $A$.  The second and third conditions say which edges not in $G^*[A]$ are in the marginal graph.  As an example, consider the graph in Figure~\ref{fig:exampleabc}(a) and let $A = \{1,2,3,4,5\}$.  From the second condition, the edge $(3,4)$ is not in the marginal graph since all paths from $3$ to $4$ pass through a subset of the nodes in $A$.  From the third condition, the edge $(4,5)$ is in the marginal graph since there exists a path $\{4,8,5\}$ that does not go through any nodes in $A \backslash \{4,5\}$.  Similarly, the marginal graph over $A = \{4,5,6,7,8\}$ can be constructed as in Figure~\ref{fig:exampleabc}(c).  The importance of marginal graphs is highlighted in the following proposition.

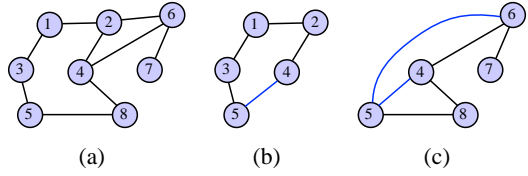
\begin{figure}
\centering
\subfigure[]{

\scalebox{0.5} % Change this value to rescale the drawing.
{
\begin{pspicture}(0,-1.6800001)(4.72,1.6800001)
\definecolor{color0b}{rgb}{0.8,0.8,1.0}
\pscircle[linewidth=0.04,dimen=outer,fillstyle=solid,fillcolor=color0b](1.11,1.07){0.35}
\usefont{T1}{ptm}{m}{n}
\rput(1.0646875,1.0949999){\large 1}
\pscircle[linewidth=0.04,dimen=outer,fillstyle=solid,fillcolor=color0b](2.69,1.13){0.35}
\usefont{T1}{ptm}{m}{n}
\rput(2.6825,1.155){\large 2}
\pscircle[linewidth=0.04,dimen=outer,fillstyle=solid,fillcolor=color0b](0.59,-1.33){0.35}
\usefont{T1}{ptm}{m}{n}
\rput(0.56,-1.3050001){\large 5}
\pscircle[linewidth=0.04,dimen=outer,fillstyle=solid,fillcolor=color0b](4.37,1.33){0.35}
\usefont{T1}{ptm}{m}{n}
\rput(4.3540626,1.3549999){\large 6}
\pscircle[linewidth=0.04,dimen=outer,fillstyle=solid,fillcolor=color0b](3.11,-1.33){0.35}
\usefont{T1}{ptm}{m}{n}
\rput(3.078125,-1.3050001){\large 8}
\pscircle[linewidth=0.04,dimen=outer,fillstyle=solid,fillcolor=color0b](1.93,-0.19000006){0.35}
\usefont{T1}{ptm}{m}{n}
\rput(1.9278125,-0.16500005){\large 4}
\pscircle[linewidth=0.04,dimen=outer,fillstyle=solid,fillcolor=color0b](3.77,-0.13000005){0.35}
\usefont{T1}{ptm}{m}{n}
\rput(3.7540622,-0.10500015){\large 7}
\pscircle[linewidth=0.04,dimen=outer,fillstyle=solid,fillcolor=color0b](0.35,-0.13000005){0.35}
\usefont{T1}{ptm}{m}{n}
\rput(0.315625,-0.10500005){\large 3}
\psline[linewidth=0.04cm](1.44,1.0999999)(2.36,1.12)
\psline[linewidth=0.04cm](0.92,0.81999993)(0.52,0.15999995)
\psline[linewidth=0.04cm](0.38,-0.46000004)(0.52,-1.0)
\psline[linewidth=0.04cm](0.92,-1.32)(2.78,-1.32)
\psline[linewidth=0.04cm](2.16,-0.40000004)(2.9,-1.08)
\psline[linewidth=0.04cm](2.52,0.84)(2.08,0.11999995)
\psline[linewidth=0.04cm](3.0,1.1999999)(4.04,1.3)
\psline[linewidth=0.04cm](2.22,-0.02000005)(4.14,1.0999999)
\psline[linewidth=0.04cm](4.28,1.02)(3.92,0.17999995)
\end{pspicture} 
}
}
\subfigure[]{
% Generated with LaTeXDraw 2.0.8
% Sun Sep 25 09:49:17 CDT 2011
% \usepackage[usenames,dvipsnames]{pstricks}
% \usepackage{epsfig}
% \usepackage{pst-grad} % For gradients
% \usepackage{pst-plot} % For axes
\scalebox{0.5} % Change this value to rescale the drawing.
{
\begin{pspicture}(0,-1.58)(3.04,1.58)
\definecolor{color0b}{rgb}{0.8,0.8,1.0}
\definecolor{color4693}{rgb}{0.0,0.2,1.0}
\pscircle[linewidth=0.04,dimen=outer,fillstyle=solid,fillcolor=color0b](1.11,1.1700001){0.35}
\usefont{T1}{ptm}{m}{n}
\rput(1.0646875,1.195){\large 1}
\pscircle[linewidth=0.04,dimen=outer,fillstyle=solid,fillcolor=color0b](2.69,1.23){0.35}
\usefont{T1}{ptm}{m}{n}
\rput(2.6825,1.255){\large 2}
\pscircle[linewidth=0.04,dimen=outer,fillstyle=solid,fillcolor=color0b](0.59,-1.23){0.35}
\usefont{T1}{ptm}{m}{n}
\rput(0.56,-1.205){\large 5}
\pscircle[linewidth=0.04,dimen=outer,fillstyle=solid,fillcolor=color0b](1.93,-0.09){0.35}
\usefont{T1}{ptm}{m}{n}
\rput(1.9278125,-0.065){\large 4}
\pscircle[linewidth=0.04,dimen=outer,fillstyle=solid,fillcolor=color0b](0.35,-0.03){0.35}
\usefont{T1}{ptm}{m}{n}
\rput(0.315625,-0.0050){\large 3}
\psline[linewidth=0.04cm](1.44,1.2)(2.36,1.22)
\psline[linewidth=0.04cm](0.92,0.92)(0.52,0.26)
\psline[linewidth=0.04cm](0.38,-0.36)(0.52,-0.9)
\psline[linewidth=0.04cm](2.52,0.94)(2.08,0.22)
\psline[linewidth=0.04cm,linecolor=color4693](1.68,-0.32)(0.82,-1.0)
\end{pspicture} 
}
}
\subfigure[]{
% Generated with LaTeXDraw 2.0.8
% Sun Sep 25 09:52:56 CDT 2011
% \usepackage[usenames,dvipsnames]{pstricks}
% \usepackage{epsfig}
% \usepackage{pst-grad} % For gradients
% \usepackage{pst-plot} % For axes
\scalebox{0.5} % Change this value to rescale the drawing.
{
\begin{pspicture}(0,-1.68)(4.48,1.68)
\definecolor{color0b}{rgb}{0.8,0.8,1.0}
\definecolor{color4693}{rgb}{0.0,0.2,1.0}
\pscircle[linewidth=0.04,dimen=outer,fillstyle=solid,fillcolor=color0b](0.35,-1.33){0.35}
\usefont{T1}{ptm}{m}{n}
\rput(0.295,-1.3050001){\large 5}
\pscircle[linewidth=0.04,dimen=outer,fillstyle=solid,fillcolor=color0b](4.13,1.33){0.35}
\usefont{T1}{ptm}{m}{n}
\rput(4.0960937,1.3549999){\large 6}
\pscircle[linewidth=0.04,dimen=outer,fillstyle=solid,fillcolor=color0b](2.87,-1.33){0.35}
\usefont{T1}{ptm}{m}{n}
\rput(2.8121874,-1.3050001){\large 8}
\pscircle[linewidth=0.04,dimen=outer,fillstyle=solid,fillcolor=color0b](1.69,-0.19000006){0.35}
\usefont{T1}{ptm}{m}{n}
\rput(1.6767187,-0.16500005){\large 4}
\pscircle[linewidth=0.04,dimen=outer,fillstyle=solid,fillcolor=color0b](3.53,-0.13000005){0.35}
\usefont{T1}{ptm}{m}{n}
\rput(3.4960935,-0.10500015){\large 7}
\psline[linewidth=0.04cm](0.68,-1.32)(2.54,-1.32)
\psline[linewidth=0.04cm](1.92,-0.40000004)(2.66,-1.08)
\psline[linewidth=0.04cm](1.98,-0.02000005)(3.9,1.0999999)
\psline[linewidth=0.04cm](4.04,1.02)(3.68,0.17999995)
\psline[linewidth=0.04cm,linecolor=color4693](1.4,-0.36)(0.56,-1.08)
\psbezier[linewidth=0.04,linecolor=color4693](0.4,-1.0)(0.41931817,-0.23174603)(0.8714608,0.2639724)(1.5784091,0.7669841)(2.2853575,1.2699958)(2.853409,1.42)(3.8,1.2855556)
\end{pspicture} 
}
}
\caption{(a) A graph over eight nodes. (b) The marginal graph over $\{1,2,3,4,5\}$. (c) The marginal graph over $\{4,5,6,7,8\}$.}
\label{fig:exampleabc}
\end{figure}

\begin{proposition}
\label{prop:marg}
If $P_X > 0$ is Markov on $G^* = (V,E(G^*))$ and not Markov on any subgraph of $G^*$, then for any subset of vertices $A \subseteq V$, $P_{X_A}$ is Markov on the marginal graph $G^{*,m}[A]$ and not Markov on any subgraph of $G^{*,m}[A]$.
\end{proposition}
\begin{proof}
Suppose $P_{X_A}$ is Markov on the graph $\check{G}_A$ and not Markov on any subgraph of $\check{G}_A$.  We will show that $\check{G}_A = G^m[A]$. 
\begin{itemize}
\item If $(i,j) \in G$, then $X_i \not\ind X_j | X_{S}$ for every $S \subseteq V \backslash \{i,j\}$.  Thus, $G[A] \subset \check{G}_A$.  
\item For any edge $(i,j) \in K_A \backslash G[A]$, suppose that for every path from $i$ to $j$ contains at least one node from $A \backslash \{i,j\}$.  Then, there exists a set of nodes $S \subseteq A \backslash \{i,j\}$ such that $X_i \ind X_j | X_S$ and $(i,j) \notin \check{G}_A$.
\item For any edge $(i,j) \in K_A \backslash G[A]$, suppose that there exists a path from $i$ to $j$ such that all nodes in the path, except $i$ and $j$, are in $V \backslash A$.  This means we cannot find a separator for $i$ and $j$ in the set $A$, so $(i,j) \in \check{G}_A$.
\end{itemize}
From the construction of $\check{G}_A$ and Definition~\ref{def:marginalgraph}, it is clear that $\check{G}_A = G^m[A]$.
\end{proof}

Using Proposition~\ref{prop:marg}, it is clear that if the UGMS algorithm $\Psi$ in Assumption~\ref{ass:cons} is applied to a subset of vertices $A$, the output will be a consistent estimator of the marginal graph $G^{*,m}[A]$.  Note that from Definition~\ref{def:marginalgraph}, although the marginal graph contains all edges in $G^*[A]$, it may contain additional edges as well.  Given only the marginal graph $G^{*,m}[A]$, it is not clear how to identify edges that are in $G^*[A]$.  For example, suppose $G^*$ is a graph over four nodes and let the graph be a single cycle.  The marginal graph over any subset of three nodes is always the complete graph.  Given the complete graph over three nodes, computing the induced subgraph over the three nodes is nontrivial.

\section{Examples of UGMS Algorithms}
\label{app:examples}

We give examples of standard UGMS algorithms and show how they can be used to implement step 3 in Algorithm~\ref{alg:ugmsregion} when estimating edges in a region of a region graph.  For simplicity, we review algorithms for UGMS when $P_X$ is a Gaussian distribution with mean zero and covariance $\Sigma^*$.  Such distributions are referred to as Gaussian graphical models.  It is well known \citep{SpeedKiiveri1986} that that the inverse covariance matrix $\Theta^* = (\Sigma^*)^{-1}$, also known as precision matrix, is such that for all $i \ne j$, $\Theta_{ij}^* \ne 0$ if and only if $(i,j) \in E(G^*)$.  In other words, the graph $G^*$ can be estimated given an estimate of the covariance or inverse covariance matrix of $X$.  We review two standard algorithms for estimating $G^*$: graphical Lasso and neighborhood selection using Lasso (nLasso).

\subsection{Graphical Lasso (gLasso)}

Define the empirical covariance matrix $\widehat{S}_A$ over a set of vertices $A \subset V$ as follows:
\begin{equation}
\widehat{S}_A = \frac{1}{n} \sum_{k = 1}^{n} X_A^{(k)} \left(X_A^{(k)}\right)^T \,.
\end{equation}
Recall from Algorithm~\ref{alg:ugmsregion}, we apply a UGMS algorithm $\overline{R}$ to estimate edges in $H'_R$ defined in (\ref{eq:haprime}). The graphical Lasso (gLasso) estimates $\widehat{E}_R$ by solving the following convex optimization problem:
\begin{align}
\widehat{\Theta} &= \arg \underset{ \Theta \succ 0, \Theta_{ij} = 0 \; \forall \; (i,j) \notin H^m\left[\overline{R}\right] }{\max} \left\{ 
\log \det (\Theta) - \text{trace}\left(\widehat{S}_{\overline{R}} \; \Theta \right) - \lambda \sum_{(i,j) \in H_R'} \Theta_{ij}
\right\} \label{eq:glasso} \\
\widehat{E}_R &= \{ (i,j) \in H_A' : \widehat{\Theta}_{ij} \ne 0\} \,.
\end{align}
The graph $H^m[\overline{R}]$ is the marginal graph over $\overline{R}$ (see Appendix~\ref{app:marginal}).
When $\overline{R} = V$, $H = K_V$, and $H_A' = K_V$, the above equations recover the standard gLasso estimator, which was first proposed in \citet{BanerjeeGhaoui2008}.  Equation~(\ref{eq:glasso}) can be solved using algorithms in \citet{YuanLin2007,BanerjeeGhaoui2008,scheinberg2010sparse,hsieh2011sparse}.  Theoretical properties of the estimates $\widehat{\Theta}$ and $\widehat{E}_R$ have been studied in \citet{Ravikumarcovariance2008}.  Note that the regularization parameter in (\ref{eq:glasso}) controls the sparsity of $\widehat{E}_R$.  A larger $\lambda$ corresponds to a sparser solution.  Further, we only regularize the terms in $\Theta_{ij}$ corresponding to the edges that need to be estimated, i.e., the edges in $H_R'$.  Finally, Equation~(\ref{eq:glasso}) also accounts for the edges $H$ by computing the marginal graph over $\overline{R}$.  In general, $H^m\left[\overline{R}\right]$ can be replaced by any graph that is superset of $H^m\left[\overline{R}\right]$.

\subsection{Neighborhood Selection (nLasso)}
Using the local Markov property of undirected graphical models (see Definition~\ref{def:ugm}), we know that if $P_X$ is Markov on $G^*$, then $P\left(X_i \, | \, X_{V \backslash i}\right) = P\left(X_i \,| \,X_{ne_{G^{*}}(i)} \right)$.  This motivates an algorithm for estimating the neighborhood of each node and then combining all these estimates to estimate $G^*$.  For Gaussian graphical models, this can be achieved by solving a Lasso problem \citep{Tibshirani1996} at each node \citep{NicolaiPeter2006}.  Recall that we are interested in estimating all edges in $H'_R$ by applying a UGMS algorithm to $\overline{R}$.  The neighborhood selection using Lasso (nLasso) algorithm is given as follows:
\begin{align}
H'' &= K_{\overline{R}} \backslash H^m\left[\overline{R}\right] \\
\widehat{\beta}^{k} &= 
\arg \min_{\beta_{i} = 0, i \in { ne_{H''}(k) \cup k \cup V \backslash A }} \left\{ \left\| \Xf^n_k - \Xf^n \beta \right\|_2^2 + \lambda \sum_{i \in ne_{H'_R}(k)} \left|\beta_i \right| \right\} \label{eq:nLasso1}\\
\widehat{ne}^{k} &= \left\{ i : \widehat{\beta}^{k}_i \ne 0 \right\} \label{eq:nLasso2}\\
\widehat{E}_R &= \bigcup_{k \in \overline{R}}  \left\{ (k,i) : i \in \widehat{ne}^{k} \right\} \,. \label{eq:nLasso3}
\end{align}
Notice that in the above algorithm if $i$ is estimated to be a neighbor of $j$, then we include the edge $(i,j)$ even if $j$ is not estimated to be a neighbor of $i$.  This is called the union rule for combining neighborhood estimates.  In our numerical simulations, we use the intersection rule to combine neighborhood estimates, i.e., $(i,j)$ is estimated only if $i$ is estimated to be a neighbor of $j$ and $j$ is estimated to be a neighbor of $i$.  Theoretical analysis of nLasso has been carried out in \citet{NicolaiPeter2006,Wainwright2009}.  Note that, when estimating the neighbors of a node $k$, we only penalize the neighbors in $H_R'$.  Further, we use prior knowledge about some of the edges by using the graph $H$ in (\ref{eq:nLasso1}).  References \citet{BreslerMosselSly2008,NetrapalliSanghaviAllerton2010,RavikumarWainwrightLafferty2010} extend the neighborhood selection based method to discrete valued graphical models.

\section{Proof of Proposition~\ref{prop:regionest}}
\label{sub:proofpropregioest}
We first prove the following result.
\begin{lemma}
\label{prop:pathsIn}
For any $(i,j) \in H_{R}'$, there either exists no non-direct path from $i$ to $j$ in $H$ or all non-direct paths in $H$ pass through a subset of $\overline{R}$.
\end{lemma}
\begin{proof}
We first show the result for $R \in {\cal R}^1$.  This means that $R$ is one of the clusters in the junction tree used to construct the region graph and $ch(R)$ is the set of all separators of cardinality greater than one connected to the cluster $R$ in the junction tree.  Subsequently, $\overline{R} = R$.  If $ch(R) = \emptyset$, the claim trivially holds.  Let $ch(R) \ne \emptyset$ and
suppose there exists a non-direct path from $i$ to $j$ that passes through a set of vertices $\bar{S}$ not in $\overline{R}$.
Then, there will exist a separator $S$ in the junction tree such that $S$ separates $\{i,j\}$ and $\bar{S}$.  Thus, all paths in $H$ from $i$ and $j$ to $\bar{S}$ pass through $S$.  This implies that either there is no non-direct path from $i$ to $j$ in $H$ or else we have reached a contradiction about the existence of a non-direct path from $i$ to $j$ that passes through the set $\bar{S}$ not in $\overline{R}$.

Now, suppose $R \in {\cal R}^l$ for $l > 1$.  The set $an(R)$ contains all the clusters in the junction tree than contain $R$.  From the running intersection property of junction trees, all these clusters must form a subtree in the original junction tree.  Merge $\overline{R}$ into one cluster and find a new junction tree ${\cal J}'$ by keeping the rest of the clusters the same.  It is clear $\overline{R}$ will be in the first row of the updated region graph.  The arguments used above can be repeated to prove the claim.
\end{proof}
\noindent
We now prove Proposition~\ref{prop:regionest}.

\smallskip

Case 1: Let $(i,j) \in H'_R$ and $(i,j) \notin G^*$.  If there exists no non-direct path from $i$ to $j$ in $H$, then the edge $(i,j)$ can be estimated by solving a UGMS problem over $i$ and $j$.  By the definition of $\overline{R}$, $i,j \in \overline{R}$.  Suppose there does exist non-direct paths from $i$ to $j$ in $H$.  From Lemma~\ref{prop:pathsIn}, all such paths pass through $\overline{R}$. Thus, the conditional independence of $X_i$ and $X_j$ can be determined from $X_{\overline{R} \backslash \{i,j\}}$.

\smallskip

Case 2: Let $(i,j) \in H'_R$ and $(i,j) \in G^*$.  From Lemma~\ref{prop:pathsIn} and using the fact that $E(G^*) \subseteq E(H)$, we know that all paths from $i$ to $j$ pass through $\overline{R}$.
This means that if $X_i \not\ind X_j | X_{\overline{R} \backslash \{i,j\}}$, then $X_i \not\ind X_j | X_{V \backslash \{i,j\}}$.

\section{Analysis of the PC-Algorithm in Algorithm~\ref{alg:fPC}}
\label{app:GeneralResult}

In this section, we present the analysis of Algorithm~\ref{alg:fPC} using results from \citet{AnimaTanWillsky2011b} and \citet{KalischBuhlmann2007}.  The analysis presented here is for the non-junction tree based algorithm.  Throughout this Section, assume
\[ \widehat{G} = \fPC(\eta,\Xf^n,K_V,K_V) \,,\]
where $K_V$ is the complete graph over the vertices $V$.  Further, let the threshold for the conditional independence test in (\ref{eq:cit}) be $\lambda_n$. We are interested in finding conditions under which $\widehat{G} = G^*$ with high probability.
\begin{theorem}
\label{thm:suffcondgeneral}
Under Assumptions (A1)-(A5), there exists a conditional independence test such that if
\begin{align}
n = \Omega( \rho_{min}^{-2} \eta \log (p)) \; \text{or} \;
\rho_{min} = \Omega( \sqrt{\eta \log (p) / n}), \label{eq:rhopc}
\end{align}
then $P(\widehat{G} \ne G) \rightarrow 0$ as $n \rightarrow \infty$.
\end{theorem}

We now prove Theorem~\ref{thm:suffcondgeneral}.  Define the set $B_{\eta}$ as follows:
\begin{equation}
B_{\eta} = \{(i,j,S): i, j \in V, i \ne j, S \subseteq V \backslash \{i,j\}, |S| \le \eta \} \,.
\end{equation}
The following concentration inequality follows from \citet{AnimaTanWillsky2011b}.
\begin{lemma}
Under Assumption~(A4), there exists constants $c_1$ and $c_2$ such that for $\epsilon < M$,
\begin{equation}
\sup_{(i,j,S) \in B_{\eta}}P\left(||{\rho}_{ij|S}| - |\widehat{\rho}_{ij|S}|| > \xi\right) \le c_1 \exp \left( -c_2 (n - \eta) \xi^2 \right) \,,
\label{eq:conceq}
\end{equation}
where $n$ is the number of vector valued measurements made of $X_i, X_j$, and $X_S$.
\end{lemma}
Let $P_e = P( \widehat{G} \ne G)$, where the probability measure $P$ is with respect to $P_X$.  Recall that we threshold the empirical conditional partial correlation $\widehat{\rho}_{ij|S}$ to test for conditional independence, i.e., $\widehat{\rho}_{ij|S} \le \lambda_n \Longrightarrow X_i \ind X_j | X_S$.  An error may occur if there exists two distinct vertices $i$ and $j$ such that either $\rho_{ij|S} = 0$ and $|\widehat{\rho}_{ij|S}| > \lambda_n$ or $|\rho_{ij|S}| > 0$ and $|\widehat{\rho}_{ij|S}| \le \lambda_n$.  Thus, we have
\begin{align}
P_e &\le P({\cal E}_1) + P({\cal E}_2) \,, \\
P({\cal E}_1) &= P\left( \bigcup_{(i,j) \notin G} \{ \text{$\exists$ $S$ s.t. $|\widehat{\rho}_{ij|S}| > \lambda_n$} \} \right) \\
P({\cal E}_2) &= P\left( \bigcup_{(i,j) \in G} \{ \text{$\exists$ $S$ s.t. $|\widehat{\rho}_{ij|S}| \le \lambda_n$} \} \right) \,.
\end{align}
We will find conditions under which $P({\cal E}_1) \rightarrow 0$ and $P({\cal E}_2) \rightarrow 0$ which will imply that $P_e \rightarrow 0$.  The term $P({\cal E}_1)$, the probability of including an edge in $\widehat{G}$ that does not belong to the true graph, can be upper bounded as follows:
\begin{align}
P({\cal E}_1) &\le P\left( \bigcup_{(i,j) \notin G} \{ \text{$\exists$ $S$ s.t. $|\widehat{\rho}_{ij|S}| > \lambda_n$} \} \right)
\le P\left( \bigcup_{(i,j) \notin G, S \subset V \backslash \{i,j\}} \{ \text{$|\widehat{\rho}_{ij|S}| > \lambda_n$} \} \right) \\
&\le p^{\eta+2} \sup_{(i,j,S) \in B_{\eta}} P\left(|\widehat{\rho}_{ij|S}| > \lambda_n\right) \\
&\le c_1 p^{\eta+2} \exp\left(-c_2 (n-\eta)\lambda_n^2 \right)
= c_1 \exp\left( (\eta+2) \log(p) -c_2 (n-\eta)\lambda_n^2 \right)
\end{align}
The terms $p^{\eta+2}$ comes from the fact that there are at most $p^2$ number of edges and the algorithm searches over at most $p^{\eta}$ number of separators for each edge.  Choosing $\lambda_n$ such that
\begin{equation}
\lim_{n,p \rightarrow \infty} \frac{(n-\eta)\lambda_n^2}{(\eta+2)\log(p)} = \infty
\label{eq:xic1}
\end{equation}
ensures that $P({\cal E}_1) \rightarrow 0$ as $n,p \rightarrow \infty$.  Further, choose $\lambda_n$ such that for $c_3 < 1$
\begin{align}
\lambda_n < c_3 \rho_{min} \label{eq:xic2} \,.
\end{align}
The term $P({\cal E}_2)$, the probability of not including an edge in $\widehat{G}$ that does belong to the true graph, can be upper bounded as follows:
\begin{align}
P({\cal E}_2) &\le P\left( \bigcup_{(i,j) \in G} \{ \text{$\exists$ $S$ s.t. $|\widehat{\rho}_{ij|S}| \le \lambda_n$} \} \right) \\
&\le P\left( \bigcup_{(i,j) \in G, S \subset V \backslash \{i,j\}}  \text{$|\rho_{ij|S}|-|\widehat{\rho}_{ij|S}| > |\rho_{ij|S}| - \lambda_n$} \right) \\
&\le p^{\eta+2} \sup_{(i,j,S) \in B_{\eta}} P\left(|\rho_{ij|S}|-|\widehat{\rho}_{ij|S}| > |\rho_{ij|S}| - \lambda_n\right) \\
&\le p^{\eta+2} \sup_{(i,j,S) \in B_{\eta}} P\left(||\rho_{ij|S}|-|\widehat{\rho}_{ij|S}|| > \rho_{min} - \lambda_n\right) \\
&\le c_1 p^{\eta+2} \exp\left(-c_2 (n-\eta)(\rho_{min} - \lambda_n)^2 \right)
= c_1\exp\left( (\eta+2) \log(p) - c_4(n-\eta) \rho_{min}^2 \right) \label{eq:a5}
\,.
\end{align}
To get (\ref{eq:a5}), we use (\ref{eq:xic2}) so that $(\rho_{min}-\lambda_n) > (1-c_3) \rho_{min}$.  For some constant $c_5 > 0$, suppose that for all $n > n'$ and $p>p'$,
\begin{equation}
c_4(n-\eta)\rho_{min}^2 > (\eta+2+c_5) \log(p) \,. \label{eq:nn1}
\end{equation}
Given (\ref{eq:nn1}), $P({\cal E}_2) \rightarrow 0$ as $n,p \rightarrow \infty$.  In asymptotic notation, we can write (\ref{eq:nn1}) as
\begin{equation}
n = \Omega(\rho_{min}^{-2} \eta \log(p))
\end{equation}
which proves the Theorem.  The conditional independence test is such that $\lambda_n$ is chosen to satisfy (\ref{eq:xic1}) and (\ref{eq:xic2}).  In asymptotic notation, we can show that $\lambda_n = O(\rho_{min})$ and $\lambda_n^2 = \Omega\left( \eta \log(p)/n \right)$ satisfies (\ref{eq:xic1}) and (\ref{eq:xic2}).

\section{Proof of Theorem~\ref{thm:mainResult1}}
\label{app:mainResult1}

To prove the theorem, it is sufficient to establish that
\begin{align}
\rho_0 &= \Omega\left( \sqrt{\eta_T \log(p)/n} \right)  \label{eq:rho0} \\
\rho_1 &= \Omega\left( \sqrt{\eta \log(p_1)/n} \right) \label{eq:rrho1} \\ 
\rho_2 &= \Omega\left( \sqrt{\eta  \log(p_2)/n} \right) \label{eq:rrho2} \\ 
\rho_T &= \Omega\left( \sqrt{\eta  \log(p_T)/n} \right) \label{eq:rho_T} \,.
\end{align}

Let $H$ be the graph estimated in Step~1.  An error occurs if for an edge $(i,j) \in G^*$ there exists a subset of vertices $S$ such that $|S| \le \eta_T$ and $|\widehat{\rho}_{ij|S}| \le \lambda_n^0$.  Using the proof of Theorem~\ref{thm:suffcondgeneral} (see analysis of $P({\cal E}_2)$), it is easy to see that $n = \Omega(\rho_0^{-2} \eta_T  \log(p))$ is sufficient for $P(E(G^*) \not\subset E(H)) \rightarrow 0$ as $n \rightarrow 0$.  Further, the threshold is chosen such that $\lambda_n^0 = O(\rho_0)$ and $(\lambda_n^0)^2 = \Omega\left( \eta_T \log(p)/n \right)$.  This proves (\ref{eq:rho0}).

In Step~2, we estimate the graphs $\widehat{G}_1$ and $\widehat{G}_2$ by applying the PC-Algorithm to the vertices $V_1 \cup T$ and $V_2 \cup T$, respectively.  For $\widehat{G}_1$, given that all edges that have a separator of size $\eta_T$ have been removed, we can again use the analysis in the proof of Theorem~\ref{thm:suffcondgeneral} to show that for $\lambda_n^1 = O(\rho_1)$ and $(\lambda_n^1)^2 = \Omega\left( \eta \log(p_1)/n \right)$, $n = \Omega(\rho_1^{-2} \eta \log(p_1))$ is sufficient for $P(\widehat{G}_1 \ne {G}^*[V_1 \cup T] \backslash K_{T}) | G^* \subset H) \rightarrow 0$ as $n \rightarrow \infty$.  This proves (\ref{eq:rrho1}).  Using similar analysis, we can prove (\ref{eq:rrho2}) and (\ref{eq:rho_T}).

The probability of error can be written as
\begin{align}
P_e &\le P(G^* \not\subset H) + \sum_{k=1}^{2}P( \widehat{G}_k \ne {G}^* [V_k \cup T] \backslash K_{T} | G^* \subset H) \nonumber \\
&+ P(\widehat{G}_T \ne G^*[T] | G^* \subset H,\widehat{G} = {G}[V_1 \cup T]^* \backslash K_T,G^*[V_2 \cup T] = G[V_2 \cup T] \backslash K_T) \,.
\end{align}
Given (\ref{eq:rho0})-(\ref{eq:rho_T}), each term on the right goes to $0$ as $n \rightarrow \infty$, so $P_e \rightarrow 0$ as $n \rightarrow \infty$.

\end{document}